%% file: main.tex
\documentclass[11pt]{article} 
\usepackage{fullpage}
\usepackage[table]{xcolor}
\usepackage{hyperref}
\usepackage{url}

\usepackage[utf8]{inputenc} % allow utf-8 input
\usepackage[T1]{fontenc}    % use 8-bit T1 fonts
\usepackage{hyperref}       % hyperlinks
\usepackage{url}            % simple URL typesetting
\usepackage{booktabs}       % professional-quality tables
\usepackage{amsfonts}       % blackboard math symbols
\usepackage{nicefrac}       % compact symbols for 1/2, etc.
\usepackage{microtype}      % microtypography
\usepackage{xcolor}         % colors

\usepackage{times}
\usepackage{tikz}
\usepackage{mathtools}
\usepackage{amsthm}
\usepackage{amsmath}
\usepackage{bbm}
\usepackage{verbatim}
\usepackage{enumitem}
\usepackage{xspace}
\usepackage{todonotes}
\usepackage{amsfonts,mathrsfs,color}
\usepackage{xcolor}
\usepackage{graphicx}
\usepackage{booktabs}
\usepackage{nicefrac}
\usepackage{mdframed}
\usepackage{contour}
\contourlength{0.1pt}
\contournumber{10}
\usepackage{natbib}
\usepackage{multirow}
\usepackage{tikz}
\usepackage{wrapfig}
\usepackage{tabularx} % used to handle word wrapping
\usepackage{pifont}
\usepackage{enumitem}

\usepackage{subcaption} 
\usepackage{makecell}

\newcommand{\notcheckmark}{%
  \checkmark\raisebox{0.1em}{\makebox[0pt][r]{\kern-0.2em\textbackslash\kern0.1em}}
}

\newcommand{\declarecolor}[2]{\definecolor{#1}{RGB}{#2}\expandafter\newcommand\csname #1\endcsname[1]{\textcolor{#1}{##1}}}
\declarecolor{White}{255, 255, 255}
\declarecolor{Black}{0, 0, 0}
\declarecolor{Maroon}{128, 0, 0}
\declarecolor{Coral}{255, 127, 80}
\declarecolor{Red}{182, 21, 21}
\declarecolor{LimeGreen}{50, 205, 50}
\declarecolor{DarkGreen}{0, 80, 0}
\declarecolor{Purple}{146, 42, 158}
\declarecolor{Navy}{0, 0, 128}
\declarecolor{LightBlue}{84, 101, 202}
\usepackage{hyperref}
\hypersetup{
    colorlinks,
    citecolor=DarkGreen,
    linkcolor=LightBlue}
\usepackage[noabbrev, capitalise, nameinlink]{cleveref}
\input{macro}

\usepackage[ruled,linesnumbered,noend]{algorithm2e} % For algorithms
\SetKwComment{Comment}{/* }{ */}
\SetKwProg{Fn}{function}
\SetAlFnt{\relax}
\SetAlCapFnt{\relax}
\SetAlCapNameFnt{\relax}
\SetAlCapHSkip{0pt}

\newcommand{\piref}{\pi_{\mathrm{ref}}}
\newcommand{\pisft}{\pi_{\mathrm{sft}}}
\newcommand{\piinit}{\pi_{\mathrm{init}}}

\newcommand{\SPPO}{\mathrm{SPPO}}

\newcommand{\regSPPO}{\textnormal{Reg-SPPO}\xspace}
\newcommand{\DPO}{\mathrm{DPO}}
\newcommand{\DRO}{\mathrm{DRO}}
\newcommand{\regDRO}{\textnormal{Reg-DRO}\xspace}
\newcommand{\IPO}{\mathrm{IPO}}
\newcommand{\REBEL}{\mathrm{REBEL}}
\newcommand{\regREBEL}{\textnormal{Reg-REBEL}\xspace}
\newcommand{\INPO}{\mathrm{INPO}}
\newcommand{\LPO}{COMAL\xspace}
\newcommand{\MP}{Mirror-Prox\xspace}
\newcommand{\hpi}{\hat{\pi}}

\newcommand{\compactparagraph}[1]{\noindent\textbf{#1}}

\date{}

\title{COMAL: A \textbf{Co}nvergent \textbf{M}eta-Algorithm for \textbf{Al}igning LLMs with General Preferences}
\author{
Yixin Liu\thanks{Equal contribution; alphabetically ordered.}~$^1$, Argyris Oikonomou\footnotemark[1]~$^1$, Weiqiang Zheng\footnotemark[1]~$^1$, Yang Cai\footnotemark[2]~$^1$, Arman Cohan\footnotemark[2]~$^{1,2}$
 \\
$^1$Yale University, $^2$Allen Institute for AI\\
\texttt{\{yixin.liu, argyris.oikonomou, weiqiang.zheng\}@yale.edu} \\
\texttt{\{yang.cai, arman.cohan\}@yale.edu}
}

\begin{document}

\maketitle

\renewcommand{\thefootnote}{\fnsymbol{footnote}}
\setcounter{footnote}{2} 
\footnotetext[2]{Equal co-advising; alphabetically ordered.}
\renewcommand{\thefootnote}{\arabic{footnote}}
\setcounter{footnote}{0}

\begin{abstract}
    Many alignment methods, including reinforcement learning from human feedback (RLHF), rely on the Bradley-Terry reward assumption, which is not always sufficient to capture the full range and complexity of general human preferences. We explore RLHF under a general preference framework by modeling the alignment problem as a two-player zero-sum game in a game-theoretic framework, where the Nash equilibrium policy guarantees a 50\% win rate against any competing policy.   However, previous self-play algorithms for finding the Nash policy either diverge or only converge to a Nash policy in a modified game, even in a simple synthetic setting, thereby failing to maintain the 50\% win rate guarantee against all other policies. We propose a meta-algorithm, \textbf{Co}nvergent \textbf{M}eta \textbf{Al}ignment Algorithm (\LPO), for language model alignment with general preferences, inspired by convergent algorithms in game theory. We provide theoretical analysis that our meta-algorithm converges to an exact Nash policy in the last iterate and demonstrate its effectiveness on a range of synthetic and preference optimization datasets. COMAL is simple and can be integrated with many existing methods designed for preference optimization with minimal changes, and empirically it consistently maintains above 60.2\% and 56.8\% win rates, when applied to Llama-3-8B-Instruct and Qwen2.5-7B, against all compared algorithms under controlled evaluations.
\end{abstract}

\section{Introduction}
\label{sec:intro}

One of the most widely adopted approaches to addressing the challenge of aligning LLMs with human values and preferences is Reinforcement Learning from Human Feedback (RLHF) \citep{ChristianoLBMLA17,Ouyang0JAWMZASR22}. This framework consists of two steps: first, learning a reward model from a human preferences dataset, and second, optimizing the LLM using the proximal policy optimization (PPO) algorithm~\citep{schulman2017proximal}. More recently, \citet{rafailov2024direct} observed that the first step can be bypassed, proposing the direct preference optimization (DPO) algorithm, which directly optimizes the LLM from the dataset.

However, the aforementioned approaches crucially rely on the assumption that human preferences can be expressed using the Bradley-Terry (BT) model \citep{Bradley1952RankAO}. Unfortunately, the BT model is too restrictive to capture the richness and complexity of human preferences. For example, the BT model can only induce \emph{transitive} preferences -- i.e., if more people favor A over B, and B over C, then more people must favor A over C. 
Such transitivity may not hold in the presence of diverse populations and is also incompatible with evidence from human decision-making~\citep{may1954intransitivity,tversky1969intransitivity}. 
To illustrate this, consider a simple case where users are evaluating responses from an assistant to a nuanced question like: ``What's the best way to spend a Sunday?'' Some might prefer Response A (outdoor activities) over B (reading a book), while others prefer B over C (watching TV), yet a third group prefers C over A. These cyclic preferences -- A > B > C > A -- cannot be modeled by the BT framework. 
Moreover, even if each individual has a consistent (transitive) ranking, the aggregated preferences can exhibit intransitivity. 
In fact, even a mixture of two BT models cannot be parameterized by a single BT model.

\begin{figure*}[t]
  \centering
  \begin{subfigure}[t]{0.72\textwidth}
    \centering
    \includegraphics[width=0.19\textwidth]{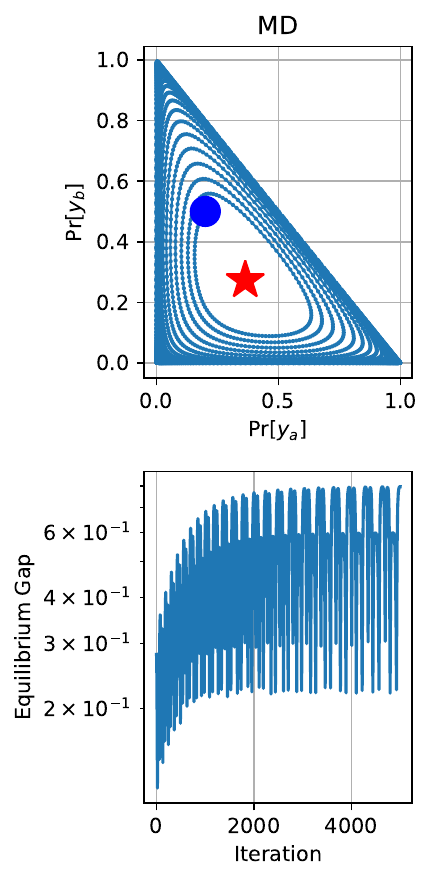}\hfill
    \includegraphics[width=0.19\textwidth]{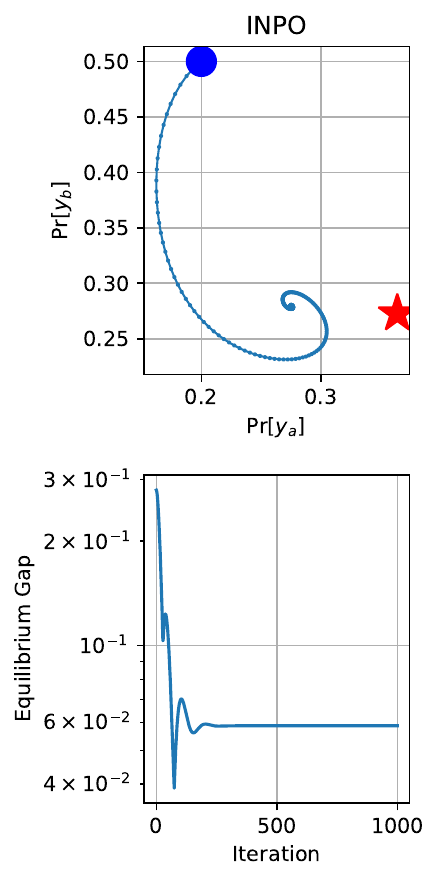}\hfill
    \includegraphics[width=0.19\textwidth]{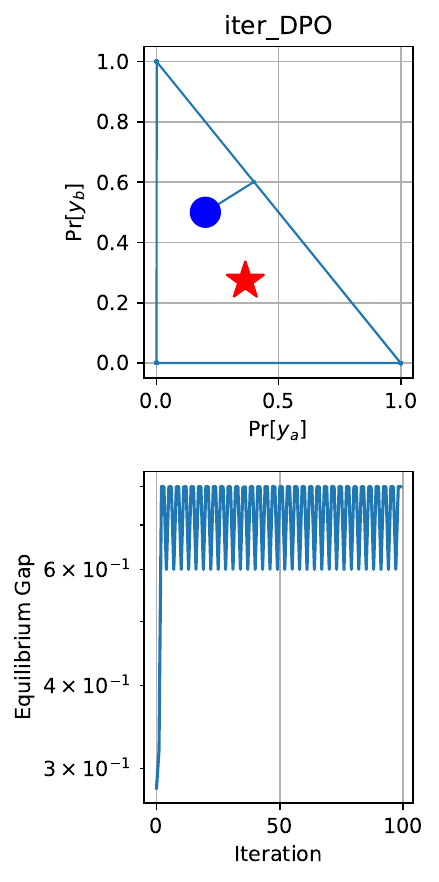}\hfill
    \includegraphics[width=0.19\textwidth]{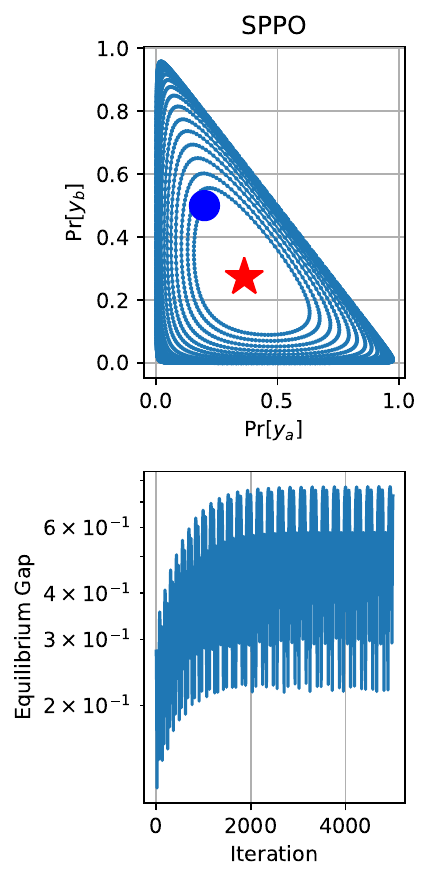}\hfill    
    \includegraphics[width=0.197\textwidth]{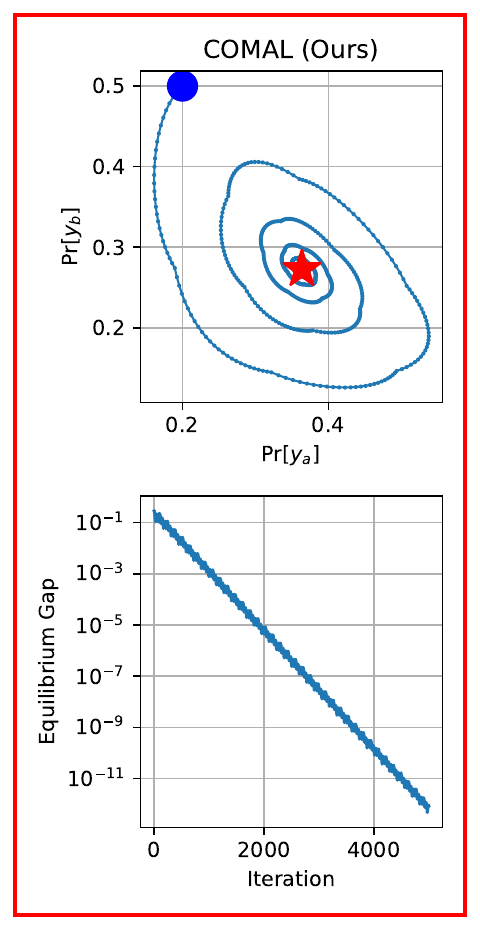}\hfill
    \caption{COMAL converges to the optimal solution, while other preference optimization methods do not. We initialize all algorithms at the blue dot; the Nash equilibrium is the red star.}
    \label{fig:convergence-synthetic}
  \end{subfigure}
  \hfill
  \begin{subfigure}[t]{0.25\textwidth}
    \centering
    \includegraphics[width=0.98\textwidth]{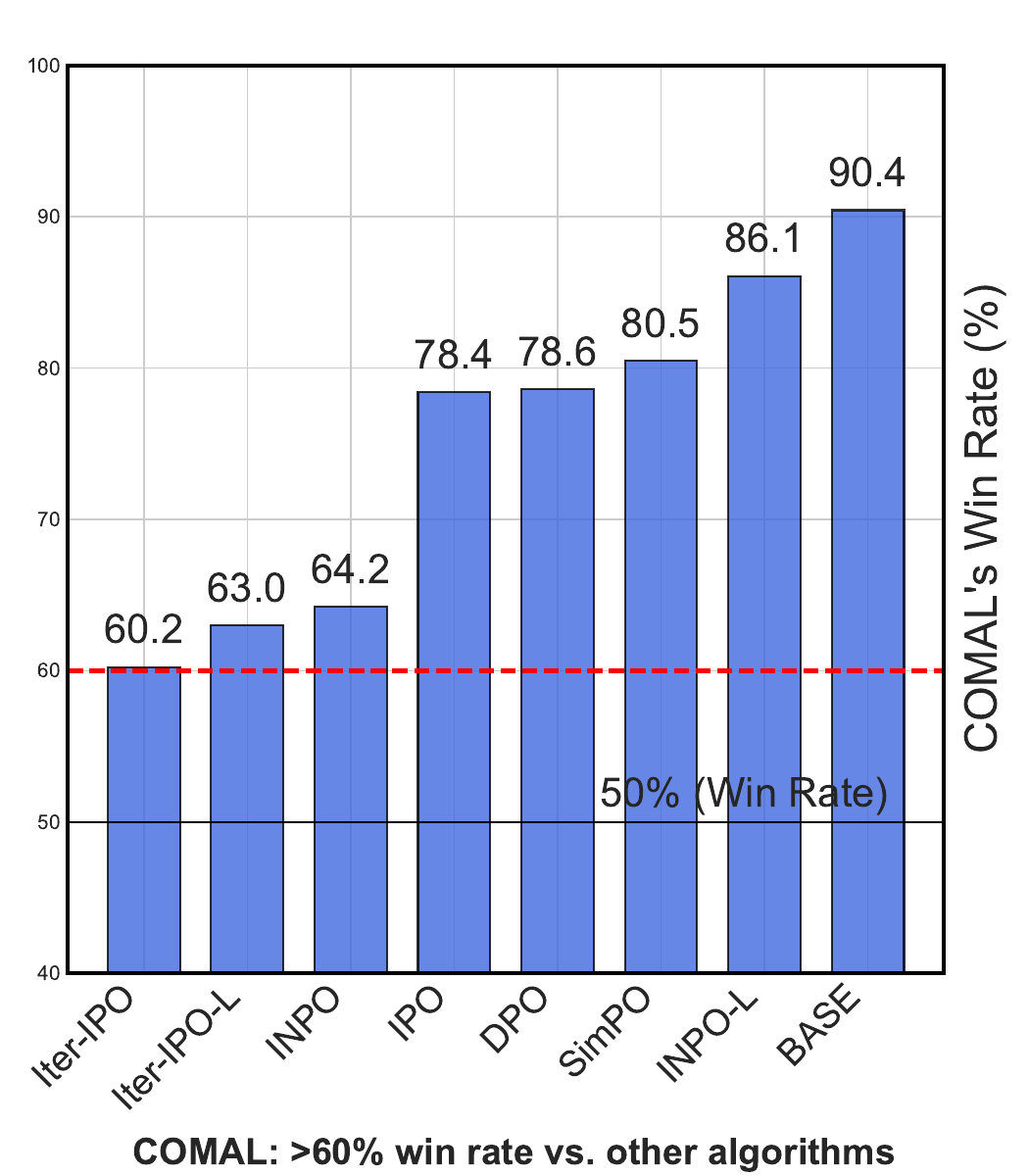}
    \caption{COMAL’s win rate against other PO algorithms.}
    \label{fig:lpo-win-rate}
  \end{subfigure}
  \caption{(a) convergence behavior of five methods (\S\ref{sec:syn-exp}); (b) win-rate comparison with Llama-3 (\S\ref{sec:llm}).}
% %\vspace{-15pt}
  \label{fig:main-empirical}
\end{figure*}

To overcome this limitation, recent research has begun to explore alignment under general preferences. ~\cite{munos2024nash,swamy2024minimaximalist} formulate this alignment problem as a symmetric two-player zero-sum \emph{alignment game} (\Cref{dfn:alignment game}), where both players’ strategies are LLMs, and their payoffs are determined by the win rate against the opponent’s LLM according to the preference model. The objective is to identify a Nash equilibrium policy, which guarantees at least 50\% win rate against any competing policy. Existing algorithms for finding such robust policies present significant challenges. In particular, current methods can suffer from instability, often failing to converge, or may inadvertently optimize for a solution to a modified version of the original problem. As a result, these approaches may fail to guarantee the desired win rate against arbitrary opponents, leaving robust alignment an open and active area of research and motivating us to investigate the following question:

\textbf{Question:} Is there an algorithm that \emph{converges} to the Nash equilibrium policy of the alignment game (\Cref{dfn:alignment game}), thus guaranteeing 50\% win-rate against any competing policy?

\compactparagraph{Our Contributions:} 
We propose a novel meta-algorithm, the \textbf{Co}nvergent \textbf{M}eta \textbf{Al}ignment Algorithm (\LPO), that iteratively refines language model policies by solving a regularized two-player zero-sum game at each round, using the current policy as a reference point. The procedure at each round is as follows: 

\emph{Step 1:} In iteration $t$, solve a KL-regularized two-player zero-sum game with respect to the reference policy $\piref = \pi_{t-1}$. Let $\pi_t$ be the Nash equilibrium of this regularized game.

\emph{Step 2:} Update the reference policy $\piref$ to the current policy $\pi_t$ and repeat the process.

The rationale behind \LPO is that it is a practical implementation of the Conceptual Prox-method~\citep{nemirovski2004prox}, a \emph{convergent} algorithm for solving two-player zero-sum games, whether regularized or not. Importantly, Step 1 can be implemented using the \(\Prox\) operator, a well-known concept in the optimization literature~\citep{parikh2014proximal}. A crucial observation we make here is that many existing algorithms -- including PPO~\citep{schulman2017proximal}, GRPO~\citep{shao2024deepseekmath, guo2025deepseek}, DPO~\citep{rafailov2024direct}, IPO~\citep{azar2024general}, SPPO~\citep{wu2024self}, REBEL~\citep{GaoCZOSBJBLS}, DRO~\citep{richemond2024offline}, and INPO~\citep{zhang2024iterative}, \emph{inter alia} -- can be interpreted as practical implementations of the \(\Prox\) operator in the context of LLM training (see \S\ref{sec:prox operator} for a detailed discussion). As a result, COMAL is simple and can be integrated with many existing methods designed
for preference optimization with minimal changes.

One significant departure of \LPO from existing methods for the game-theoretic formulation of alignment is that we adaptively update the reference policy rather than keeping it fixed. A potential concern with this approach is that the policy might drift too far from the initial policy, leading to instability and quality degradation. However, we provide both theoretical guarantees and experimental evidence demonstrating that this dynamic updating strategy consistently \emph{enhances} model performance while maintaining stability.

\textbf{Theoretical guarantee:} Given any implementation of the $\Prox$ operator, COMAL provably converges to the Nash equilibrium policy in the last iterate. While existing algorithms like iterative IPO~\citep{azar2024general} and SPPO~\citep{wu2024self} only guarantee average-iterate convergence (which is impractical for LLMs) or convergence to a KL-regularized Nash equilibrium \citep{munos2024nash,zhang2024iterative}, COMAL is the \emph{first} algorithm that has provable last-iterate convergence to the unregularized Nash equilibrium.\footnote{We remark that a concurrent work~\citep{wang2025magnetic} proposes an algorithm based on Magnetic Mirror Descent~\citep{sokota2023a} with last-iterate convergence. Our algorithms and theirs are all variants of the conceptual prox algorithm~\citep{nemirovski2004prox}. While their theoretical results require solving a regularized game \emph{exactly}, we provide stronger results showing last-iterate convergence under the more practical setting with only \emph{approximate} solutions (see \Cref{theorem:last-iterate-app}). Their experiments and our experiments both confirm the effectiveness of convergent regularized learning algorithms for LLM alignment. }

\textbf{Empirical improvements:} We first conduct synthetic controlled experiments on a $3\times 3$ two-player zero-sum alignment game and demonstrate that COMAL is the only algorithm that converges to the Nash equilibrium. Under realistic LLM training settings, in experiments with Llama-3-8B-Instruct \citep{dubey2024llama3herdmodels}  and Qwen2.5-7B~\citep{yang2024qwen25} on UltraFeedback \citep{cui2023ultrafeedback}, COMAL achieves above 60\% and 56\% win rates, respectively, against all compared algorithms according to the preference oracle.
% %\vspace{-5pt}
\section{Background}
% %\vspace{-5pt}
We begin by introducing notation for language model alignment and preference modeling.
Let $\Delta(\mathcal{Z})$ denote the set of distributions over a set $\mathcal{Z}$. Let $\mathcal{X}$ be the instruction set with a fixed distribution $\rho \in \Delta(\mathcal{X})$, and $\mathcal{Y}$ be the response set. Given an instruction $x \in \mathcal{X}$, an LLM policy $\pi$ specifies an output distribution $\pi(\cdot \mid x) \in \Delta(\mathcal{Y})$. For $p, q \in \Delta(\mathcal{Z})$, the Kullback-Leibler (KL) divergence is $\KL(p \| q) := \sum_{z \in \mathcal{Z}} p(z) \log \frac{p(z)}{q(z)}$. The sigmoid function is $\sigma(x) := \frac{e^x}{e^x + 1}$. We use $\supp(p)$ to denote the support of distribution $p$. This paper focuses on general preference models.

\begin{definition}[General Preference Model]
    A general preference model $\-P: \+X \times \+Y \times \+Y \rightarrow [0,1]$ satisfies $\-P(y_1 \succ y_2 \mid x) = 1 - \-P(y_2 \succ y_1 \mid x)$. When we query $\-P$ with $(x, y_1, y_2)$, it outputs $1$ with probability $\-P(y_1 \succ y_2 \mid x)$ meaning $y_1$ is preferred over $y_2$, and it outputs $0$ otherwise. The \emph{win rate} of $\pi_1$ over $\pi_2$ under preference model $\-P$ is $\-P(\pi_1 \succ \pi_2)$$:= \-E_{x \sim \rho}\InBrackets{ \-E_{y_1 \sim \pi_1, y_2 \sim \pi_2} \InBrackets{ \-P(y_1 \succ y_2 \mid x)}}$.
\end{definition} 
We present the Bradley-Terry (BT) model and additional backgrounds on RLHF and DPO to \S\ref{app:background}. %
% %\vspace{-5pt}
\subsection{Alignment with General Preference Models}
The Bradley-Terry (BT) model, while widely used in preference modeling, has fundamental limitations that restrict its ability to capture the full complexity of human preferences such as intransitive preferences,
especially when aggregating preferences across diverse populations or when dealing with nuanced, context-dependent decisions \citep{munos2024nash,swamy2024minimaximalist}.
To address these limitations and achieve alignment with general preferences, following \citep{munos2024nash,swamy2024minimaximalist},  we model the policy optimization problem as a two-player zero-sum game.

\begin{definition}[Alignment Game]\label{dfn:alignment game} 
    The \emph{alignment game} is a two-player zero-sum game with objective
    \begin{align}\label{eq:game objective}
        J(\pi_1, \pi_2) := \-P(\pi_1 \succ \pi_2) -\frac{1}{2}. 
    \end{align}
    The constant $\frac{1}{2}$ is introduced only to ensure the game is zero-sum and it has no other effect.
    We focus on policies with $\Pi:=\{\pi: \supp(\pi) \subseteq \supp(\piinit)\}$ in the support of the initial policy. A \textbf{Nash equilibrium} policy is 
     $(\pi_1^\star, \pi_2^\star) \in\argmax_{\pi_1\in\Pi}\argmin_{\pi_2\in\Pi} J(\pi_1, \pi_2)$ and satisfies
     $J(\pi_1, \pi_2^\star) \le J(\pi_1^\star, \pi_2^\star) \le J(\pi_1^\star, \pi_2), \forall \pi_1, \pi_2 \in \Pi.$
\end{definition}
In this game, the max-player controls $\pi_1$ and tries to maximize $J(\pi_1, \pi_2)$ while the min-player controls $\pi_2$ and tries to minimize $J(\pi_1, \pi_2)$. Since the game for two players $J(\pi_1, \pi_2)$ is symmetric~\citep{ye2024theoretical}, the game has a symmetric Nash equilibrium $(\pi^\star, \pi^\star)$. Moreover, the Nash equilibrium policy $\pi^\star$ guarantees that for any other policy $\pi$, its win rate is at least $\-P(\pi^\star \succ \pi) \ge \-P(\pi^\star \succ \pi^\star) = 50\%$. 
Our goal is to find a Nash equilibrium policy. %

Existing online iterative preference optimization methods designed for or applicable to the original game, including iterative IPO~\citep{azar2024general} and SPPO~\citep{wu2024self}, are based on Multiplicative Weights Update (MWU, definition in \S\ref{sec:solving game}), and thus \emph{diverge in the last iterate} as we show in \S\ref{sec:syn-exp}.\footnote{The MWU algorithm only has a weaker average-iterate convergence, i.e., $\frac{1}{T}\sum_{t=1}^T \pi^t$ converges.} There is also a line of works including Nash-MD~\citep{munos2024nash, ye2024theoretical}, Online IPO~\citep{calandriello2024human}, INPO~\citep{ zhang2024iterative} aim to find the Nash equilibrium of a modified KL-regularized game $J_{\tau}(\pi_1, \pi_2, \piref)$ defined as
\begin{align}\label{eq: regularized objective}
         &J(\pi_1, \pi_2) - \tau \-E_{x\sim \rho}\InBrackets{\KL(\pi_1(\cdot \mid x)|| \piref(\cdot \mid x))} + \tau \-E_{x\sim \rho}\InBrackets{\KL(\pi_2(\cdot \mid x)|| \piref(\cdot \mid x))}.
\end{align}
The additional KL regularization terms in the objective are introduced for training stability. However, the Nash equilibrium of the modified game no longer guarantees a win rate of at least 50\% against any competing policy. We compare these algorithms in \Cref{tab:algorithm-comparison}.

Moreover, most existing theoretical convergence guarantees only hold for the average iterate, i.e., the uniform mixture of training iterates, which is not used in practice. The last iterate is widely used in practice, is more space-efficient~\citep{munos2024nash}, and has better performance demonstrated by existing experimental results~\citep{munos2024nash,wu2024self,zhang2024iterative}. This motivates us to design principled algorithms with provable last-iterate convergence to Nash equilibrium policy.

\definecolor{lightgreen}{rgb}{0.8, 1, 0.8}  % 
\section{Convergent Meta-Algorithm for Alignment}
\label{sec:theory}
We propose a simple meta-algorithm, \textbf{Co}nvergent \textbf{M}eta \textbf{Al}ignment Algorithm (\LPO, Algorithm~\ref{alg:main}), for aligning LLMs with general preferences. %
In \S\ref{sec:LPO} and \ref{sec:solving game}, we present the theoretical foundations of \LPO and analyze its convergence properties. \S\ref{sec:prox operator} describes its practical implementation that integrates \LPO with existing preference learning methods.

%\vspace{-5pt}
\subsection{\LPO} \label{sec:LPO}
We now introduce \LPO, our meta-algorithm for preference-based policy optimization, inspired by the conceptual prox-method~\citep{nemirovski2004prox} from convex optimization and game theory. The prox-method has recently demonstrated strong practical performance in computing Nash equilibria for large-scale two-player zero-sum games~\citep{perolat2021poincare, song2020optimistic, pmlr-v235-abe24a} and has proven highly effective for the training of advanced game-theoretic AI systems~\citep{perolat2022mastering}. Here, we adapt this framework into an online iterative procedure that guarantees convergence to the Nash equilibrium in the alignment game $J(\pi_1, \pi_2)$~\eqref{eq:game objective}.%

%\vspace{-10pt}
\begin{algorithm}[!ht]
\LinesNotNumbered
    \caption{\textbf{Co}nvergent \textbf{M}eta \textbf{Al}ignment Algorithm (\LPO) for solving alignment game}\label{alg:main}
    \KwIn{Initial policy $\piinit$, preference oracle $\mathbb{P}$, regularization $\tau > 0$, number of iterations $T\geq 1$}
    \KwOut{Optimized policy $\pi^T$}
    Initialize $\pi^1, \piref\leftarrow \piinit$ \\
    \For{$t = 1, 2, \ldots, T-1$}{
        $\pi^{t+1} \leftarrow \argmax_{\pi_1}\min_{\pi_2} J_\tau(\pi_1, \pi_2, \piref)$ using Algorithm~\ref{alg:regularized game} (discussed in \S\ref{sec:solving game}) \\
        $\piref \leftarrow \pi^{t+1}$
    }
    \textbf{return} $\pi^T$
\end{algorithm}
%\vspace{-10pt}

\paragraph{Algorithmic Structure and Motivation.}
At each iteration $t$, \LPO~formulates and solves a \emph{regularized zero-sum game}, defined by the objective $J_\tau(\pi_1, \pi_2, \piref)$~\eqref{eq: regularized objective}, where the regularization encourages policies to remain close (in KL divergence) to a reference policy $\piref$. Specifically, the next policy $\pi^{t+1}$ is identified as a Nash equilibrium of this regularized game, with the current reference set to $\piref = \pi^t$. (See Algorithm~\ref{alg:regularized game} and further exposition in \S\ref{sec:solving game}). 
After convergence within this regularized subproblem, the reference policy is \emph{updated} to the newly computed $\pi^{t+1}$ (the latest iterate): $\piref \gets \pi^{t+1}$, and the process repeats.
This mechanism operationalizes a central insight of proximal algorithms: by updating the regularization center only when a regularized Nash equilibrium is reached, we ensure stable yet progressive movement toward the Nash equilibrium.

\textbf{Convergence and Monotonicity Guarantee.}
A key property of \LPO~is that the KL divergence to the Nash equilibrium policy $\pi^\star$ of the orginal game is monotonically non-increasing:$
    \KL(\pi^\star \Vert \pi^{t+1}) \le \KL(\pi^\star\Vert\pi^t).$
This holds for any choice of $\tau > 0$ (\Cref{lemma:single-step-app}), permitting the regularization strength to be adaptively adjusted during training without requiring a vanishing decay schedule. Each iteration thus provably brings the policy closer to the original Nash solution, justifying the update of the reference policy.
\begin{theorem}\label{thm: main LINEPO last-iterate}
     We assume that there exists a Nash equilibrium $\pi^\star$ of $J(\pi_1,\pi_2)$ (defined in \eqref{eq:game objective}) such that $\supp(\pi^\star) = \supp(\piinit)$. In every iteration $t \ge 1$, it holds that $\KL(\pi^\star|| \pi^{t+1}) \le \KL(\pi^\star||\pi^t)$. Moreover, \LPO has last-iterate convergence, i.e., $\lim_{t \rightarrow \infty} \pi^t$ exists and is a Nash equilibrium.
\end{theorem}

Moreover, while prior works~\citep{perolat2021poincare, pmlr-v235-abe24a, wang2025magnetic} require each regularized game to be solved \emph{exactly}, we prove a stronger result (\Cref{theorem:last-iterate-app}): last-iterate convergence holds even when each regularized game is solved only \emph{approximately}, as long as sufficient progress is made at each stage. This makes our result more robust and practical. Formal statements and proofs are provided in \S\ref{app:last-iterate convergence}. We also give non-asymptotic convergence rate in \Cref{thm:linear-strong-monotone}.

\textbf{Relation to Previous Work.} Prior iterative approaches to general-preference policy optimization---such as mirror descent-style algorithms~\citep{azar2024general, wu2024self}---typically only guarantee that the \emph{average} iterate converges (in mixture) to a Nash policy. However, in practice, averaging across many deep neural network checkpoints is both storage- and deployment-inefficient and uncommon. Furthermore, existing methods for last-iterate convergence apply only to \emph{regularized} games~\citep{munos2024nash, zhang2024iterative}, yielding stationary points that may diverge from true Nash equilibria of the alignment game (see also Table~\ref{tab:algorithm-comparison}). In contrast, \LPO~is the first framework to attain fully practical and provable last-iterate convergence to the Nash equilibrium of the alignment game, even in large-scale LLM contexts. The convergence in alignment game without regularization is crucial to ensure 50\% win rate against any other policy.

\textbf{Practical Instantiation.}
Each \LPO~iteration involves solving a regularized zero-sum game $J_\tau(\pi_1, \pi_2, \piref)$, for which many policy optimization algorithms originally developed for RLHF and preference learning (e.g., PPO, DPO, IPO, INPO) can serve as efficient sub-solvers; see \S\ref{sec:solving game} for discussion and \S\ref{app:prox} for variants. While the theoretical properties of \LPO~provide a strong foundation, its practical implementation and empirical validation in large-scale LLM alignment constitute a central contribution of this work. We show that \LPO~can be instantiated with, for example, \emph{INPO}~\citep{zhang2024iterative} as the regularized game solver (Algorithm~\ref{alg:LPO-practical}), yielding substantial and consistent performance gains across challenging alignment benchmarks. Our results demonstrate that \LPO~not only offers strong convergence guarantees but is also easy to deploy, highly scalable, and effective for real-world preference optimization and LLM fine-tuning. Notably, integrating \LPO~into existing pipelines typically requires only minimal modifications—chiefly, adding periodic reference policy updates and an outer iteration loop—making it directly compatible with current large-scale alignment workflows. \textbf{We note that the per-iteration computational cost of our algorithm is comparable to other alignment algorithms tested in our experiments}—differing by only a few percent—while achieving better performance without significant computational overhead.

%\vspace{-5pt}
\subsection{Solving a Regularized Game}\label{sec:solving game}

Each iteration of \LPO{} requires solving a regularized zero-sum game $J_\tau(\pi_1, \pi_2, \piref)$. We present Mirror Descent (MD) in Algorithm~\ref{alg:regularized game} for computing a Nash equilibrium of this game. MD builds on the prox operator, a principled tool from convex optimization that ensures stability and supports broad applicability. Importantly, we later show that this prox operator can be instantiated using a variety of modern policy optimization algorithms. For simplicity, we consider policies $\pi \in \Delta(\mathcal{Y})$ and omit dependence on the instruction $x$; all discussions extend naturally to the contextual setting.

% %\vspace{-10pt}
\begin{algorithm}[t]\label{alg:sub-game solver}
\LinesNotNumbered
    \caption{Regularized game solver for $J_\tau(\pi_1, \pi_2, \piref)$ -- $\argmax_{\pi_1}\min_{\pi_2} J_\tau(\pi_1, \pi_2, \piref)$}\label{alg:regularized game}
    \KwIn{Reference policy $\piref$, preference oracle $\mathbb{P}$, regularization $\tau > 0$, step size $\eta > 0$, number of iterations $K \ge 1$}
    \KwOut{Regularized Nash equilibrium policy $\mu_K$}
    Initialize $\mu^1 \leftarrow \piref$ \\
    \For{$k = 1, 2, \ldots, K-1$}{
        $g^k_\tau \leftarrow \nabla_{\mu^k} J_\tau(\mu^k, \mu^k, \piref) = \mathbb{P}( \cdot \succ \mu_k) - \tau\InParentheses{\log\frac{\mu_k(\cdot)}{\piref(\cdot)}+1}  $ \texttt{// Gradient}\\
        $\mu^{k+1} \leftarrow \Prox(\mu_k, \eta g^k_\tau)$ 
    }
    \textbf{return} $\mu_K$
\end{algorithm}

\textbf{Mirror Descent and Multiplicative Weights Update (MWU).}
Mirror Descent (MD) is a foundational family of iterative optimization algorithms, widely used in game theory, machine learning, and online learning~\citep{nemirovskij1983problem}. %
At a high level, MD generalizes vanilla gradient descent by using a geometry-aware update rule that better respects the structure of the optimization domain through a more flexible notion of ‘distance,’ defined by a regularizer. A particularly important special case is the \emph{Multiplicative Weights Update} (MWU) algorithm~\citep{arora2012multiplicative}, which can be viewed as Mirror Descent performed with the negative entropy regularizer. For concreteness, suppose we want to maximize some smooth objective $f(\pi)$ over probabilistic policies $\pi$. At iteration $t$, with current policy $\pi^t$, MWU computes the updated policy $\pi^{t+1}$ as the solution to:
\begin{align}
\label{eq:MWU}
    \pi^{t+1} := \Prox(\pi^t, \nabla f(\pi^T)):= \argmax_{\pi} \left\{ \langle \nabla f(\pi^t), \pi \rangle - \eta^{-1} \cdot \KL(\pi || \pi^t) \right\},
\end{align}
where $\eta$ is a positive parameter (step size), and $\KL(\cdot \| \cdot)$ is the Kullback-Leibler (KL) divergence, which in this case measures how much the new policy deviates from the previous one. Intuitively, this update chooses a new policy by trading off following the gradient of $f$ with staying close (in KL) to the prior policy, preventing overly aggressive changes that could destabilize learning.
This update can be viewed more generally through the lens of the \emph{proximal operator} (or \emph{prox operator})—a mathematical abstraction that unifies many optimization steps used in machine learning, including projected gradient descent and mirror descent with Bregman divergences~\citep{parikh2014proximal}. We include a detailed discussion on the prox operator in \S\ref{app:prox operator}.

\textbf{Non-asymptotic Convergence.}
Denote $\pi^\star_\tau$ the Nash equilibrium of the KL regularized game $J_\tau(\pi_1, \pi_2, \piref)$, which is $\tau$-strongly monotone. We can apply existing results to show that  MWU (Algorithm~\ref{alg:regularized game}) achieves linear last-iterate convergence rate: the KL divergence to the Nash equilibrium $\pi^\star_\tau$ decreases exponentially fast. The proof is in \S\ref{app:regularized}. \Cref{thm:linear-strong-monotone} also implies \textbf{a non-asymptotic convergence to an approximate Nash equilibrium}: we can choose $\tau = O(\varepsilon)$ and approaching an $\varepsilon$-approximate Nash equilibrium of the original alignment game~\eqref{eq:game objective} in $\Tilde{O}(1/\varepsilon^2)$ iterations.
\begin{theorem}\label{thm:linear-strong-monotone}
    For step size $0 < \eta \le \frac{\tau}{\tau^2 + 0.5}$,  Algorithm~\ref{alg:regularized game} guarantees for every $k \ge 1$, $\KL(\pi^\star_\tau||\mu^{k+1}) \le (1 - \frac{\eta \tau}{2})^k\KL(\pi^\star_\tau|| \piref)$.
\end{theorem}

%\vspace{-5pt}
\subsection{Practical methods for computing the prox operator}\label{sec:prox operator}
\begin{algorithm}[!ht]
\LinesNotNumbered
    \caption{Practical Implementation of \LPO integrated with INPO (Algorithm~\ref{alg:INPO})}\label{alg:LPO-practical}
    \KwIn{Initial policy $\piinit$, regularization $\{\tau_t > 0\}$, step size $\{\eta_t > 0\}$, number of outer iterations $T \ge 1$, number of inner iterations $\{K_t \ge 1\}$, preference oracle $\mathbb{P}$.}
    \KwOut{Optimized policy $\pi^T$}
    Initialize $\pi^1, \piref \leftarrow \piinit$ \\
    \For{$t = 1, 2, \ldots, T-1$}{
        $\pi^{t+1} \leftarrow \INPO(\piref, \tau_t, \eta_t, K_t, \mathbb{P})$ (Algorithm~\ref{alg:INPO})\\
        $\piref \leftarrow \pi^{t+1}$ 
    }
    \textbf{return} $\pi^T$
\end{algorithm}
We show how to implement \LPO in practical large-scale applications like LLM alignment by computing the prox operator, with a concrete implementation presented in \textbf{Algorithm~\ref{alg:LPO-practical}}.
Specifically, we observe that many existing algorithms designed for RLHF and preference optimization with neural network parameters can be extended to solve the prox operator %
. These algorithms include RL algorithms like PPO~\citep{schulman2017proximal} and GRPO~\citep{shao2024deepseekmath, guo2025deepseek} and loss-minimization algorithms like, DPO~\citep{rafailov2024direct}, IPO~\citep{azar2024general}, 
SPPO~\citep{wu2024self},  
REBEL~\citep{GaoCZOSBJBLS},
DRO~\citep{richemond2024offline},
INPO~\citep{zhang2024iterative}. Each of them may be preferred in certain settings. Due to space limitations, we defer the detailed discussion to \S\ref{app:prox}. We also note that our meta algorithm, COMAL, can be integrated with many existing methods designed for preference optimization with minimal change, and we present concrete instantiations of COMAL using iterative GRPO, SPPO, REBEL, and DRO in \S\ref{app:COMAL-SPPODROREBEL}.

Our unified view on existing diverse preference methods through the perspective of computing the prox operator opens the possibility of applying other algorithms from online learning and optimization to robust LLM alignment. We include implementations for two other last-iterate convergent algorithms, the Mirror-Prox algorithm~\citep{nemirovski2004prox} and the Optimistic Multiplicative Weights Update algorithm~\citep{rakhlin2013optimization, syrgkanis2015fast}, in \S\ref{app:MD&OMWU}.

%\vspace{-20pt}
\section{Synthetic Experiments}\label{sec:synthetic experiments}
\label{sec:syn-exp}
%\vspace{-5pt}
We conduct experiments on a simple bandit problem with $\+Y = \{y_a, y_b, y_c\}$ and non-BT preference model over $\+Y$. Specifically, we set $\-P[y_b \succ y_a] = \-P[y_c \succ y_b] = 0.9$ and $\-P[y_a \succ y_c] = 0.8$. Observe that the preference is intransitive and exhibits a preference cycle $y_c \succ y_b \succ y_a \succ y_c$. The detailed setup and result analysis are in \S\ref{app:synthetic} and \Cref{fig:main-empirical}, \ref{fig:synthetic gradient}, and \ref{fig: synthetic sample}. 
Due to the space limit, we only briefly discuss the results here. 
Our experiments show that iterative DPO, iterative IPO~\citep{azar2024general}, and SPPO~\citep{wu2024self} all cycle and diverge away from the unique Nash equilibrium. The INPO algorithm converges in the modified game as we show in \Cref{thm:linear-strong-monotone}. However, the converging point is not the Nash equilibrium of the original game and suffers a constant equilibrium gap. 
\LPO is the only algorithm that converges to the Nash equilibrium.

% %\vspace{-10pt}
\section{LLM-Based Experiments}
\label{sec:llm}
% %\vspace{-5pt}
 % 
 
% 
We conduct experiments based on Llama-3-8B-Instruct~\citep{dubey2024llama3herdmodels} and Qwen2.5-7B~\citep{yang2024qwen25},\footnote{Additional experiments based on Qwen2-1.5B~\citep{yang2024qwen2} are also provided in the \S\ref{app:llm-1.5b}.}  on a commonly used dataset UltraFeedback~\citep{cui2023ultrafeedback} to show the effectiveness of \LPO under the practical preference optimization setting, following Algorithm~\ref{alg:LPO-practical}.

% %\vspace{-5pt}
\subsection{Experimental Settings}
\label{subsec:exp-setting}
\compactparagraph{Instruction Set.} 
Our training experiments are conducted on the 64K instructions from the UltraFeedback dataset, which covers a broad range of instruction types and is well-suited and widely used for studying preference optimization in practical scenarios.

\compactparagraph{Preference Oracle.}
We choose a mixture of two BT reward models as the preference oracle to simulate the preference diversity among human annotators.
Specifically, the win rate of an output $y_a$ over $y_b$ parameterized by a mixture of two BT reward models $r_1$ and $r_2$ is
\begin{equation}
P(y_a > y_b)=\frac{1}{2}\cdot\frac{e^{r_1(y_a)}}{e^{r_1(y_a)} + e^{r_1(y_b)}} +\frac{1}{2}\cdot\frac{e^{r_2(y_a)}}{e^{r_2(y_a)} + e^{r_2(y_b)}}.
\end{equation}
The two reward models used are Skywork-Reward-Llama-3.1-8B-v0.2~\citep{liu2024skywork} and ArmoRM-Llama3-8B-v0.1~\citep{wang2024interpretable}, both achieving strong performance on various human preference alignment benchmarks in RewardBench~\citep{lambert2024rewardbench}.

\compactparagraph{Preference Data Generation.}
To construct the preference data, i.e., output pairs with a preference annotation specifying which one is better, we adopt the setting of \citet{zhang2024iterative} by sampling 5 candidate outputs for each instruction with a temperature of 0.8 and applying the preference oracle to select the best and the worst candidates to form a data point.

\compactparagraph{Baselines.}
The following baselines are compared: 
(1) \textbf{BASE}: Llama-3-8B-Instruct, which has already been fine-tuned, can be directly used as the base model following SimPO~\citep{meng2024simpo}.
For Qwen2.5-7B, we finetune it using the standard SFT objective on the Tulu3 SFT dataset~\citep{lambert2024t}.
(2) vanilla \textbf{DPO}~\citep{rafailov2024direct} and (3) vanilla \textbf{IPO}~\citep{azar2024general}, where one training iteration is performed over the entire instruction set of UltraFeedback with output pairs sampled from the BASE policy;
(5) \textbf{INPO}~\citep{zhang2024iterative}(Algorithm~\ref{alg:INPO}), where each iteration of training is performed on a single data subset;
(6) \textbf{Iterative IPO (Iter-IPO)}, which follows a training setting similar to INPO but without the KL regularization with respect to the static reference policy.

\compactparagraph{Training Details.}
To reduce computational cost, the instructions in UltraFeedBack are divided into six equal subsets (10K each), with one subset used per training iteration.
For iterative optimization algorithms, 18 training iterations are performed.
All iterative optimization algorithms compared have similar computational costs, each taking around 100 hours on 8 NVIDIA A6000 GPUs.
To the best of our knowledge, multi-iteration training like ours has rarely been explored in previous work.
For example, INPO only trained up to 3 iterations, equivalent to just one full round over UltraFeedback's instructions.
The overall update process is as follows:
(1) \textbf{Iter-IPO}: at each iteration, the reference policy in the IPO loss (Eq.~\ref{eq:ipo-loss}) is updated to the policy produced in the previous iteration;
(2) \textbf{INPO}: at each iteration, one optimization step in Algorithm~\ref{alg:INPO} is performed, with the reference policy fixed to the BASE policy;
(3) \textbf{\LPO}: as outlined in Algorithm~\ref{alg:LPO-practical}, \LPO uses INPO as a sub-routine, and updates the reference policy in INPO every 6 iterations, i.e., an entire pass of the instruction set.

\compactparagraph{Hyper-Parameters.} 
We conduct a grid search for the KL regularization strength, $\eta^{-1}$, for DPO, IPO and INPO, within the range of 0.001 - 0.1.
The value of $\tau$ in INPO (\Cref{eq:INPO}) is determined by following \citet{zhang2024iterative}, where $\eta\tau$ is set to a fixed ratio, $1/3$.
We found \textbf{Iter-IPO} and \textbf{INPO} achieve the best performance when $\eta^{-1}$ is 0.002.
However, in Llama-3 training, we observe rapid performance degradation of both algorithms after 6 training iterations.
Therefore, to study the algorithms' behavior in more training iterations, we perform additional experiments with $\eta^{-1}$ set to 0.1 (\textbf{Iter-IPO-L} and \textbf{INPO-L}), which leads to stabler training.
For \textbf{\LPO}, since it involves multi-round INPO training with adjustable KL regularization strengths (Algorithm~\ref{alg:LPO-practical}), we set $\eta^{-1}$ to 0.002 for the first INPO training round (i.e., 6 iterations) and adjust it to 0.1 for the subsequent two rounds, balancing training stability with efficiency.
In Qwen2.5 training, $\eta^{-1}$ is fixed to 0.002 for all algorithms since the training process remains stable.
More details are in \Cref{app:hyper}.

\compactparagraph{Evaluations.}
We use the instructions in a widely used benchmark, AlpacaEval~\citep{alpaca_eval}, to construct the test set, since these instructions cover various task scenarios.
However, instead of using GPT-4, the default evaluator for the AlpacaEval benchmark, \textbf{we chose to use the same preference oracle used during training as the evaluator}.
This follows the setting of previous work~\citep{munos2024nash}, which provides a controlled experimental setting, ensuring that the preference oracle the model learns to fit is also the one used to evaluate its performance.

 \begin{table*}[h]
\small
\caption{Performance comparison of different training algorithms evaluated by the preference oracle. The row v.s. column win rate (\%) is reported.
All the training is based on the \textbf{BASE} checkpoint, \texttt{Llama-3-8B-Instruct}.
For Iterative IPO (\textbf{Iter-IPO}) and \textbf{INPO}, we report their performance with both a small, optimal regularization ($\eta^{-1}=0.002$) after 6 iterations and a large regularization ($\eta^{-1}=0.1$, \textbf{Iter-IPO-L} and \textbf{INPO-L}) after 18 iterations. 
}
\vspace{-10pt}
\begin{center}
% \addtolength{\tabcolsep}{-2pt} 
\begin{tabular}{@{}lrrrrrrrrr@{}} \toprule
Row/Column & BASE & IPO & DPO & Iter-IPO-L & Iter-IPO & INPO-L & INPO & COMAL & Avg \\
\midrule
IPO         & 93.04 & 50.00 & 47.20 & 28.20 & 20.75 & 83.23 & 25.22 & 21.61 & 46.16 \\
DPO        & 92.42 & 52.80 & 50.00 & 28.57 & 21.37 & 81.49 & 26.46 & 21.37 & 46.81 \\
Iter-IPO    & \textbf{94.16} & \textbf{79.25} & \textbf{78.63} & 50.68 & 50.00 & \textbf{89.19} & 53.79 & 39.75 & 66.93 \\
INPO        & 92.92 & 74.78 & 73.54 & 47.08 & 46.21 & 87.20 & 50.00 & 35.78 & 63.44 \\
\midrule
COMAL       & 90.43 & 78.39 & \textbf{78.63} & \textbf{62.98} & \textbf{60.25} & 86.09 & \textbf{64.22} & \textbf{50.00} & \textbf{71.37} \\
\bottomrule
% \hline
% %\vspace{-20pt}
\end{tabular}
% \addtolength{\tabcolsep}{+2pt} 
\end{center}
\label{tab:main-comparison}
\end{table*}

\begin{table*}[h]
\small
\caption{Performance comparison of different training algorithms evaluated by the preference oracle. The row v.s. column win rate (\%) is reported.
All the training is based on the \textbf{BASE} checkpoint, which is fine-tuned from \texttt{Qwen2.5-7B} using the SFT objective.
}
% %\vspace{-5pt}
\vspace{-10pt}
\begin{center}
% \addtolength{\tabcolsep}{-2pt} 
\begin{tabular}{@{}lrrrrrrrr@{}} \toprule
  Row/Column   &   BASE &     IPO &   DPO &   Iter-IPO &     INPO &   COMAL &   Avg \\
\midrule
IPO      & 91.43 & 50.00 & 50.19 &      22.98 &  23.73 &   21.37 & 43.28 \\
 DPO      & 90.68 & 49.81 & 50.00 &      23.35 &  23.60 &   20.50 & 42.99 \\
 Iter-IPO & \textbf{91.68} & 77.02 & 76.65 &      50.00 &  50.43 &   43.11 & 64.81 \\
 INPO     & 90.81 & 76.27 & 76.40 &      49.57 &  50.00 &   42.11 & 64.19 \\
 \midrule
 COMAL    & 90.68 & \textbf{78.63} & \textbf{79.50} &      \textbf{56.89} &  \textbf{57.89} &   \textbf{50.00} & \textbf{68.93} \\
 \bottomrule
% \hline
% %\vspace{-20pt}
\end{tabular}
% \addtolength{\tabcolsep}{+2pt} 
\end{center}
\label{tab:main-comparison-qwen}
\end{table*}

\begin{figure}[h]
    \centering
    \includegraphics[scale=0.50]{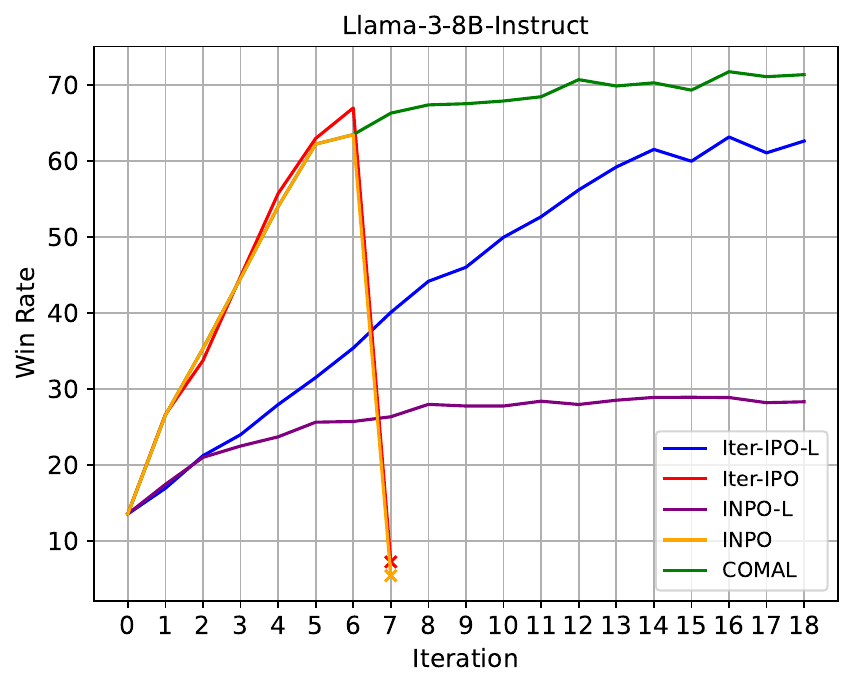}
    \includegraphics[scale=0.50]{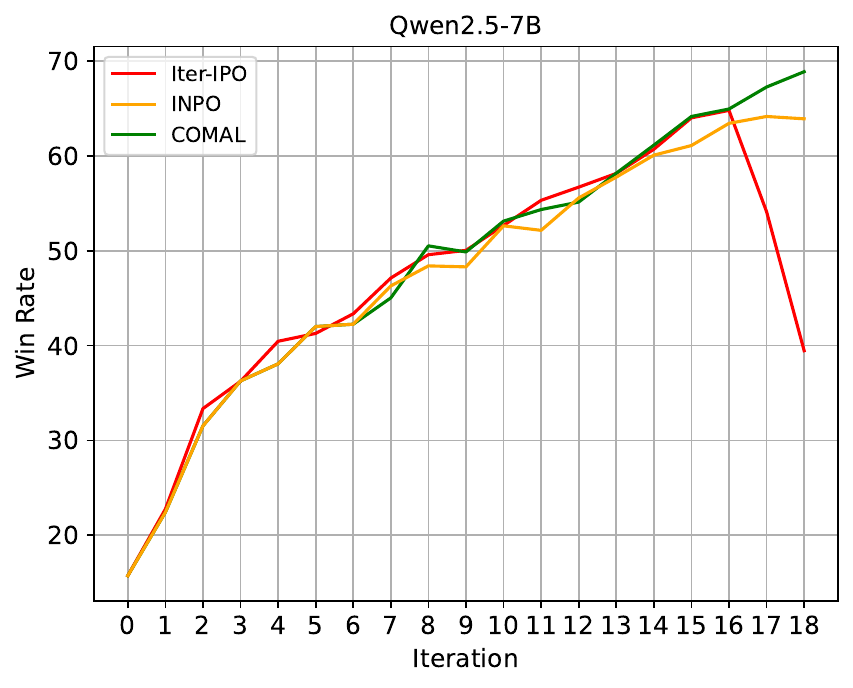}
    \vspace{-5pt}
    \caption{Comparisons of Iterative IPO (Iter-IPO), INPO, and \LPO. The average win rates of the trained checkpoints at each iteration against each training algorithm are displayed.
    }
    % %\vspace{-10pt}
    \label{fig:main-comparison}
\end{figure}

\begin{table}[h]
\small
\caption{Performance of various preference optimization algorithms on standard benchmarks.}
\vspace{-10pt}
\begin{center}
\addtolength{\tabcolsep}{-1pt} 
\begin{tabular}{@{}lrrrrrr@{}} \toprule
  \textbf{Method}   &  \textbf{GSM8K} & \textbf{MMLU} & \textbf{BBH} & \textbf{HumanEval} & \textbf{AlpacaEval}  & \textbf{Arena-Hard} \\
\midrule
 BASE (Llama-3-8B-Instruct)   & 77.0  & 63.5   & 66.4  & 79.2   &    25.0 &  21.3   \\
 IPO(-Llama-3-8B-Instruct)           & 74.0 & 63.1 & 53.4 & 65.3 & 48.7 & 38.9 \\
DPO(-Llama-3-8B-Instruct)           & 73.5 & 62.7 & 51.0 & 65.6 & 48.6 & 33.0 \\
Iter-IPO(-Llama-3-8B-Instruct)      & 71.5 & 64.4 & 62.6 & 77.7 & 50.6 & 43.8 \\
INPO(-Llama-3-8B-Instruct)          & 73.0 & 64.7 & 61.7 & 77.4 & 51.6 & 41.0 \\
 COMAL(-Llama-3-8B-Instruct) & 77.5  & 64.9   & 63.3  & 77.2   & 53.5 & 41.3 \\
 \midrule
 BASE (Qwen2.5-7B-SFT) & 77.5  & 70.9   & 65.6  & 84.0   &  14.7 &  22.3 \\
 IPO(-Qwen2.5-7B-SFT)              & 91.0 & 70.2 & 67.2 & 86.7 & 33.4 & 53.2 \\
DPO(-Qwen2.5-7B-SFT)              & 91.0 & 70.2 & 66.9 & 86.6 & 34.8 & 54.3 \\
Iter-IPO(-Qwen2.5-7B-SFT)         & 91.0 & 70.8 & 71.9 & 86.3 & 42.9 & 64.5 \\
INPO(-Qwen2.5-7B-SFT)             & 91.5 & 70.7 & 71.0 & 86.8 & 39.8 & 62.2 \\
COMAL(-Qwen2.5-7B-SFT)            & 91.0 & 70.8 & 72.3 & 85.1 & 42.2 & 63.0 \\
 \bottomrule
% \hline
% %\vspace{-25pt}
\end{tabular}
\addtolength{\tabcolsep}{+1pt} 
\end{center}
\label{tab:main-benchmarks}
\end{table}

% %\vspace{-5pt}
\subsection{Result Analysis}
\label{subsec:exp-result}

 \Cref{tab:main-comparison} and \Cref{tab:main-comparison-qwen} perform pairwise comparisons of different algorithms.
For \textbf{Iter-IPO} and \textbf{INPO}, we evaluate their \textit{best} checkpoints due to significant performance degradation thereafter. 
For \textbf{Iter-IPO-L}, \textbf{INPO-L}, and \textbf{\LPO}, comparisons are made at the final 18-iteration checkpoint.
The result shows that \textbf{\LPO achieves a win rate exceeding 60.2\% against all competing algorithms when using Llama-3-8B-Instruct, and 56.9\% with Qwen2.5-7B,} demonstrating its effectiveness.

\Cref{fig:main-comparison} presents the training dynamics of three iterative preference optimization algorithms, where the average win rate is computed against all the algorithms in \Cref{tab:main-comparison} and \Cref{tab:main-comparison-qwen}.
We note that:

\noindent (1) \LPO consistently outperforms other algorithms, \textbf{showing steady improvements even in the late stages of the training period.}

\noindent (2) 
Both Iter-IPO and INPO exhibit rapid degradation at the 7th training iteration in Llama-3 training.
We posit that this is because Llama-3-8B-Instruct has already undergone extensive post-training, making further optimization more delicate. 
Training with a larger KL-regularization with Llama-3-8B-Instruct leads to stabler training for both Iter-IPO(-L) and INPO(-L).
However, it also introduces a lower performance upper bound.
As discussed above, COMAL overcomes this limitation by dynamically adjusting the strength of the KL-regularization.

\compactparagraph{Evaluation Results on Standard Benchmarks.}
To verify that the checkpoints produced by our algorithm retain general capabilities, we compare their performance against the baselines on six standard LLM benchmarks as a sanity check.
These include GSM8K for math problem solving~\citep{Cobbe2021TrainingVT}, MMLU for multi-task language understanding~\citep{hendrycks2021measuring}, BigBench Hard (BBH) for reasoning~\citep{suzgun-etal-2023-challenging},  HumanEval for coding~\citep{Chen2021EvaluatingLL}, and two LLM alignment evaluation benchmarks, AlpacaEval and Arena-Hard, where the original evaluator, GPT-4, is used.
The results in \Cref{tab:main-benchmarks}, highlighting two findings:

(1) \textbf{COMAL maintains comparable performance on standard academic benchmarks};
(2) While not optimized for GPT-4’s preferences,\textbf{ COMAL performs strongly on AlpacaEval and Arena-Hard compared to the baselines}, indicating its generalizability.
We note that COMAL does not outperform Iter-IPO on Arena-Hard.
However, as noted above, we compare Iter-IPO at its best checkpoint, whereas COMAL is evaluated at the final checkpoint, because Iter-IPO's performance declines near the end of training (\Cref{fig:main-comparison}).
Moreover, 
since Arena-Hard compares each model only against a fixed baseline (GPT-4), its setup does not fully align with COMAL's objective.

\compactparagraph{Discussion on Updating the Reference Policy.} Our theoretical analysis in \Cref{sec:theory} indicates the reference policy in \LPO's objective needs to be updated in order to converge to the alignment game (\Cref{eq:game objective}).
Emprically, it means that COMAL does not have a KL-regularization from a static reference policy.
However, as shown in \Cref{tab:main-benchmarks}, COMAL does not suffer substantially from the ``alignment tax''~\citep{dong2024rlhf, Ouyang0JAWMZASR22}.
Moreover, we observe that its improvement is not solely from relaxing the KL-constraint -- Iter-IPO has even smaller constraints from a reference policy updated at each iteration, but fails to outperform COMAL and suffers from training instability.

\section{Conclusion}
We have proposed \LPO, a meta-algorithm for preference optimization that provably converges to the Nash equilibrium policy in the last iterate.
We have provided a theoretical analysis of the properties of \LPO and have empirically demonstrated its effectiveness under both synthetic and real-world experimental settings, where \LPO consistently maintains a win rate above 50\% against other policies in controlled settings.
We believe \LPO has significant potential to enhance the performance of LLMs in the alignment fine-tuning setting, due to its theoretical guarantees and flexibility, as it can be integrated with existing learning algorithms while overcoming their limitations.

\section*{Acknowledgements}
We are grateful for the TPU compute support provided by the Google TRC program and for the OpenAI API credits support provided by OpenAI's Researcher Access Program.

\bibliography{ref, refs}
\bibliographystyle{plainnat}

\appendix
\tableofcontents

\section{Related Work}

\paragraph{Alignment under Preference models} Most existing approaches adopt the Bradley-Terry (BT) preference model \citep{Bradley1952RankAO, ChristianoLBMLA17}, which involves first learning a preference model and then optimizing the objective function with a KL divergence penalty relative to the original language model. For example, RLHF~\citep{Ouyang0JAWMZASR22} aims to ensure that LLMs follow instructions by initially learning a BT model and subsequently fine-tuning the model based on the learned reward while regularizing it with the original LLM.

Building on this framework, \citet{rafailov2024direct} introduces Direct Preference Optimization (DPO) that maintains the assumption of the BT model for preferences but eliminates the preference learning step by reformulating the objective and optimizing it directly. Additionally, \citet{EthayarajhXMJK} diverges from the traditional BT-based methods by deriving algorithms that bypass the preference modeling step altogether. Instead, they model user preferences based on Kahneman and Tversky’s utility theory.

\paragraph{Alignment Solution Concepts under General Preferences}
\citet{azar2024general} is the first to consider general preferences. They propose the IPO algorithm, an offline algorithm that directly optimizes the win rate of the model penalized by the KL divergence with respect to the original model. \citet{munos2024nash} also consider general preferences and aim to find the \emph{von Neumann winner}, which corresponds to the Nash equilibrium of a game played between the two LLMs over the win rate. They propose a variant of the Mirror Descent (MD) algorithm called Nash-MD and show last-iterate convergence in the KL-regularized game.  Concurrently, \citet{swamy2024minimaximalist} study the same solution concept focusing more on sequential games. \citet{calandriello2024human} proved that the objective of the the IPO algorithm coincides with the Nash policy under a proper choice of the parameter that controls the regularization.

\paragraph{Iterative Self-Play Algorithms} Apart from the aforementioned works, recent research has also proposed practical implementations of Mirror Descent (MD) algorithms, which can be used to learn Nash equilibria through self-play. \citet{RossetCMSAX} propose Direct Nash Optimization (DNO), where, at each iteration, the model regresses predicted preferences against actual preferences using cross-entropy loss. Similarly, \citet{wu2024self} introduces the Self-Play Preference Optimization (SPPO) method,  \citet{GaoCZOSBJBLS} introduces Reinforcement Learning via Regressing Relative Rewards (REBEL), and \citet{richemond2024offline} introduces the Direct Reward Optimization (DRO) which regresses the loss using the $\ell_2$ distance at each iteration. %
Since these algorithms simulate the MD update, when applied in a two-player zero-sum game, they only have average-iterate convergence but all \emph{diverge in the last iterate}. Moreover, all these methods require the estimation of the win rate, which can be computationally expensive. 

Most closely related to our work is Iterative Nash Policy Optimization (INPO) by \citet{zhang2024iterative}, which continues to use $\ell_2$ distance regression. However, by further reformulating and simplifying the objective in a manner similar to IPO, INPO eliminates the need to estimate the expected win rate. The primary distinction between our approach and INPO is that INPO is designed for the KL-regularized game and is equivalent to MD; while our algorithm \LPO is inspired by the Conceptual Prox algorithm and guarantees last-iterate convergence in the original game. This fundamental difference allows \LPO to achieve more favorable convergence properties and outperform INPO, achieving a win rate strictly greater than 50\% against it.

\paragraph{Last-Iterate Convergence in Games} Mirror Descent fails to converge in simple zero-sum games, often resulting in cycling behavior \citep{MertikopoulosPP18}. In contrast, several algorithms have been shown to achieve last-iterate convergence including the Proximal Point (PP) method \citep{Rockafellar_pp}, Extra-Gradient (EG) \citep{korpelevich_extragradient_1976}, Optimistic Gradient Descent (OGD) \citep{popov_modification_1980, rakhlin2013optimization}, and the Conceptual Prox/Mirror Prox methods~\citep{nemirovski2004prox}. The asymptotic convergence properties of these algorithms have been extensively studied \citep{ popov_modification_1980, facchinei2003finite,  IusemPS, nemirovski2004prox, daskalakis2018limit}. Recently, there has been a growing focus on establishing finite-time convergence guarantees for these methods, addressing the practical necessity of understanding their performance within a limited number of iterations (see e.g., \citep{mokhtari2020convergence, mokhtari2020unified, golowich2020last, golowich2020tight, perolat2021poincare, BauschkeMW21, wei2021linear, cai2022finite, GorbunovTG22,cai2023accelerated,cai2023doubly,cai2023uncoupled, pmlr-v235-abe24a, cai2024accelerated,cai2024fast} and references therein). In particular, \citet{perolat2021poincare,pmlr-v235-abe24a, sokota2023a} propose algorithms that are variants of the Conceptual-Prox algorithm~\citep{nemirovski2004prox} and achieve last-iterate convergence under the assumption the regularized game can be solved exactly. Our work further extends their results to the case where the regularized game can be solved only approximately and demonstrates \LPO's effectiveness in large-scale LLM alignment setting. 

While our work focuses on the Conceptual-Prox algorithm, in \S\ref{app:MD&OMWU} we also include practical implementations of other convergent methods, including the mirror-prox method~\citep{nemirovski2004prox} that generalizes the extragradient method~\citep{korpelevich_extragradient_1976}, and the Optimistic Multiplicative Weight Update algorithm~\citep{rakhlin2013optimization}. We remark that several concurrent and subsequent works~\citep{zhou2025extragradient, zhang2025improving, wu2025multi, tiapkin2025accelerating} have also investigated both the theoretical and practical performance of \MP\  (which subsumes the extragradient method) and OMWU for LLM alignment. Taken together with our experiments, these studies provide extensive evidence that provably last-iterate convergent algorithms are effective for LLM alignment.

\begin{table}
\small
\caption{Property comparison of different preference optimization algorithms. 
The algorithms are compared based on whether they work for \textbf{\textit{general preferences}} and
whether they exhibit \textbf{\textit{last-iterate convergence}} in two-player zero-sum games.
\notcheckmark: convergence only in the modified KL-regularized game $J_\tau(\pi_1, \pi_2, \piref)$~\eqref{eq: regularized objective} but not in $J(\pi_1, \pi_2)$~\eqref{eq:game objective}.}
\addtolength{\tabcolsep}{-2pt} 
    \centering
    \begin{tabular}{lcc}
    \toprule
         Algorithm & \makecell{General Preference} & Convergence \\
    \midrule
    DPO~\citep{rafailov2024direct} IPO~\citep{azar2024general} & \xmark & \xmark \\ \midrule
    SPPO~\citep{wu2024self} Nash-MD~\citep{munos2024nash} & \cmark & \xmark  \\ \midrule
    INPO~\citep{zhang2024iterative} & \cmark & \notcheckmark  \\ \midrule
   \rowcolor{lightgreen} \LPO (Algorithm~\ref{alg:main})  & \cmark & \cmark  \\
    \bottomrule
    \end{tabular}
\addtolength{\tabcolsep}{+2pt} 
    \label{tab:algorithm-comparison}
\end{table}

\section{Additional Backgrounds}\label{app:background}
A special case of the general preference model is the Bradley-Terry (BT) model, which assumes a reward function parameterizes the preference. 

\begin{definition}[Bradley-Terry Model]
A preference model $\-P$ satisfies the \emph{Bradley-Terry (BT) assumption} if there exists a reward function $r^*: \+X \times \+Y \rightarrow \-R$ such that 
\begin{align*}
    \-P(y_1 \succ y_2 \mid x)  &= \frac{\exp\InParentheses{r^*(x, y_1)}}{\exp\InParentheses{r^*(x, y_1)}+ \exp\InParentheses{r^*(x, y_2)}} \\ & = \sigma(r^*(x, y_1) - r^*(x, y_2)).
\end{align*}
\end{definition}

\subsection{Alignment under the Bradley-Terry Model Assumption}
\paragraph{RLHF} Reinforcement Learning from Human Feedback (RLHF) is to first learn a reward function $r$ under the BT model and then find the optimal KL regularized policy $\pi^*$ w.r.t. the learned reward function $r$:
\begin{align}
\label{eq:RLHF}
    \pi^* &:= \arg \max_{\pi} \-E_{x \sim \rho, y \sim \pi(\cdot \mid x)} \nonumber \\
    &\InBrackets{ r(x,y) - \eta^{-1}\KL(\pi(\cdot \mid x) || \piref(\cdot \mid x))},
\end{align}
where $\eta^{-1} > 0$ controls the regularization, and $\piref$ is the initial reference model, usually the policy $\pisft$ obtained from pre-training and supervised fine-tuning. 

\paragraph{DPO} \citet{rafailov2024direct} observe that the regularized optimization problem \eqref{eq:RLHF} has a closed-form solution: 
for any $x$ and $y$, 
$
    \pi^*(y\mid x) = \frac{\piref(y \mid x) \exp\InParentheses{\eta r(x,y)} }{Z_x},
$
where $Z_x = \-E_{y \sim \piref(\cdot \mid x)}\InBrackets{\exp(\frac{1}{\eta} r(y,x))}$ is the normalization constant known as the partition function. Since $\pi^*$ implicitly parameterizes the reward function $r$. \citet{rafailov2024direct} propose direct preference optimization (DPO) to learn the optimal policy using the maximum likelihood objective directly:
\begin{align*}
    &\ell_\DPO(\pi;\piref) = - \-E_{(x, y_w, y_l) \sim \+D} \\
    &\InBrackets{\log \sigma\InParentheses{\eta^{-1} \log \frac{\pi(y_w\mid x)}{\piref(y_w\mid x)} - \eta^{-1} \log \frac{\pi(y_l\mid x)}{\piref(y_l\mid x)} }},
\end{align*}
where $\+D$ is a data set containing win-loss pair of responses $\{y_w, y_l\}$ given prompt $x$.

\section{Prox Operator}\label{app:prox operator}

\textbf{Prox Operator and Bregman Divergence.}
To define the prox operator, we first introduce the \emph{Bregman divergence}, which generalizes the notion of squared distance. For a convex function $\varphi : \mathcal{Z} \rightarrow \mathbb{R}$ (called the regularizer), the Bregman divergence between $z$ and $z'$ is defined as $
    D_\varphi(z \| z') := \varphi(z) - \varphi(z') - \langle \nabla \varphi(z'), z - z' \rangle$,
where $\nabla \varphi(z')$ is the gradient of $\varphi$ at $z'$. The \emph{prox operator} then takes a current point $z \in \mathcal{Z}$ and a (sub)gradient direction $g \in \mathbb{R}^n$, and returns the next point according to:
\[
    \mathrm{Prox}(z, g) := \argmax_{z'} \left\{ \langle g, z' \rangle - D_\varphi(z' \| z) \right\}.
\]
This formulation interpolates between moving in the direction of $g$ and staying close to $z$, as measured by the Bregman divergence for the chosen regularizer. Two important \emph{special cases} are
\begin{itemize}
    \item When $\varphi(z) = \frac{1}{2} \|z\|^2$ (the squared Euclidean norm), $D_\varphi$ is just the squared distance, and the prox operator reduces to the usual projected gradient step.
    \item When $\varphi(z)$ is the negative entropy (as in MWU), the Bregman divergence is the KL divergence, leading to updates appropriate for probability distributions.
\end{itemize}
In our framework, we will instantiate the prox operator with choices of $\varphi$ and $g$ that map directly onto concrete policy-learning algorithms. In this paper, when we refer to the prox operator, we use the negative entropy regularizer $\varphi(z) = \sum_{i=1}^n z[i] \ln z[i]$, for which the corresponding Bregman divergence $D_\varphi$ is the KL divergence. Under this choice, the MWU update in Equation~\eqref{eq:MWU} is equivalent to the prox-form update $\pi^{t+1} = \Prox(\pi^t, \eta \nabla f(\pi^t))$.

\subsection{Properties of the Prox Operator}

Recall that $\Prox(z, g) = \argmax_{z' \in \+Z}  \InAngles{g, z'} - D_\varphi(z'||z) = \argmax_{z' \in \+Z} \InAngles{g+\nabla\varphi(z), z'} - \varphi(z')$. The following properties of the prox operator are well-known in the literature(e.g., \citep{nemirovski2004prox}) 
\begin{lemma}
    $\Prox(z,g) = z'$ if and only if $\InAngles{g + \nabla\varphi(z) -\nabla\varphi(z'), z' - z^*} \ge 0$ for all $z^* \in \+Z$.
\end{lemma}
\begin{corollary}\label{coro:three point inequality}
    Let $\Prox(z,g) = z'$, then 
    \begin{align*}
        \InAngles{g, z^* - z'} \le D_\varphi(z^*||z) - D_\varphi(z^*||z') - D_\varphi(z'||z), \quad \forall z^*\in \+Z
    \end{align*}
\end{corollary}

\section{Last-Iterate Convergence of COMAL}\label{app:last-iterate convergence}
The proof of \Cref{thm: main LINEPO last-iterate} is largely inspired by existing results for the conceptual prox algorithm in the literature~\citep{facchinei2003finite, nemirovski2004prox}. We first consider the case where each step of \LPO, $\pi^{t+1} \leftarrow \argmax_{\pi_1}\min_{\pi_2} J_\tau(\pi_1, \pi_2, \piref)$, can be solved \emph{exactly} in \Cref{app:exact}. We then extend the proof to the case where we only solve the regularized game \emph{approximately} in \Cref{app:approximation}. In both cases, we prove last-iterate convergence to Nash equilibrium, i.e., $\lim_{t\rightarrow \infty} \pi^t$ exists and is a Nash equilibrium. The proof for the latter case seems to be the first in the literature.

In \Cref{thm: main LINEPO last-iterate}, we make the following assumption. 
\begin{assumption}\label{assumption:support}
    We assume there exists a Nash equilibrium $\pi^\star$ such that $\supp(\pi^\star) = \supp(\piinit)$.
\end{assumption}
This assumption is mild and \textbf{much weaker} than the ``Bounded Log Density" assumptions used in previous works~\citep{RossetCMSAX, zhang2024iterative}, which directly assumes $|\log \frac{\pi^t}{\piinit}|$ is bounded. 

\subsection{Last-Iterate Convergence under Exact Solutions}\label{app:exact}
Recall that  $\Pi:=\{\pi: \supp(\pi) \subseteq \supp(\piinit)\}$. Then $\KL(\pi||\piinit) \le D:= \max_{y: \piinit(y)>0} \log\piinit(y)$ is bounded for any $\pi \in \Pi$. We first prove $\KL(\pi^\star ||\pi^{t+1}) \le \KL(\pi^\star || \pi^t)$ for any $t \ge 1$.
\begin{lemma}\label{lem:single-step}
    Let $\pi^\star$ be an Nash equilibrium of $J(\pi_1, \pi_2)$. Then for any $\tau > 0$, if \[(\pi, \pi) = \argmax_{\pi_1 \in \Pi}\argmin_{\pi_2 \in \Pi} J_\tau(\pi_1, \pi_2, \piref),\]
    then 
    \[\KL(\pi^\star|| \pi) \le \KL(\pi^\star||\piref) - \KL(\pi||\piref)\]
\end{lemma} 
\begin{proof}
    By definition of the prox operator, we have
    \begin{align}\label{eq:proximal point}
        \pi &= \argmax_{\pi_1\in \Pi} J_\tau(\pi_1, \pi, \piref) \nonumber\\
            &= \argmax_{\pi_1 \in \Pi} \-P(\pi_1 \succ \pi) - \tau \KL(\pi_1, \piref)\nonumber \\
            &= \Prox(\piref, \frac{1}{\tau}\-P(\cdot \succ \pi) ).
   \end{align}
   Using \Cref{coro:three point inequality}, we have for any $\pi' \in \Pi$,
   \begin{align}\label{eq:regret}
      \frac{1}{\tau} \InParentheses{ \-P(\pi' \succ \pi) - \-P(\pi \succ \pi)} \le \KL(\pi'||\piref) - \KL(\pi'||\pi) - \KL(\pi|| \piref).
   \end{align}
   Plugging $\pi' = \pi^\star$ into the above inequality and noting that $\-P(\pi \succ \pi) = \frac{1}{2}$, we get 
   \begin{align*}
      \frac{1}{\tau} \InParentheses{ \-P(\pi^\star \succ \pi) - \frac{1}{2}} \le \KL(\pi^\star||\piref) - \KL(\pi^\star||\pi) - \KL(\pi|| \piref).
   \end{align*}
   Since $\pi^\star$ is a Nash equilibrium and thus $\-P(\pi^\star \succ \pi) \ge \frac{1}{2}$, the lefthand side of the above inequality is $\ge 0$. Then we have
   \begin{align*}
       \KL(\pi^\star||\pi) \le \KL(\pi^\star||\piref) - \KL(\pi|| \piref).
   \end{align*}
\end{proof}
\Cref{lem:single-step} implies the following properties on the trajectory $\{\pi^t\}$.
\begin{corollary}\label{coro:distance decreasing & bounded path length}
    Denote $\pi^\star$ an Nash equilibrium such that $\supp(\pi^\star) = \supp(\piinit)$ as guaranteed by \Cref{assumption:support}. Then the following holds for the trajectory $\{\pi^t\}$ produced by \LPO:
    \begin{itemize}
        \item[1.] $\KL(\pi^\star||\pi^{t+1}) \le \KL(\pi^\star||\pi^t)$ for all $t \ge 1$.
        \item[2.] $\sum_{t=1}^\infty \KL(\pi^{t+1}||\pi^t) \le \KL(\pi^\star|| \piinit)  < +\infty$.  
        \item[3.] For all $t \ge 1$, it holds that for $y \in \supp(\piinit)$, $\pi^t(y) \ge c > 0$ where $c$ is some constant $c$ depends only on $\pi^\star$ and $\piinit$.  This also holds even for any limit point of $\{\pi^t\}$.
    \end{itemize}
\end{corollary}
\begin{proof}
    The first item is direct from \Cref{lem:single-step}. The second item is also direct by applying \Cref{lem:single-step} for $t \ge 1$:
    \begin{align*}
        \sum_{t=1}^\infty \KL(\pi^{t+1}||\pi^t) \le \sum_{t=1}^\infty \KL(\pi^\star||\pi^t)  - \KL(\pi^\star|| \pi^{t+1}) \le \KL(\pi^\star|| \piinit) \le D < \infty.   
    \end{align*}
    Now we consider the third item. Define $D:=\KL(\pi^\star|| \piinit)$ and $p_{\min}: =\min_{y \in \supp(\pi^\star)} \pi^\star(y)$. By \Cref{assumption:support}, $p_{\min}>0$. Then
    \begin{align*}
        \KL(\pi^\star || \pi^t) \le D &\Rightarrow p_{\min} \log \frac{p_{\min}}{\pi^t(y)} \le D, \forall y \in \supp(\pi^\star) \\
        &\Rightarrow \pi^t(y) \ge \frac{ p_{\min}}{\exp(D/p_{\min})}, \forall y \in \supp(\pi^\star).
    \end{align*}
    Since the above holds for all $\pi^t$, it also holds for any limit point of $\{\pi^t\}$.
\end{proof}

Since the sequence $\{\pi^t\}$ is bounded (all lies in the simplex), it has at least one limit point $\hat{\pi}$. The next lemma shows that a limit point must be a Nash equilibrium.

\begin{lemma}\label{lemma:fixed point to NE}
    If $\hat{\pi}$ is a limit point of $\{\pi^t\}$, %
    then $\hat{\pi}$ is a Nash equilibrium of $J(\pi_1, \pi_2)$. 
\end{lemma}
\begin{proof}
    By item 2 in \Cref{coro:distance decreasing & bounded path length}, we have $\lim_{t\rightarrow\infty} \KL(\pi^{t+1}||\pi^t) = 0$. This implies $\lim_{t\rightarrow\infty} \InNorms{\pi^{t+1} - \pi^t} = 0$. As $\hat{\pi}$ is a limit point of $\{\pi^t\}$, we let $\{\pi^k: k\in \kappa\}$ be the subsequence that converges to $\hat{\pi}$. Then by \Cref{eq:proximal point}, we have 
    \begin{align*}
        &\lim_{k \in \kappa, k \rightarrow \infty} \pi^{k+1} = \lim_{k \in \kappa, k \rightarrow \infty} \Prox(\pi^k, \frac{1}{\tau} \-P(\cdot\succ \pi^{k+1})) \\
        \Rightarrow & \hat{\pi} = \Prox(\hat{\pi}, \frac{1}{\tau}\-P(\cdot \succ \hat{\pi})).
    \end{align*}
    Thus $\hat{\pi}$ is a fixed point of $\Prox(\pi, \frac{1}{\tau}\-P(\cdot \succ \pi)$. Moreover, by item 3 in \Cref{coro:distance decreasing & bounded path length}, we have $\supp(\hat{\pi}) = \supp(\piinit)$.
    Now consider both the max and min player running MWU initialized with $\pi^1 = \hat{\pi}$. Then we have $\pi^t = \hat{\pi}$ for all $t \ge 1$. By \Cref{eq:regret}, we have for any $\pi' \in \Pi$,
    \begin{align*}
       \frac{1}{\tau} \sum_{t=1}^\infty \InParentheses{\-P(\pi' \succ \hat{\pi}) - \frac{1}{2}} \le \KL(\pi'|| \hat{\pi}) < \infty,
    \end{align*}
    where the inequality holds since $\supp(\pi') \subseteq \supp(\hat{\pi})$. As a result, we get 
    \begin{align*}
        \-P(\pi' \succ \hat{\pi}) \le \frac{1}{2}, \forall \pi' \in \Pi \Leftrightarrow \-P(\hat{\pi} \succ \pi') \ge \frac{1}{2}, \forall \pi' \in \Pi
    \end{align*}
    Thus $\hat{\pi}$ is a Nash equilibrium of $J(\pi_1, \pi_2)$. 
\end{proof}

\paragraph{Proof of \Cref{thm: main LINEPO last-iterate}} By \Cref{lemma:fixed point to NE}, we know a limit point $\hat{\pi}$ is a Nash equilibrium. Then by \Cref{coro:distance decreasing & bounded path length}, $\{\KL(\hat{\pi}||\pi^t) \ge 0\}$ is a decreasing sequence. Thus $\{\KL(\hat{\pi}||\pi^t)\}$ converges. Let $\{\pi^k: k\in \kappa\}$ be a subsequence that converges to $\hat{\pi}$. Then we have
\begin{align*}
    \lim_{t\rightarrow \infty} \KL(\hat{\pi}||\pi^t) = \lim_{k \in \kappa, k\rightarrow \infty} \KL(\hat{\pi}||\pi^k) = \KL(\hat{\pi}||\hat{\pi}) = 0.
\end{align*}
Thus we have $\lim_{t\rightarrow \infty}\pi^t = \hat{\pi}$ is a Nash equilibrium.  This completed the proof of \cref{thm: main LINEPO last-iterate}.

\subsection{Last-Iterate Convergence under Approximate Solutions}\label{app:approximation}
This section considers the case where we can not solve the regularized game $J_\tau(\pi_1, \pi_2, \piref)$ exactly but only compute an approximate solution. Specifically, we consider the following inexact \LPO update: denote $\hpi^{t+1} = \argmax_{\pi_1 \in \Pi}\min_{\pi_2 \in \Pi} J_\tau(\pi_1, \pi_2, \pi^t)$ the exactly solution; the algorithm updates the next iterate $\pi^{t+1}$ as an $\varepsilon_t$-approximate solution such that 
\begin{align}\label{eq:approximate solution}
    \KL(\hpi^{t+1},\pi^{t+1}) \le \varepsilon_t = O\InParentheses{\frac{1}{t^4}}.
\end{align}
We note that we can compute $\pi^{t+1}$ within $\varepsilon_t$ error using $O(\log \frac{1}{\varepsilon_t}) = O(\log t)$ iterations of Algorithm~\ref{alg:regularized game} (\Cref{thm:linear-strong-monotone}). 

We denote $\Pi^\star$ the set of Nash equilibria such that each $\pi^\star \in \Pi^\star$ has support $\supp(\pi^\star) = \supp(\piinit)$ as guaranteed by \Cref{assumption:support}. We introduce a few quantities that depend on the Nash equilibria and the initial policy. 
\begin{definition}\label{dfn:values}
    We define the following constants.
    \begin{enumerate}
        \item $p_\mathrm{sft}:= \max\{ p > 0:  \forall y \in \supp(\piinit) , \piinit(y) \ge p \}$; $D := |\+Y| \log \frac{1}{p_\mathrm{sft}}$ so that $\KL(\pi||\piinit) \le D$ for all $\pi \in \Pi$
         \item $p_{\min}:= \max\{p > 0: \exists \pi^\star \in \Pi^\star, \forall y \in \supp(\piinit), \pi^\star(y) \ge p \}$; Let $\pi^\star \in \Pi^\star$ be a Nash equilibrium so that $\pi^\star(y) \ge p_{\min}$ holds for all $y$ in its support.
        \item $c_1  := \frac{p_{\min}}{\exp{(D+2)/p_{\min}}}$ and $c_2 := \frac{c_1}{\exp(1/c_1)}$.
    \end{enumerate}
\end{definition}

Our main result is that if each optimization problem at iteration $t$ can be solved within approximation error $\varepsilon_t \le \frac{c_1}{3t^2}$, then \LPO converges in last-iterate to a Nash equilibrium. 
\begin{theorem}[\LPO with approximate regularized game solver]\label{theorem:last-iterate-app}
    Assume \Cref{assumption:support} holds. If in each iteration $t \ge 1$, the returned iterate $\pi^{t+1}$ is an $\varepsilon_t$-approximate solution to $J_\tau(\pi_1, \pi_2, \pi^t)$ as defined in \eqref{eq:approximate solution} with $\varepsilon_t \le\frac{c_1^2}{9t^4}$ ($c_1$ defined in \Cref{dfn:values}), then $\{x^t\}$ converges to a Nash equilibrium of $J(\pi_1,\pi_2)$.
\end{theorem}

We need the following technical lemma in the proof of \Cref{theorem:last-iterate-app}.

\begin{lemma}\label{lemma:single-step-app}
    Let $\varepsilon_t \le \frac{c_1^2}{9t^4}$.  Then for all $t \ge 1$,
    \begin{enumerate}
        \item $\KL(\pi^\star|| \pi^{t+1}) \le \KL(\pi^\star || \pi^t) - \KL(\pi^{t+1}||\pi^t) + \frac{1}{t^2}$.
        \item $\min_{y \in \supp(\piinit)}\pi^{t}(y) \ge c_2$.
        \item $\lim_{t\rightarrow \infty} \InNorms{\pi^{t+1} - \pi^t} = 0$.
        \item For any Nash equilibrium $\hpi \in \Pi$ and $t \ge 1$, we have $\KL(\hpi||\pi^{t+1}) \le \KL(\hpi||\pi^t)  + \frac{1}{t^2}$
    \end{enumerate}
\end{lemma}
\begin{proof}
    By \Cref{lem:single-step}, we have $\hpi^{t+1} = \Prox(\pi^t, \-P(\cdot \succ \hpi^{t+1}))$ and 
    \begin{align}\label{eq:exact descent}
        \KL(\pi^\star|| \hpi^{t+1}) \le \KL(\pi^\star || \pi^t) - \KL(\hpi^{t+1}||\pi^t).
    \end{align}
    The above implies 
    \begin{align}\label{eq:approximate descent}
        \KL(\pi^\star|| \pi^{t+1}) &\le \KL(\pi^\star || \pi^t) - \KL(\pi^{t+1}||\pi^t) + \underbrace{\KL(\pi^\star|| \pi^{t+1}) - \KL(\pi^\star || \hpi^{t+1})}_{E_1} \nonumber
        \\ \quad &+ \underbrace{\KL(\pi^{t+1} || \pi^t) - \KL(\hpi^{t+1} || \pi^{t})}_{E_2}.
    \end{align}

    Now, we use induction to prove the claim. For the base case, we define $\pi^0 := \pi^1$ and $\varepsilon_t = 0$, then
    \paragraph{Base Case: $t = 0$} Since $\pi^0 = \pi^1$, we have $\KL(\pi^1||\pi^0) = 0$. Then it is clear that 
    \begin{align*}
        \KL(\pi^\star||\pi^1) \le \KL(\pi^\star||\pi^0) - \KL(\pi^1||\pi^0) .
    \end{align*}
    Moreover, by \Cref{prop:KL->prob lower bound} and $D \ge \KL(\pi^\star||\piinit)$, we have $\min_{y \in \supp(\pi^{1})}\pi^{1}(y) \ge c_1 \ge c_2$.
    \paragraph{Induction: $t \ge 1$} 
    We have
    \begin{align*}
        \KL(\pi^\star|| \hpi^{t+1}) &\le \KL(\pi^\star|| \pi^t) \tag{\eqref{eq:exact descent}} \\
        &\le \KL(\pi^\star || \piinit) + \sum_{t=1}^{t-1} \frac{1}{t^2}\tag{inductive hypothesis} \\
        &\le D + 2. \tag{$D \ge \KL(\pi^\star||\piinit)$}
    \end{align*}
    Using \Cref{prop:KL->prob lower bound}, we have $\min_{y \in \supp(\piinit)}\hpi^{t+1}(y) \ge c_1$. By $\KL(\hpi^{t+1}|| \pi^{t+1}) \le \varepsilon_t \le 1$ and \Cref{prop:KL->prob lower bound} again, we get $\min_{y \in \supp(\piinit)}\pi^{t+1}(y) \ge c_2:= \frac{c_1}{\exp(1/c_1)}$. Thus, both $\hpi^{t+1}$ and $\pi^{t+1}$ are bounded away from the boundary in their support. Further by $\KL(\hpi^{t+1}|| \pi^{t+1}) \le \varepsilon_t$, we have 
    \begin{align*}
        \sum_y \hpi^{t+1}(y) \log \frac{\hpi^{t+1}(y)}{\pi^{t+1}(y)} \le \varepsilon_t 
        \Rightarrow \max_{y} \log \frac{\hpi^{t+1}(y)}{\pi^{t+1}(y)} \le \frac{\varepsilon_t}{c_1}.
    \end{align*}

    As a result, we can bound
    \begin{align*}
        E_1 &= \KL(\pi^\star|| \pi^{t+1}) - \KL(\pi^\star || \hpi^{t+1}) \\
        &= \sum_y \pi^\star(y) \log \frac{\hpi^{t+1}(y)}{\pi^{t+1}(y)} \\
        &\le \max_{y} \log \frac{\hpi^{t+1}(y)}{\pi^{t+1}(y)} \\
        &\le \frac{\varepsilon_t}{c_1}.
    \end{align*}
    Moreover, we have 
    \begin{align*}
        E_2 &= \KL(\pi^{t+1} || \pi^t) - \KL(\hpi^{t+1} || \pi^{t})\\
        &= \sum_y (\pi^{t+1}(y) - \hpi^{t+1}(y))\log \frac{\pi^{t+1}(y)}{\pi^t(y)} - \KL(\hpi^{t+1}||\pi^{t+1}) \\
        &\le \InNorms{\pi^{t+1} - \hpi^{t+1}}_1 \cdot \max_y |\log \frac{\pi^{t+1}(y)}{\pi^t(y)}| \\
        &\le \sqrt{\KL(\hpi^{t+1}||\pi^{t+1})} \cdot \log \frac{1}{c_2} \tag{Pinsker's Inequality} \\
        &\le \frac{2\sqrt{\varepsilon_t}}{c_1}
    \end{align*}
    Combining the above two inequalities  with \eqref{eq:approximate descent} and noting the fact that $\varepsilon_t \le \sqrt{\varepsilon_t}$ gives 
    \[
    \KL(\pi^\star|| \pi^{t+1}) \le \KL(\pi^\star || \pi^t) - \KL(\pi^{t+1}||\pi^t) + \frac{3\sqrt{\varepsilon_t}}{c_1}.
    \]
    We conclude the claim since $\varepsilon_t \le  \frac{c_1^2}{9t^4}$. This completes the proof for item 1 and item 2.

    For item 3, we have $\sum_{t=1}^\infty \InNorms{\pi^{t+1} - \pi^t} \le \sum_{t=1}^\infty \KL(\pi^{t+1}||\pi^t) \le D+2$. Thus $\lim_{t\rightarrow \infty} \InNorms{\pi^{t+1} - \pi^t} = 0$.

    For item 4, we can use \Cref{lem:single-step} and $\hpi^{t+1} = \Prox(\pi^t, \-P(\cdot \succ \hpi^{t+1}))$ to get
    \begin{align}
        \KL(\hpi|| \pi^{t+1}) &\le \KL(\hpi|| \pi^t) - \KL(\pi^{t+1}||\pi^t) + \underbrace{\KL(\hpi|| \pi^{t+1}) - \KL(\hpi || \hpi^{t+1})}_{E_1} \nonumber
        \\ \quad &+ \underbrace{\KL(\pi^{t+1} || \pi^t) - \KL(\hpi^{t+1} || \pi^{t})}_{E_2}.
    \end{align}
    We note that $E_2 \le \frac{2\sqrt{\varepsilon_t}}{c_1}$ has been proved in the above. For $E_1$, we have \begin{align*}
        E_1 &= \KL(\hpi|| \pi^{t+1}) - \KL(\hpi || \hpi^{t+1}) \\
        &= \sum_y \hpi(y) \log \frac{\hpi^{t+1}(y)}{\pi^{t+1}(y)} \\
        &\le \max_{y} \log \frac{\hpi^{t+1}(y)}{\pi^{t+1}(y)} \\
        &\le \frac{\varepsilon_t}{c_1}. 
    \end{align*}
    Thus we have $\KL(\hpi|| \pi^{t+1}) \le \KL(\hpi|| \pi^t) + \frac{1}{t^2}$ as $\varepsilon_t \le \frac{c_1^2}{9t^4}$.
\end{proof}

\paragraph{Proof of \Cref{theorem:last-iterate-app}}
\begin{proof}
     Since the sequence $\{\pi^t\}$ is bounded, it has at least one limit point $\hat{\pi}$. By item 2 in \Cref{lemma:single-step-app}, we know $\hat{\pi}(y) \ge c_2$ for all $y \in \supp(\piinit)$. By item 3 in \Cref{lemma:single-step-app}, we have $\lim_{t\rightarrow \infty} \InNorms{\pi^{t+1} - \pi^t} = 0$. Denote $\{\pi^k : k \in \kappa\}$ a subsequence that converges to $\hat{\pi}$. Then we have 
     \begin{align*}
         \hpi &= \lim_{ k\in \kappa, \kappa \rightarrow \infty} \pi^{k+1}\\ &= \lim_{k\in \kappa, \kappa \rightarrow \infty} \hpi^{k+1} \tag{$\KL(\hpi^{k+1}, \pi^{k+1}) \le \varepsilon_k$ and $\lim_{t \rightarrow \infty} \varepsilon_t = 0$} \\
         & = \lim_{k\in \kappa, \kappa \rightarrow \infty} \Prox(\pi^k, \frac{1}{\tau} \-P(\cdot \succ \hpi^{k+1})) \\
         & = \lim_{k\in \kappa, \kappa \rightarrow \infty} \Prox(\pi^{k+1}, \frac{1}{\tau} \-P(\cdot \succ \hpi^{k+1})) \tag{$\lim_{t\rightarrow \infty} \InNorms{\pi^{t+1} - \pi^t} = 0$} \\
         & = \lim_{k\in \kappa, \kappa \rightarrow \infty} \Prox(\pi^{k+1}, \frac{1}{\tau} \-P(\cdot \succ \pi^{k+1})) \tag{$\KL(\hpi^{k+1}, \pi^{k+1}) \le \varepsilon_k$ and $\lim_{t \rightarrow \infty} \varepsilon_t = 0$} \\
         & = \Prox(\hpi, \frac{1}{\tau} \-P(\cdot \succ \hpi)).
     \end{align*}
     Since $\hpi$ is a fixed point of $\Prox(\pi, \frac{1}{\tau} \-P(\cdot \succ \pi))$ and $\supp(\hpi) = \supp(\piinit)$, we can use the same proof in \Cref{lemma:fixed point to NE} to show that $\hpi$ is a Nash equilibrium of $J(\pi_1, \pi_2)$. 

     Given that $\hpi$ is a Nash equilibrium of the original game, we can apply item 4 in \Cref{lemma:single-step-app} and get
     \begin{align*}
         \KL(\hpi||\pi^{t+1}) \le \KL(\hpi||\pi^t) + \frac{1}{t^2}.
     \end{align*}
     Now we show the sequence $\{x^t\}$ converges to $\hpi$. 
     Fix any $\epsilon > 0$.  Let $T_1 \ge 1$ such that $\sum_{t=T_1}^\infty \frac{1}{t^2} < \frac{\epsilon}{2}$, Since $\hpi$ is a limit point of $\{x^t\}$, there exists $T_2 \ge T_1$ such that $\KL(\hpi|| \pi^{T_2}) \le \frac{\epsilon}{2}$. Then for any $t \ge T_2$, we have
     \begin{align*}
         \KL(\hpi||\pi^{t+1}) \le \KL(\hpi|| \pi^{T_2}) + \sum_{t=T_2}^\infty \frac{1}{t^2} \le \frac{\epsilon}{2} + \frac{\epsilon}{2} = \epsilon.
     \end{align*}
     Since the above holds for any $\varepsilon > 0$, we know $\lim_{t \rightarrow \infty} \KL(\hpi || \pi^t) = 0$ and thus $\{x^t\}$ converges to $\hpi$. This completes the proof.
\end{proof}

\subsection{Auxiliary propostion}
\begin{proposition}\label{prop:KL->prob lower bound}
    Let $\pi_1$ and $\pi_2$ be two distributions with the same support. If there exists $p, D > 0$ such that $\min_{y \in \supp(\pi_1)} \pi_1(y) \ge p$ and $\KL(\pi_1||\pi_2) \le D$, then $\supp(\pi_2) = \supp(\pi_1)$ and
    \begin{align*}
        \min_{y \in \supp(\pi_1)} \pi_2(y) \ge \frac{p}{\exp(D/p)}.
    \end{align*}
\end{proposition}
\begin{proof}
    We have 
    \begin{align*}
        \KL(\pi_1 || \pi_2) \le D &\Rightarrow p \log \frac{p}{\pi_2(y)} \le D, \forall y \in \supp(\pi_1) \\
        &\Rightarrow \pi_2(y) \ge \frac{ p}{\exp(D/p)}, \forall y \in \supp(\pi_1).
    \end{align*}
\end{proof}

\section{Proof of \Cref{thm:linear-strong-monotone}}\label{app:regularized}
We show that MWU (Algorithm~\ref{alg:regularized game}) has linear convergence to the unique Nash equilibrium of a KL-regularized zero-sum game $J(\pi_1, \pi_2, \piref)$. We denote $\mu^\star = \pi^\star_\tau$ its unique Nash equilibrium. Our proof is inspired by~\citep[Lemma F.1]{pmlr-v235-abe24a} that give linear convergence of MWU in KL-regularized game. Here, we include a simpler proof with slightly better constants for our setting for completeness.

We prove the following descent lemma, which immediately implies \Cref{thm:linear-strong-monotone}.
\begin{lemma}
    If we choose $\eta \in (0, \frac{\tau}{\tau^2 + \frac{1}{2}}]$ in MWU (Algorithm~\ref{alg:regularized game}), then we have for every $k \ge 1$
    \begin{align*}
        \KL(\mu^\star, \mu^{k+1}) \le \InParentheses{1-\frac{\eta\tau}{2}}\KL(\mu^\star, \mu^k).
    \end{align*}
\end{lemma}
\begin{proof}
We define the gradient operator $G: \Pi \rightarrow \-R^{|\+Y|}$ of $J(\pi_1, \pi_2)$ and the gradient operator $A: \Pi \rightarrow \-R^{|\+Y|}$ of the KL regularization $\KL(\pi, \piref)$ as follows.
\begin{align*}
    G(\pi) &:= \-P(\cdot \succ \pi) \\
    A(\pi) &:= \nabla_\pi \KL(\pi, \piref) = \log\frac{\pi(\cdot)}{\piref(\cdot)}.
\end{align*}
We define the composite operator $F = G - \tau A $. Then MWU update in Algorithm~\ref{alg:regularized game} is equivalent to 
\begin{align*}
    \mu^{k+1} = \Prox(\mu^k, \eta F(\mu^k)).
\end{align*}

Using \Cref{coro:three point inequality}, we have 
\begin{align*}
    &\InAngles{\eta F(\mu^k), \mu^\star - \mu^{k+1}} \le \KL(\mu^\star||\mu^k) - \KL(\mu^\star||\mu^{k+1}) - \KL(\mu^{k+1}||\mu^k)
\end{align*}
We focus on the left-hand side of the above inequality. Since $\mu^\star$ is a Nash equilibrium of the regularized game with gradient $F$, we have $\InAngles{\eta F(\mu^\star), \mu^\star - \mu^{k+1}} \ge 0$ and thus
\begin{align*}
     &\InAngles{\eta F(\mu^k), \mu^\star - \mu^{k+1}} \\&\ge  \InAngles{\eta F(\mu^k), \mu^\star - \mu^{k+1}} - \InAngles{\eta F(\mu^\star), \mu^\star - \mu^{k+1}} \\
     &= \underbrace{\eta\InAngles{G(\mu^k) - G(\mu^{k+1}), \mu^\star - \mu^{k+1}}}_{\mathrm{term_1}} +\underbrace{ \eta \tau\InAngles{A(\mu^k) - A(\mu^\star), \mu^{k+1} - \mu^\star}}_{\mathrm{term_2}} \\
     & \quad + \underbrace{\eta \InAngles{G(\mu^{k+1})- G(\mu^\star), \mu^\star - \mu^{k+1}}}_{\mathrm{term_3}=0}.
\end{align*}
We note that $\mathrm{term_3} = 0$ since $G$ is the gradient of a zero-sum game:
\begin{align*}
    &\InAngles{G(\mu^{k+1})- G(\mu^\star), \mu^\star - \mu^{k+1}}\\
    &= \-P(\mu^\star \succ \mu^{k+1}) + \-P(\mu^{k+1} \succ \mu^\star) - \frac{1}{2} - \frac{1}{2} = 0.
\end{align*}
For $\mathrm{term_2}$, we can apply the three-point identity for the Bregman divergence as follows:
\begin{align*}
    \mathrm{term_2} &=  \eta \tau\InAngles{A(\mu^k) - A(\mu^\star), \mu^{k+1} - \mu^\star} \\
    &= \eta\tau \InAngles{ \log \frac{\mu^k}{\mu^\star}, \mu^{k+1} -\mu^\star} \\
    &= \eta \tau \InParentheses{ \KL(\mu^\star || \mu^k) -\KL(\mu^{k+1}||\mu^k) + \KL(\mu^{k+1}|| \mu^\star) } \\
    &\ge \eta \tau \InParentheses{ \KL(\mu^\star || \mu^k) -\KL(\mu^{k+1}||\mu^k)}. 
\end{align*}
For $\mathrm{term_1}$, we will use the $1$-Lipschitzness of $G$ and Cauchy-Swarz inequality:
\begin{align*}
    \mathrm{term_1} &= \eta\InAngles{G(\mu^k) - G(\mu^{k+1}), \mu^\star - \mu^{k+1}} \\
    & \ge -\eta \InParentheses{ \frac{1}{2\tau}\InNorms{G(\mu^k) - G(\mu^{k+1})}^2_\infty + \frac{\tau}{2}\InNorms{\mu^\star - \mu^{k+1}}_1^2} \\
    &\ge -\eta \InParentheses{ \frac{1}{2\tau}\InNorms{\mu^k - \mu^{k+1}}^2_1 + \frac{\tau}{2}\InNorms{\mu^\star - \mu^{k+1}}^2_1} \tag{$G$ is $1$-Lipschitz} \\
    &\ge -  \frac{\eta}{2\tau}\KL(\mu^{k+1} || \mu^k) - \frac{\eta\tau}{2} \KL(\mu^\star|| \mu^{k+1})
\end{align*}
Combining the above gives 
\begin{align*}
    (1 - \frac{\eta\tau}{2}) \KL(\mu^\star || \mu^{k+1}) \le (1 - \eta \tau) \KL(\mu^\star || \mu^k) - (1 - \eta \tau - \frac{\eta}{2\tau}) \KL(\mu^{k+1} ||\mu^k)
\end{align*}
Let $\eta \le \frac{1}{\tau + \frac{1}{2\tau}} = \frac{\tau}{\tau^2 + \frac{1}{2}}$, then we have $1 - \eta \tau - \frac{\eta}{2\tau} \ge 0$ and thus
\begin{align*}
    \KL(\mu^\star || \mu^{k+1})  \le \frac{1-\eta\tau}{1-\frac{\eta\tau}{2}} \KL(\mu^\star || \mu^k) \le \InParentheses{1 - \frac{\eta \tau}{2}} \KL(\mu^\star || \mu^k).
\end{align*}
This completes the proof.
\end{proof}

\section{Computing the Prox Operator using Preference Learning Methods}\label{app:prox}
We include additional examples showing how existing algorithms designed for RLHF and preference optimization with neural network parameters can be adapted to solve the prox operator $\Prox(\pi, \eta g)$ ($\eta > 0$ is the step size). These algorithms include RL algorithms like PPO and loss-minimization algorithms like DPO, IPO, SPPO, DRO, INPO, each of which may be preferred in certain settings.

\paragraph{Reinforcement Learning algorithms} We can use the Proximal Policy Optimization (PPO) algorithm~\citep{schulman2017proximal} or Group-Relative Policy optimization (GRPO)~\citep{shao2024deepseekmath, guo2025deepseek} to solve $\Prox(\pi, \eta g)$. Observe that
\begin{align*}
    \Prox(\pi, \eta g) &= \argmax_{\pi'} \{ \InAngles{\eta g, \pi'} -\KL(\pi' || \pi)\} \\
    &= \argmax_{\pi'} \-E_{y \sim \pi'} \InBrackets{ g[y] - \eta^{-1} \cdot \KL(\pi' || \pi) }
\end{align*}
shares the same form as the objective in \eqref{eq:RLHF}. Typically, we parameterize $\pi' = \pi_\theta$ with neural network parameters $\theta$ and optimize over $\theta$.

\paragraph{Loss minimization algorithms} Let us denote $\hat{\pi}$ the prox operator $\Prox(\pi, \eta g)$, then we have
\begin{align*}
    \hat{\pi}[y]= \frac{\pi(y)\exp(\eta g(y))}{Z} \Leftrightarrow \log\frac{\hat{\pi}(y)}{\pi(y)} - \eta g(y) + \log Z = 0,
\end{align*}
where $Z = \-E_{y \sim \pi}[\exp(\eta g(y))]$ is the partition function. We can directly compute the partition function $Z$ and thus $\hat{\pi}$ in small tabular cases. However, the partition function is hard to compute in general large-scale applications. Several works have recently proposed to solve the above equality by optimizing the corresponding $L_2$ loss. 

The Self-Play Preference Optimization (SPPO) loss~\citep{wu2024self} assumes $\log Z = \frac{\eta}{2}$ and optimizes
\begin{align*}
    \ell_\SPPO (\theta) = \InParentheses{ \log\frac{\pi_\theta(y)}{\pi(y)} - \eta g(y) - \frac{\eta}{2}}^2.
\end{align*}
The Direct Reward Optimization (DRO) loss~\citep{richemond2024offline} parameterizes both $\hat{\pi}$ and $\log Z$ with $\theta$ and $V_\phi$ respectively and optimize\footnote{we modified some constants in the original DRO loss to make it consistent with our presentation. The modification has no other effects.}
\begin{align*}
    \ell_\DRO (\theta, \phi) = \InParentheses{ \log\frac{\pi_\theta(y)}{\pi(y)} - \eta g(y) - \eta V_\phi}^2.
\end{align*}
The REBEL loss~\citep{GaoCZOSBJBLS} uses \emph{differences in rewards} to eliminate the partition function $Z$ and optimize the regression loss
\begin{align*}
    \ell_\REBEL(\theta) = \InParentheses{ \eta^{-1} \InParentheses{ \log \frac{\pi_\theta(y)}{\pi(y)} - \log \frac{\pi_\theta(y')}{\pi(y')}  }  - \InParentheses{g(y) - g(y')} }^2.
\end{align*}
All the above approaches can be used to solve $\Prox(\pi, \eta g)$. However, directly applying them iteratively on $J(\pi_1, \pi_2)$ is equivalent to running MWU, which provably diverges. In contrast, we can apply them in Algorithm~\ref{alg:regularized game} and then apply our meta-algorithm \LPO to guarantee convergence to a Nash equilibrium.
\begin{remark}
    The above approaches are versatile and work well for any $g$ that can be evaluated efficiently. In particular, we should consider using them when (1) $g = r$ is a reward function and we can efficiently query $r$; (2) $g = \-P(\cdot \mid \mu)$ is the win rate against a reference policy $\mu$, and we can efficiently sample from $\mu$ and have oracle access to $\-P$. These two setting are popular and practical in the LLM alignment setting.
\end{remark}

Now we turn attention to the more specific setting where $g$ corresponds to a preference model $\-P$ (could be a BT model or a general preference) and that we can collect a win-loss preference data set $\+D = \{(y_w, y_l)\}$, which is standard for LLM alignment. Although the abovementioned algorithms apply, they all require estimating $g$ (the win rate) and may be inefficient in practice. In the following, we present algorithms directly working on the sampled dataset $\+D$ without further estimation.

\paragraph{Sampled loss based on the BT preference model} %
Assume $g = r$ is the reward of the Bradley-Terry model, and the dataset $\{(y_w, y_l)\}$ consists of win-lose pairs of responses. Then we can solve $\Prox(\pi, \eta g)$ by optimize the DPO loss~\citep{rafailov2024direct} defined as
\begin{align*}
    \ell_\DPO((y_w, y_l);\theta) = - \log \sigma\InParentheses{\eta^{-1} \log \frac{\pi_\theta(y_w)}{\pi(y_w)} - \eta^{-1} \log \frac{\pi_\theta(y_l)}{\pi(y_l)} }.
\end{align*}

\paragraph{Sampled loss for general preference}
The DPO loss inspires many other loss functions that work under even weaker assumptions on the preference model. Now, we assume a general preference model $\-P$ over $\+Y$ (not necessarily the BT model). We assume $g$ is the win-rate against some policy $\mu$ such that $g_\mu(y)= \-P[y \succ \mu] := \-E_{y' \sim \mu}[\-P[y \succ y']]$ (think of $\mu$ as the reference policy $\pi_{\mathrm{ref}}$ or other online policy $\pi_t$). We assume the dataset contains win-lose pairs sampled from $\mu$: $\{y_w, y_l \sim \mu\}$.  We denote the preference distribution $\lambda_\-P(y,y')$ as a binary distribution: 
\begin{align}\label{eq:lamda}
    \lambda_\-P(y, y') = \begin{cases}
        (y, y') \text{ w.p. $\-P[y \succ y']$} \\
        (y', y) \text{ w.p. $1 - \-P[y \succ y']$} \\
    \end{cases}
\end{align}

\paragraph{IPO for computing $\Prox$ for unregularized preferences} we first show that the IPO loss could be used to solve $\pi_\theta = \Prox(\pi, \eta g_\mu)$ where $g$ is the unregularized win-rate against a reference policy $\mu$ such that $g_\mu(y)= \-P[y \succ \mu] := \-E_{y' \sim \mu}[\-P[y \succ y']]$. %
Given a dataset of win-lose pairs sampled from $\mu$: $\{y_w, y_l \sim \mu\}$, the (population) IPO loss~\citep{azar2024general} $\ell_\IPO(\theta,\mu)$ is defined as 
\begin{equation}
\label{eq:ipo-loss}
 \-E_{ \substack{ (y_w, y_l)\sim \mu \\ (y^+, y^-)\sim \lambda_{\-P}(y_w, y_l) \eqref{eq:lamda} } }\InBrackets{\InParentheses{\log \frac{\pi_\theta(y^+)}{\pi_\theta(y^-)} - \log \frac{\pi(y^+)}{\pi(y^-)}  - \frac{\eta}{2}}^2}.
\end{equation}

\citet{azar2024general} have shown that the minimizer of the $\ell_\IPO(\theta, \mu)$ satisfies $
    \pi_\theta(y) \propto \pi(y) \exp\InParentheses{-\eta \-P[y \succ \mu]} \Leftrightarrow \pi_\theta = \Prox(\pi, \eta g_\mu).
$
Thus we can compute $\Prox(\pi, \eta g_\mu)$ where $g_\mu = \-P(\cdot \succ \mu)$ by minimizing the IPO loss. %

\paragraph{INPO for computing $\Prox$ for regularized preferences} The Iterative Nash Policy Optimization (INPO) loss~\citep{zhang2024iterative} is a generalization of the IPO loss to the regularized preference setting. We show that INPO could be used to compute $\Prox(\mu, \eta g^\tau_\mu)$, where $g^\tau_\mu := \nabla_{\pi} J_\tau(\pi, \mu, \piref) = \-P(\cdot \succ \mu) - \tau \log \frac{\mu(\cdot)}{\piref(\cdot)}$ is the gradient of the regularized objective \eqref{eq: regularized objective}.  
Given a win-loss pair data set $\{y_w, y_l \sim \mu\}$, the INPO loss $\ell_\INPO(\pi)$ is defined as

\begin{align}\label{eq:INPO}
    &\ell_\INPO(\pi) := \mathbb{E}_{\substack{ (y_w, y_l)\sim \mu \\(y^+, y^-)\sim\lambda_{\mathbb{P}}(y_w, y_l) \eqref{eq:lamda}}} 
    \Bigg[\bigg(\log \frac{\pi(y^+)}{\pi(y^-)}  - \eta\tau\log \frac{\piref(y^+)}{\piref(y^-)}
    - (1-\eta\tau) \log\frac{\mu(y^+)}{\mu(y^-)} 
    - \frac{\eta}{2} \bigg)^2 \Bigg].
\end{align}
It has been proved that the minimizer of the INPO loss is $\Prox(\mu, \eta g^\tau_\mu)$~\citep{zhang2024iterative}. Thus we can use INPO in Algorithm~\ref{alg:regularized game} as a regularized game solver, as we show in Algorithm~\ref{alg:INPO}.

\subsection{COMAL integrated with INPO}
\begin{algorithm}[!ht]
\LinesNotNumbered
\caption{INPO~\citep{zhang2024iterative} for solving $J_\tau(\pi_1,\pi_2,\piref)$}\label{alg:INPO}
    \KwIn{Reference policy $\piref$, regularization $\tau > 0$, step size $\eta > 0$, number of iterations $K \ge 1$, preference oracle $\mathbb{P}$.}
    \KwOut{Approximate regularized NE policy $\mu^K$}
    Initialize $\mu^1 \leftarrow \piref$ \\
    \For{$k = 1, 2, \ldots, K-1$}{
        Generate response pairs $\{(y^{(i)}_1, y^{(i)}_2) \sim \mu^k\}_{i=1}^n$ \\
        Query preference oracle $\mathbb{P}$ to get preference data $\+D_k = \{y^{(i)}_w, y^{(i)}_l\}_{i=1}^n$ \\
        Compute $\mu^{k+1} = \argmin_{\pi\in \Pi} \mathbb{E}_{\+D_k} \ell_\INPO(\pi)$ \eqref{eq:INPO}
    }
    \textbf{return} $\mu^K$
\end{algorithm}

\textbf{Practical Implementation of \LPO}
We present an implementation of \LPO in Algorithm~\ref{alg:LPO-practical} using the INPO~\citep{zhang2024iterative} algorithm as a subgame solver. We remark that \LPO can also be implemented using PPO or many other preference learning algorithms, as we show in \cref{app:prox} and \Cref{app:COMAL-SPPODROREBEL}. Given the implementation of these existing methods, our meta-algorithm requires minimal change but achieves last-iterate convergence to a Nash equilibrium.

In practice, COMAL provides guidance for performing iterative preference optimization: the reference policy needs to be updated in order to avoid the performance upper bound imposed by a relatively weak reference policy, however, the reference policy should not be updated at each optimization step to avoid training instability.

\section{More Practical Implementations of \LPO}
\label{app:COMAL-SPPODROREBEL}
In this section, we provide more practical implementations of \LPO using iterative GRPO~\citep{shao2024deepseekmath, guo2025deepseek},  the SPPO loss~\citep{wu2024self}, the DRO loss~\citep{richemond2024offline}, and the REBEL loss~\citep{GaoCZOSBJBLS}. All these implementations demonstrate that COMAL is simple and versatile and can be integrated with many existing methods designed
for preference optimization with minimal changes.

Although the SPPO, DPO, and REBEL losses are proposed in the unregularized preference setting, we have shown how to extend these losses to compute the prox operator even for KL-regularized preferences in \Cref{app:prox}. Thus, we can integrate these losses for computing the prox operator in Algorithm~\ref{alg:regularized game} for solving the regularized game $J_\tau(\pi_1, \pi_2, \piref)$. As a result, we get the practical implementation of \LPO by using different regularized game solvers.

We omit the instruction $x \sim \rho \in \Delta(\+X)$ for notation simplicity in the following implementations. Generalization to the contextual setting is straightforward.

\subsection{Practical Implementation of COMAL using iterative GRPO~\citep{shao2024deepseekmath}}
We observe that the iterative GRPO procedure used in DeepSeekMath~\citep{shao2024deepseekmath} and DeepSeek-R1~\citep{guo2025deepseek} aligns closely with COMAL’s design principles: iterative GPRO updates the reference policy model to the latest policy model every few steps (every 400 steps in DeepSeek-R1~\citep{guo2025deepseek}), and each step solves a regularized objective. To adapt iterative GRPO to preference alignment, one simply instantiates the reward with the win-rate induced by a preference oracle $\mathbb{P}$. We include the full algorithm for completeness below in Algorithm~\ref{alg:COMAL-GRPO}. For the GRPO objective, we can use either the original objective \citep{shao2024deepseekmath} or the unbiased Dr.GRPO objective without length and std normalization~\citep{liu2025understanding}.
\begin{equation}\label{eq:GRPO}
\begin{aligned}
\mathcal{J}_{\mathrm{GRPO}}(\theta)
&= \mathbb{E}_{\{y^{(i)}\}_{i=1}^{G} \sim \pi_{\mathrm{old}}}\Bigg[
\frac{1}{G}\sum_{i=1}^{G}\frac{1}{\lvert y^{(i)}\rvert}\sum_{l=1}^{\lvert y^{(i)}\rvert}
\Bigg\{
\min\Bigg(
\frac{\pi_{\theta}\!\big(y^{(i)}_{l}\mid y^{(i)}_{<l}\big)}
     {\pi_{\mathrm{old}}\!\big(y^{(i)}_{l}\mid y^{(i)}_{<l}\big)}\,
\hat{A}_{i,l},
\\[-2pt]
&\qquad\qquad
\operatorname{clip}\!\Bigg(
\frac{\pi_{\theta}\!\big(y^{(i)}_{l}\mid y^{(i)}_{<l}\big)}
     {\pi_{\mathrm{old}}\!\big(y^{(i)}_{l}\mid y^{(i)}_{<l}\big)},
\,1-\varepsilon,\,1+\varepsilon
\Bigg)\hat{A}_{i,l}
\Bigg)
- \tau_t\,\mathbb{D}_{\mathrm{KL}}\!\big[\pi_{\theta}\,\|\,\pi_{\mathrm{ref}}\big]
\Bigg\}
\Bigg].
\end{aligned}
\end{equation}

We remark that COMAL (Algorithm~\ref{alg:main}) is a meta-algorithm that can be instantiated with any algorithm that solves the regularized game in each iteration and guarantees convergence to an exact Nash equilibrium. While we focus on using Mirror Descent (Algorithm~\ref{alg:regularized game}) for solving the regularized game and present most implementations using MD, we can also use the clairvoyant implementation of conceptual prox~\citep{farina2022clairvoyant}, where our convergence result (\Cref{thm: main LINEPO last-iterate}) still applies. Iterative GRPO for the alignment game can be seen as the clairvoyant implementation of the conceptual prox algorithm.

\begin{algorithm}[!ht]
\LinesNotNumbered
    \caption{Practical Implementation of \LPO using iterative GRPO~\citep{shao2024deepseekmath}}\label{alg:COMAL-GRPO}
    \KwIn{Initial policy $\piinit$, regularization $\{\tau_t > 0\}$, number of iterations $T \ge 1$, number of inner optimization steps $\{K_t \ge 1\}$, preference oracle $\mathbb{P}$, hyperparameter $\varepsilon$.}
    \KwOut{Optimized policy $\pi^T$}
    Initialize $\pi^1, \pi_\theta, \piref \leftarrow \piinit$ \\
    \For{$t = 1, 2, \ldots, T-1$}{     
        reference policy $\piref \leftarrow \pi^{t}$ \\
        \For{step $k=1, \ldots K_t$ }{
        Update the old policy $\pi_{\mathrm{old}}  \leftarrow \pi_\theta$\\
        Sample $G$ responses $\{y^{(i)}\}_{i=1}^G \sim \pi_{\mathrm{old}}$\\
        Query preference oracle $\mathbb{P}$ to compute the reward, i.e., empirical win-rate $r_i :=\widehat{P}[y^{(i)} \succ \pi_{\mathrm{old}}] =\frac{1}{G}\sum_{j=1}^G \mathbb{P}[y^{(i)} \succ y^{(j)}]$ for each sample $y^{(i)}$\\
        Compute $\hat{A}_{i,l}$ for the $l$-th token of $y^{(i)}$ through group relative advantage estimation.\\
        Update the policy $\pi_\theta$ by maximizing the GRPO objective~\eqref{eq:GRPO}.
        }
        $\pi^{t+1} \leftarrow \pi_\theta $\\
    }
    \textbf{return} $\pi^T$
\end{algorithm}

\subsection{COMAL integrated with SPPO~\citep{wu2024self}}
We present \regSPPO (Algorithm~\ref{alg:SPPO}) for solving a KL-regularized game $J_\tau(\pi_1, \pi_2,\piref)$, which is the instantiation of Algorithm~\ref{alg:regularized game} using the SPPO loss. Then, we give a practical implementation of \LPO integrated with the SPPO loss in Algorithm~\ref{alg:COMAL-SPPO}.
\begin{algorithm}[!ht]
\LinesNotNumbered
    \caption{\regSPPO: Extension of SPPO~\citep{wu2024self} for solving KL-regularized games}\label{alg:SPPO}
    \KwIn{Reference policy $\piref$, regularization $\tau > 0$, step size $\eta > 0$, number of rounds $K \ge 1$, preference oracle $\mathbb{P}$.}
    \KwOut{Approximate regularized Nash equilibrium policy $\mu_K$}
    Initialize $\mu^1 \leftarrow \piref$ \\
    \For{$k = 1, 2, \ldots, K-1$}{
        Generate responses $\{y^{(i)} \sim \mu^k \}_{i=1}^n$ \\
        Query preference oracle $\mathbb{P}$ to annotate the win-rate $\mathbb{P}[y^{(i)} \succ y^{(j)}], \forall i,j \in [n]$ \\
        Form dataset $\+D_t = \{(y^{(i)}, \widehat{P}[y^{(i)} \succ \mu^k])\}_{i \in [n]}$\\
        Compute $\mu^{k+1} = \mu_{\theta^{k+1}}$ where
        \resizebox{1.0\hsize}{!}{%
$\begin{aligned}
         &\theta^{k+1} = \argmin_\theta \ell_\SPPO(\theta) := \mathbb{E}_{(y, \widehat{P}[y \succ \mu^k])\sim \+D_t}\InBrackets{\InParentheses{ \log \frac{\mu_\theta(y)}{\mu^k(y)} - \eta \InParentheses{ \widehat{P}[y \succ \mu^k] - \tau \log\frac{\mu^k(y)}{\piref(y)}  -\frac{1}{2}  } }^2}
\end{aligned}$%
}
    }
    \textbf{return} $\mu^K$
\end{algorithm}
\begin{algorithm}[!ht]
\LinesNotNumbered
    \caption{Practical Implementation of \LPO integrated with \regSPPO (Algorithm~\ref{alg:SPPO})}\label{alg:COMAL-SPPO}
    \KwIn{Initial policy $\piinit$, regularization $\{\tau_t > 0\}$, step size $\{\eta_t > 0\}$, number of iterations $T \ge 1$, number of inner optimization steps $\{K_t \ge 1\}$, preference oracle $\mathbb{P}$.}
    \KwOut{Optimized policy $\pi^T$}
    Initialize $\pi^1, \piref \leftarrow \piinit$ \\
    \For{$t = 1, 2, \ldots, T-1$}{
        $\pi^{t+1} \leftarrow \regSPPO(\piref, \tau_t, \eta_t, K_t, \mathbb{P})$ defined in Algorithm~\ref{alg:SPPO}\\
        $\piref \leftarrow \pi^{t+1}$ 
    }
    \textbf{return} $\pi^T$
\end{algorithm}

\subsection{COMAL integrated with DRO~\citep{richemond2024offline}}
We present \regDRO (Algorithm~\ref{alg:DRO}) for solving a KL regularized game $J_\tau(\pi_1, \pi_2,\piref)$, which is the instantiation of Algorithm~\ref{alg:regularized game} using the DRO loss. Then, we give a practical implementation of \LPO integrated with the DRO loss in Algorithm~\ref{alg:COMAL-DRO}.
\begin{algorithm}[!ht]
\LinesNotNumbered
    \caption{\regDRO: Extension of DRO~\citep{richemond2024offline} for solving KL-regularized games}\label{alg:DRO}
    \KwIn{Reference policy $\piref$, regularization $\tau > 0$, step size $\eta > 0$, number of rounds $K \ge 1$, preference oracle $\mathbb{P}$.}
    \KwOut{Approximate regularized Nash equilibrium policy $\mu_K$}
    Initialize $\mu^1 \leftarrow \piref$ \\
    \For{$k = 1, 2, \ldots, K-1$}{
        Generate responses $\{y^{(i)} \sim \mu^k \}_{i=1}^n$ \\
        Query preference oracle $\mathbb{P}$ to annotate the win-rate $\mathbb{P}[y^{(i)} \succ y^{(j)}], \forall i,j \in [n]$ \\
        Form dataset $\+D_t = \{(y^{(i)}, \widehat{P}[y^{(i)} \succ \mu^k])\}_{i \in [n]}$\\
        Compute $\mu^{k+1} = \mu_{\theta^{k+1}}$ where
        \resizebox{1.0\hsize}{!}{%
$\begin{aligned}
         &\theta^{k+1} = \argmin_\theta \min_\phi \ell_\DRO(\theta) := \mathbb{E}_{(y, \widehat{P}[y \succ \mu^k])\sim \+D_t}\InBrackets{\InParentheses{ \log \frac{\mu_\theta(y)}{\mu^k(y)} - \eta \InParentheses{ \widehat{P}[y \succ \mu^k] - \tau \log\frac{\mu^k(y)}{\piref(y)}} - \eta V_\phi }^2}
\end{aligned}$%
}
    }
    \textbf{return} $\mu^K$
\end{algorithm}
\begin{algorithm}[!ht]
\LinesNotNumbered
    \caption{Practical Implementation of \LPO integrated with \regDRO (Algorithm~\ref{alg:DRO})}\label{alg:COMAL-DRO}
    \KwIn{Initial policy $\piinit$, regularization $\{\tau_t > 0\}$, step size $\{\eta_t > 0\}$, number of iterations $T \ge 1$, number of inner optimization steps $\{K_t \ge 1\}$, preference oracle $\mathbb{P}$.}
    \KwOut{Optimized policy $\pi^T$}
    Initialize $\pi^1, \piref \leftarrow \piinit$ \\
    \For{$t = 1, 2, \ldots, T-1$}{
        $\pi^{t+1} \leftarrow \regDRO(\piref, \tau_t, \eta_t, K_t, \mathbb{P})$ defined in Algorithm~\ref{alg:SPPO}\\
        $\piref \leftarrow \pi^{t+1}$ 
    }
    \textbf{return} $\pi^T$
\end{algorithm}

\subsection{COMAL integrated with REBEL~\citep{GaoCZOSBJBLS}}
We present \regREBEL (Algorithm~\ref{alg:REBEL}) for solving a KL regularized game $J_\tau(\pi_1, \pi_2,\piref)$, which is the instantiation of Algorithm~\ref{alg:regularized game} using the REBEL loss. Then, we give a practical implementation of \LPO  (Algorithm~\ref{alg:main}) integrated with the REBEL loss in Algorithm~\ref{alg:COMAL-REBEL}.
\begin{algorithm}[!ht]
\LinesNotNumbered
    \caption{\regREBEL: Extension of REBEL~\citep{GaoCZOSBJBLS} for solving KL-regularized games}\label{alg:REBEL}
    \KwIn{Reference policy $\piref$, regularization $\tau > 0$, step size $\eta > 0$, number of rounds $K \ge 1$, preference oracle $\mathbb{P}$.}
    \KwOut{Approximate regularized Nash equilibrium policy $\mu_K$}
    Initialize $\mu^1 \leftarrow \piref$ \\
    \For{$k = 1, 2, \ldots, K-1$}{
        Generate responses $\{y^{(i)} \sim \mu^k \}_{i=1}^n$ \\
        Query preference oracle $\mathbb{P}$ to annotate the win-rate $\mathbb{P}[y^{(i)} \succ y^{(j)}], \forall i,j \in [n]$ \\
        Form dataset $\+D_t = \{(y^{(i)}, y^{(j)}, \widehat{P}[y^{(i)} \succ \mu^k], \widehat{P}[y^{(j)} \succ \mu^k])\}_{i,j \in [n]}$\\
        Compute $\mu^{k+1} = \mu_{\theta^{k+1}}$ where
         $$\theta^{k+1} = \argmin_\theta  \ell_\REBEL(\theta)$$
        \resizebox{1.0\hsize}{!}{%
$\begin{aligned}
        \ell_\REBEL(\theta) := \mathbb{E}_{(y,y')\sim \+D_t}\InBrackets{\InParentheses{\eta^{-1}\InParentheses{ \log \frac{\mu_\theta(y)}{\mu^k(y)} - \log \frac{\mu_\theta(y')}{\mu^k(y')}} - \InParentheses{ \widehat{P}[y \succ \mu^k] - \tau \log\frac{\mu^k(y)}{\piref(y)} - \widehat{P}[y' \succ \mu^k] + \tau \log\frac{\mu^k(y')}{\piref(y')} }}^2}
\end{aligned}$%
}
    }
    \textbf{return} $\mu^K$
\end{algorithm}
\begin{algorithm}[!ht]
\LinesNotNumbered
    \caption{Practical Implementation of \LPO integrated with \regREBEL (Algorithm~\ref{alg:REBEL})}\label{alg:COMAL-REBEL}
    \KwIn{Initial policy $\piinit$, regularization $\{\tau_t > 0\}$, step size $\{\eta_t > 0\}$, number of iterations $T \ge 1$, number of inner optimization steps $\{K_t \ge 1\}$, preference oracle $\mathbb{P}$.}
    \KwOut{Optimized policy $\pi^T$}
    Initialize $\pi^1, \piref \leftarrow \piinit$ \\
    \For{$t = 1, 2, \ldots, T-1$}{
        $\pi^{t+1} \leftarrow \regREBEL(\piref, \tau_t, \eta_t, K_t, \mathbb{P})$ defined in Algorithm~\ref{alg:REBEL}\\
        $\piref \leftarrow \pi^{t+1}$ 
    }
    \textbf{return} $\pi^T$
\end{algorithm}

\section{Implementation of Mirror-Prox and Optimistic Multiplicative Weights Update}\label{app:MD&OMWU}
We note that there are other algorithms that has provable last-iterate convergence to Nash equilibrium in (unregularized) zero-sum games, including the \MP algorithm~\citep{nemirovski2004prox} and Optimistic Multiplicative Weights Update (OMWU) algorithm~\citep{rakhlin2013optimization, syrgkanis2015fast, hsieh2021adaptive}. We present practical implementations of these two algorithms in the context of LLM alignment for solving $J(\pi_1, \pi_2)$~\eqref{eq:game objective}, where we use preference optimization algorithms to solve the prox operator as shown in \S\ref{sec:prox operator} and \Cref{app:prox}. 

We denote the gradient $g(\pi):=\-P(\cdot \succ \pi)$. 
\paragraph{Mirror-Prox} The \MP algorithm~\citep{nemirovski2004prox} initialized $\pi^1 = \piinit$ and updates in each iteration $t \ge 1$:
\begin{align*}
    \pi^{t+\half} &= \Prox(\pi^t, \eta g(\pi^t)) \\
    \pi^{t+1} &= \Prox(\pi^t, \eta g(\pi^{t+\half}))
\end{align*}
We can implement \MP using PPO/DPO/IPO/SPPO/DRO/REBEL to compute the prox operator. Specifically,  we could sample from $\pi^t$ and construct a preference dataset $D_t$ and optimize certain regression loss (IPO/DRO/REBEL) to compute $\pi^{t+\half} = \Prox(\pi^t, \eta g(\pi^t))$. The procedure applies to the second step in each iteration. Thus in such an implementation, we require two sampling and two optimization procedures in each iteration. 
 
\paragraph{Optimistic Multiplicative Weights Update (OMWU)} The OMWU algorithm~\citep{rakhlin2013optimization} is an optimistic variant of the MWU algorithm. Although MWU diverges in zero-sum games, it has been shown that OMWU has last-iterate convergence to Nash equilibrium~\citep{wei2021linear,hsieh2021adaptive}. Initialized with $\pi^1 = \pi^{\frac{1}{2}} = \piinit$, OMWU updates in each iteration $t \ge 1$:
\begin{align*}
    \pi^{t+\half} &= \Prox(\pi^t, \eta g(\pi^{t-\half})) \\
    \pi^{t+1} &= \Prox(\pi^t, \eta g(\pi^{t+\half}))
\end{align*}
Similarly, we can implement OMWU to solve $J(\pi_1, \pi_2)$ using preference methods to compute the prox operator as shown in \S\ref{sec:prox operator}. Moreover, OMWU has an equivalent update rule: initialize $\pi^1 = \pi^0 = \piinit$
\begin{align*}
    \pi^{t+1} = \Prox(\pi^t, 2\eta g(\pi^t) -\eta g(\pi^{t-1})),
\end{align*}
which requires computing only one prox operator in each iteration. 

We leave a systematic evaluation of \MP\ and OMWU at a large scale, including LLM alignment, to future work.

\section{Synthetic Experiments}\label{app:synthetic}

\begin{figure}[h] % 
    \centering
    \includegraphics[width=0.22\textwidth]{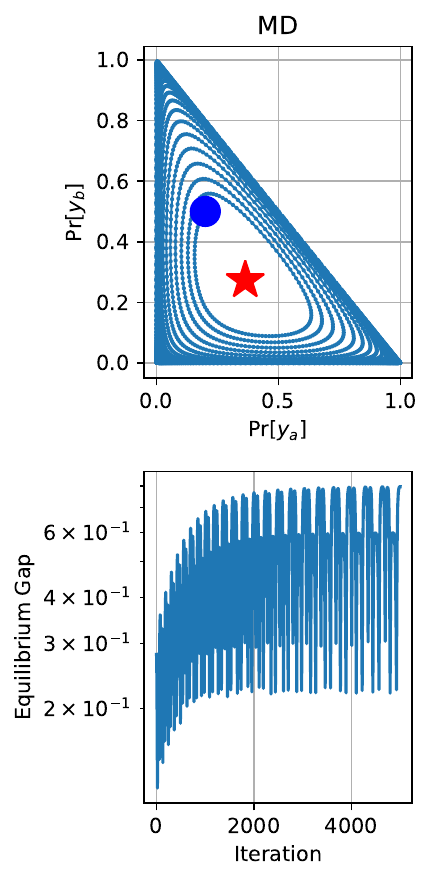} % 
    \includegraphics[width=0.22\textwidth]{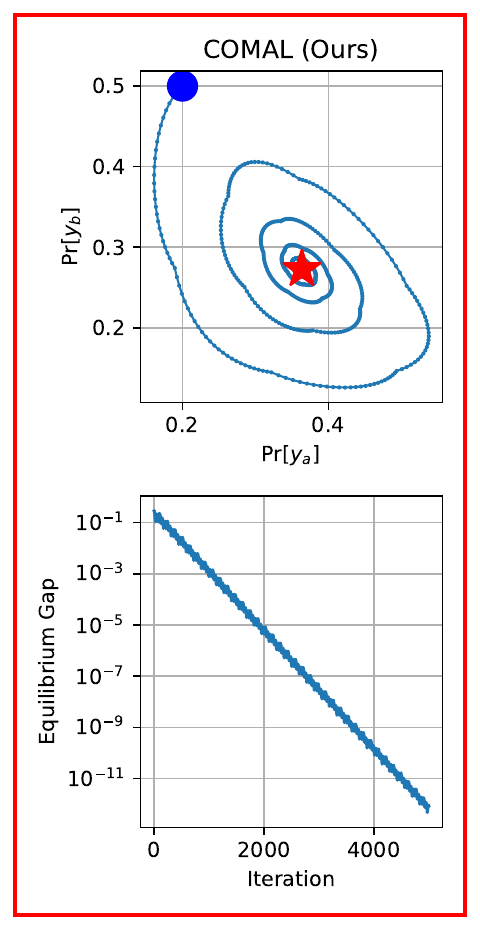} % 
    \caption{Dyanmics on a simple $3$-dimensional preference game. The unique Nash equilibrium is $[4/11, 3/11, 4/11]$ represented as red star. We initialize all algorithms at the blue dot point $[0.2, 0.5, 0.3]$.}
    \label{fig:synthetic gradient}
    %\vspace{-10pt}
\end{figure}

\begin{figure}[h]
    \centering
    \includegraphics[scale=0.45]{figsv3/Iter_DPO_iters100_eta0.3_prob_lower1e-05.pdf}
    \includegraphics[scale=0.45]{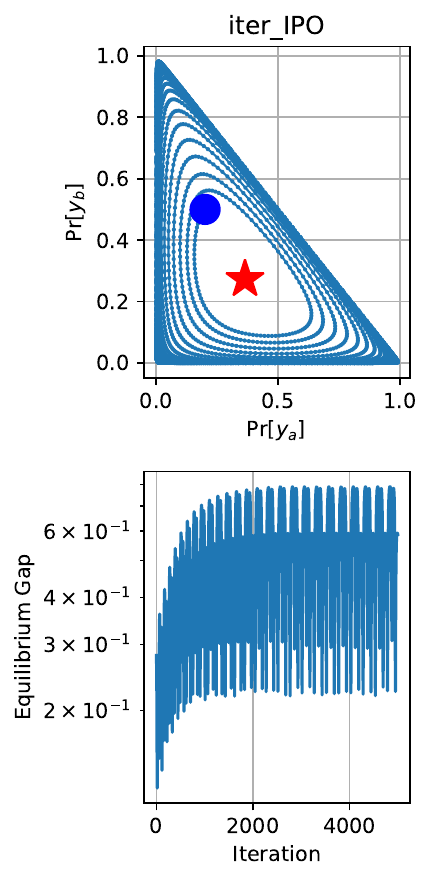}
    \includegraphics[scale=0.45]{figsv3/SPPO_iters5000_eta0.3_prob_lower1e-05.pdf}
    \includegraphics[scale=0.45]{figsv3/INPO_iters1000_eta0.3_tau0.1.pdf}
    \caption{Dyanmics on a simple $3$-dimensional preference game. The unique Nash equilibrium is $[4/11, 3/11, 4/11]$ represented as red star. We initialize all algorithms at the blue dot point $[0.2, 0.5, 0.3]$. }
    \label{fig: synthetic sample}
\end{figure}

\paragraph{Experiment Setup} Recall that we set $\-P[y_b \succ y_a] = \-P[y_c \succ y_b] = 0.9$ and $\-P[y_a \succ y_c] = 0.8$. This results in the following zero-sum game: we have policies $\Pi = \Delta(\{y_a, y_b, y_c\})$ and objective
\begin{align*}
    J(\pi_1, \pi_2) = \pi_1^\top A \pi_2 - 0.5, \textnormal{ where } A = \begin{bmatrix}
        0.5 & 0.1 & 0.8 \\
        0.9 & 0.5 & 0.1 \\
        0.2 & 0.9 & 0.5
    \end{bmatrix}.
\end{align*}

The game has a unique Nash equilibrium $[4/11, 3/11, 4/11]$. We set the initial policy to be $\pi^1 = [0.2, 0.5, 0.3]$ for all algorithms. We choose $\eta = 0.3$ for iterative DPO, iterative IPO, and SPPO. We choose $\eta = 0.3$ and $\tau = 0.1$ for INPO and \LPO. For \LPO (Algorithm~\ref{alg:LPO-practical}), we set $ T= 200$ and $K_t = 25$ so the total number of iterations is $T \cdot K_t = 5000$.

\paragraph{Experiments using noiseless gradient} We present numerical results of mirror-descent (MD) algorithms (equivalent to MWU) and \LPO(Algorithm~\ref{alg:main}) in \Cref{fig:synthetic gradient}. We can see that the MD algorithm diverges from the unique Nash equilibrium and suffers a large equilibrium gap, while \LPO achieves fast last-iterate convergence to the Nash equilibrium, aligned with our theoretical results (\Cref{thm: main LINEPO last-iterate}).

\paragraph{Experiements using preference samples} Since the popular iterative DPO algorithm does not contain a gradient step, we also conduct experiments with only Oracle query access to the preference model. We compare the performance of various algorithms, including iterative DPO, iterative IPO, SPPO, and INPO and present results in \Cref{fig: synthetic sample}. 
The sample-only setting is also more aligned with what happens in practice. We use a sufficient number of samples in each iteration for every algorithm. As a result, the \LPO performs the same as in the noiseless gradient setting, while the iterative IPO algorithm becomes equivalent to the MD algorithm. We note the following:

    \textit{Iterative DPO:} We observe that iterative DPO cycles between extreme policies (e.g., outputting $y_a$ with probability close to $1$). This is aligned with ~\citep{azar2024general}, where they found  DPO will converge to the deterministic policy regardless of the regularization parameter in extreme preference settings. The cycling behavior of iterative DPO may be explained as follows: in each iteration, DPO converges to a nearly deterministic policy output $y$; then the new preference data shows that $y' \ne y$ is more preferred; finally, iterative DPO cycles over $\+Y$ since the preference itself exhibits a cycle and there is no clear winner.
    
    \textit{Iterative IPO}~\citep{azar2024general,calandriello2024human}: The IPO loss is a variant of the DPO loss, but it does not rely on the BT model assumption and works for a general preference model. However, as we have discussed before, (exactly) minimizing the IPO loss is equivalent to performing one MD step, and thus, iterative IPO is equivalent to MD up to sampling error. As a result, we observe that iterative IPO also exhibits cycling behavior.
    
    \textit{SPPO}~\citep{wu2024self}: The SPPO algorithm (see \Cref{app:prox}) is not exactly the same as MWU since SPPO assumes the partition function is always $Z = \log \frac{\eta}{2}$ which may not be the case. We observe that SPPO exhibits very similar cycling behavior as MD. We conclude that SPPO approximates MD very well in this instance and exhibits similar behavior.

    \textit{INPO}~\citep{zhang2024iterative}: The INPO algorithm is designed for finding the Nash equilibrium of the KL regularized game $J_\tau(\pi_1, \pi_2, \piref)$. As we proved in \Cref{thm:linear-strong-monotone}, INPO does not diverge and exhibits last-iterate convergence. However, it converges to a point that differs from the Nash equilibrium of the game $J(\pi_1,\pi_2)$ and has constant equilibrium gap. %

\section{Hyperparameters and Training Details for LLM Experiments}
\label{app:hyper}
We follow a similar training recipe proposed in \citet{tunstall2023zephyr} for the experiments.
Specifically, at each training iteration, the models are fine-tuned for one epoch with a batch size of 32 and a maximum learning rate of $5 \times 10^{-7}$, using a cosine learning rate scheduler with 10\% of warmup steps.
We conduct a grid search for the strength of the KL regularization, $\eta^{-1}$, in the loss functions of DPO, IPO and INPO, within the range of 0.001 - 0.1.
INPO has another hyper-parameter $\tau$ which controls the strength of the KL regularization from the reference policy.
Its value is determined following \citet{zhang2024iterative}, where $\eta\tau$ is set to a fixed ratio, $1/3$.

\section{LLM-Based Experiments with 1.5B LLM}
\label{app:llm-1.5b}
In \S\ref{sec:llm}, we conduct experiments using an 8B LLM, Llama-3-8B-Instruct.
Here, we provide additional experiments with a pre-trained smaller LLM, Qwen2-1.5B~\citep{yang2024qwen2}.
Its smaller size allows us to perform more training iterations.

\subsection{Experimental Settings}
\label{subsec:exp-setting-1.5b}
Some of the experiment settings are identical to the settings in \S\ref{sec:llm}.
Therefore, here we only outline the differences in the settings.

\compactparagraph{Preference Oracle}
The preference oracle we used is  Llama-3-OffsetBias-8B~\citep{park2024offsetbias}, which is a pairwise preference model that predicts which output is better given an instruction and a pair of outputs.
Fine-tuned from Meta-Llama-3-8B-Instruct~\citep{dubey2024llama3herdmodels}, it achieves strong performance on various human preference alignment benchmarks in RewardBench~\citep{lambert2024rewardbench}.
We selected it as the preference oracle for its balance of computational efficiency and alignment with human preferences, making it suitable for iterative preference optimization.

\compactparagraph{Baselines}
We include the following baselines for comparisons with \LPO: 
(1) SFT, which fine-tunes the pre-trained Qwen2-1.5B on the UltraChat dataset, with the resulting checkpoint serving as the starting point and/or reference policy for the other training algorithms;
(2) vanilla DPO~\citep{rafailov2024direct} and (3) vanilla IPO~\citep{azar2024general}, where one training iteration is performed over the entire instruction set of UltraFeedback with output pairs sampled from the SFT policy;
(4) INPO~\citep{zhang2024iterative}, where each iteration of training is performed on a single data split;
(5) iterative IPO, which follows a training setting similar to INPO but without the KL regularization with respect to the reference policy.

\compactparagraph{Evaluations}
We chose to use the same preference oracle used during data generation, Llama-3-OffsetBias-8B, as the evaluator. 
This decision was made to maintain a controlled experimental setting, ensuring that the preference oracle the model learns to fit is also the one used to evaluate its performance.

\compactparagraph{Training Details}
We follow the training recipe proposed in \citet{tunstall2023zephyr} for the experiments.
Specifically, at each training iteration, the models are fine-tuned for 3 epochs with a batch size of 32 and maximum learning rate of $5 \times 10^{-7}$, using a linear learning rate scheduler where 10\% of the steps are for warmup and the rest for linearly decreasing the rate.
The checkpoints are selected based on their validation loss on the UltraFeedback dataset.
The training is performed on 8 NVIDIA A6000 Ada GPUs with 48GB memory, and one training iteration over the 10K instructions takes around 5 hours.
Due to the relatively high computational requirements and the large number of training iterations we tested (up to 42), we opted to use a moderately sized LLM and did not conduct an exhaustive hyper-parameter search, instead referencing settings from previous work when appropriate.

\compactparagraph{Hyper-Parameters} We conduct a grid search for the strength of the KL regularization, $\eta^{-1}$, in both vanilla DPO and IPO.
We found that DPO achieves the best performance when $\eta^{-1}$ is set to 0.01, while IPO achieves the best performance when $\eta^{-1}$ is set within the range of 0.002 - 0.01.
We then choose the value of $\eta^{-1}$ to be 0.002 to encourage larger learning steps.
This value of $\eta$ is also used for iterative IPO.
For INPO, we compare two settings where $\eta^{-1}$ is set to 0.002 and 0.01, corresponding to a small and a large regularization respectively.
INPO has another hyper-parameter $\tau$ which controls the strength of the KL regularization from the reference policy.
We determine its value following the setting of \citet{zhang2024iterative}, where $\eta\tau$ is set to a fixed ratio, $1/3$.
Regarding \LPO, which is implemented based on INPO as outlined in Algorithm~\ref{alg:LPO-practical}, $\eta^{-1}$ is also set to 0.002 at the beginning of the training.
The reference policy used in \LPO is updated when the first optimization step begins to converge or overfit, and $\eta^{-1}$ is increased to 0.01 to improve training stability.

\begin{figure}[h]
    \centering
    \includegraphics[scale=0.50]{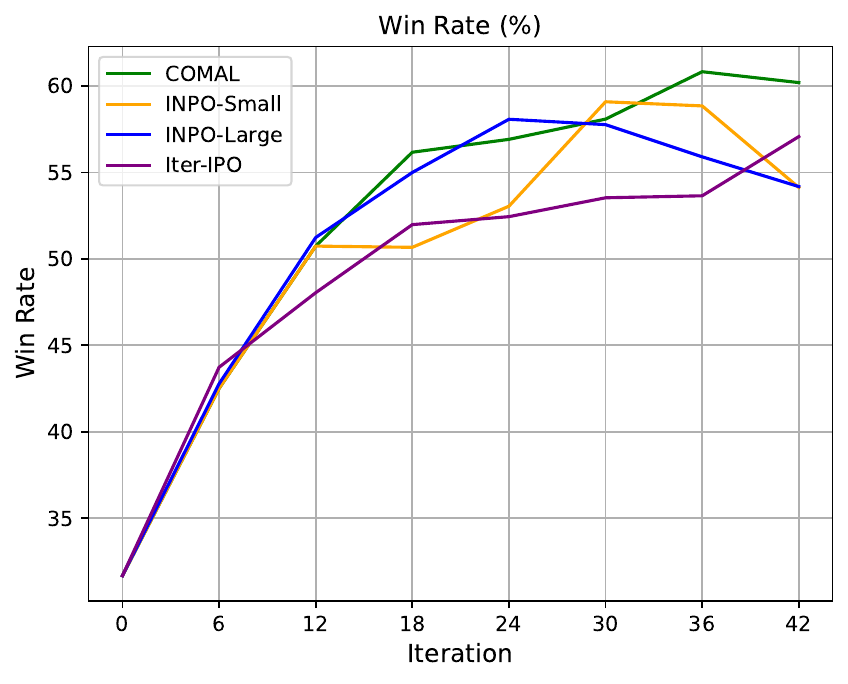}
    \includegraphics[scale=0.50]{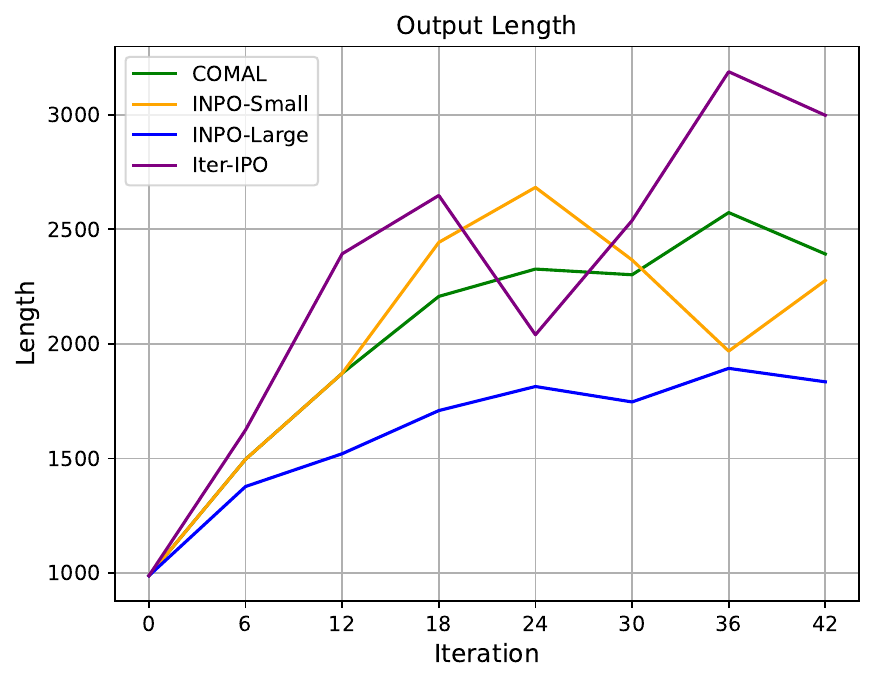}
    %\vspace{-8pt}
    \caption{Comparisons of Iterative IPO (Iter-IPO), INPO, and \LPO. The average win rates of the trained checkpoints against the best checkpoints of each training algorithm, and the average lengths of the outputs are compared. 
    For INPO, two variations with a small regularization ($\eta^{-1}=0.002$, INPO-Small) and a large regularization ($\eta^{-1}=0.01$, INPO-Large) are compared.
    }
    \label{fig:main-comparison-1.5b}
\end{figure}

\begin{table}[t]
% \small
\caption{Performance comparison of different training algorithms. The row v.s. column win rate (\%) is reported. 
The \textit{best} checkpoints produced by each training algorithm are compared.
For INPO, we report two variations with a small regularization ($\eta^{-1}=0.002$, INPO-Small) and a large regularization ($\eta^{-1}=0.01$, INPO-Large).}
\begin{center}
% \addtolength{\tabcolsep}{-1pt} 
\begin{tabular}{@{}lrrrrrrrr@{}} \toprule
  Row/Column   &   SFT &   DPO &   IPO &   Iter-IPO &   INPO-Large &   INPO-Small &   \LPO &   Avg \\
\midrule
Iter-IPO     & 67.33 & 62.36 & 58.76 &      50.00 &        48.20 &        44.72 &   44.10 & 53.64 \\
 INPO-Large   & \textbf{77.02} & 69.81 & 67.83 &      51.80 &        50.00 &        46.21 &   44.84 & 58.22 \\
 INPO-Small   & 73.66 & 66.21 & 66.46 &      55.28 &        53.79 &        50.00 &   48.70 & 59.16 \\
 COMAL        & 74.53 & \textbf{70.56} & \textbf{68.82} &      \textbf{55.90} &        \textbf{55.16} &        \textbf{51.30} &   \textbf{50.00} & \textbf{60.90} \\
 \bottomrule
% \hline
% %\vspace{-20pt}
\end{tabular}
% \addtolength{\tabcolsep}{+1pt} 
\end{center}
\label{tab:main-comparison-1.5b}
\end{table}

\begin{table}[t]
% \small
\caption{Performance comparison of different training algorithms. The row v.s. column win rate (\%) is reported. 
The \textit{last} checkpoints produced by each training algorithm are compared.
For INPO, we report two variations with a small regularization ($\eta^{-1}=0.002$, INPO-Small) and a large regularization ($\eta^{-1}=0.01$, INPO-Large).}
\begin{center}
% \addtolength{\tabcolsep}{-1pt} 
\begin{tabular}{@{}lrrrrrrrr@{}} \toprule
  Row/Column   &   SFT &   DPO &   IPO &   Iter-IPO &   INPO-Large &   INPO-Small &   \LPO &   Avg \\
\midrule
Iter-IPO     & 67.33 & 62.36 & 58.76 &      50.00 &        50.93 &        49.07 &   45.47 & 54.84 \\
 INPO-Large   & 70.43 & 62.98 & 61.61 &      49.07 &        50.00 &        48.07 &   41.61 & 54.83 \\
 INPO-Small   & 68.57 & 61.12 & 59.88 &      50.93 &        51.93 &        50.00 &   43.23 & 55.09 \\
 COMAL        & \textbf{74.53} & \textbf{67.83} & \textbf{65.09} &      \textbf{54.53} &        \textbf{58.39} &        \textbf{56.77} &   \textbf{50.00} & \textbf{61.02} \\
 \bottomrule
% \hline
% %\vspace{-20pt}
\end{tabular}
% \addtolength{\tabcolsep}{+1pt} 
\end{center}
\label{tab:add-comparison}
\end{table}

\subsection{Result Analysis}
\label{subsec:exp-result-1.5b}

Figure~\ref{fig:main-comparison-1.5b} presents the training dynamics of three iterative preference optimization algorithms we compared: iterative IPO (Iter-IPO), INPO with a small and a large regularization (INPO-Small and INPO-Large), and \LPO, which are demonstrated by their checkpoints' win rates against the \emph{best checkpoints} produced by 7 different algorithms: SFT, IPO, DPO, Iter-IPO, INPO-Small, INPO-Large, \LPO, and the average lengths of their outputs.
For INPO and \LPO, the model is trained for up to 42 iterations, equivalent to 7 training rounds over the entire instruction set since it has been split into 6 subsets.
We note that:

\noindent (1) Iter-IPO shows a quicker improvement rate at the beginning of the training, but its performance begins to lag behind other algorithms after the first training round with a rapid increase in output length, which indicates the inherent instability of this training algorithm.

\noindent (2) INPO achieves stronger performance and larger improvement rates compared to Iter-IPO.
However, the win rates of both INPO-Small and INPO-Large start to decrease after 5 training rounds.
We suspect this suggests that INPO has started to converge and/or overfit.
Moreover, for INPO-Small, its performance shows only a minor improvement and even a slight decline during training rounds 2 to 4 (iterations 12 - 24).
Therefore, for \LPO, which shares the same training trajectory as INPO-Small for the first two training rounds, we update the reference policy at the beginning of the third training round, following the optimization process described in Algorithm~\ref{alg:LPO-practical}.

\noindent (3) \LPO is able to further improve the model performance with the updated reference policy.
Notably, its performance continues to improve up until the 6th training round, when the other algorithms begin to degrade, demonstrating the benefit of updating the reference policy.

Table~\ref{tab:main-comparison-1.5b} provides pairwise comparisons between the \textit{best} checkpoints of the iterative preference optimization algorithms and a few baselines.
It demonstrates the clear advantage of \LPO, which is able to achieve a win rate that is strictly above 50\%  against all the other checkpoints.
The comparison of the \textit{final} checkpoints of different algorithms after the last iteration is presented in Table~\ref{tab:add-comparison}, where \LPO is able to achieve significantly better performance thanks to its stability.

\section{Limitations}
\label{app:limitations}
\paragraph{Provable Guarantee on Computing the Prox Operators} Our theoretical guarantee on the last-iterate convergence of COMAL relies on computing the prox operator and solving a regularized game approximately. Although we provide many practical loss minimization approaches that compute the prox operator, the applicability of our results in practical LLM settings lacks a provable guarantee since the losses could be highly non-convex, for which no provably efficient algorithms exist. We also remark that our analysis is non-trivial and novel, which gives a more robust guarantee than existing works~\citep{perolat2021poincare, sokota2023a, pmlr-v235-abe24a} that require solving the regularized game \emph{exactly}.

\paragraph{Theoretical Convergence Guarantees} Our \Cref{thm: main LINEPO last-iterate} provides asymptotic last-iterate convergence to exact Nash equilibrium and \Cref{thm:linear-strong-monotone} gives non-asymptotic $\title{O}(1/\varepsilon^2)$ convergence to an $\varepsilon$-approximate Nash equilibrium when we choose the regularization $\tau = O(\varepsilon)$. Here we discuss the possibility of achieving non-asymptotic convergence with non-vanishing regularization $\tau$. We remark that there are algorithms with $\ell_2$ regularization that have polynomial last-iterate convergence rates~\citep{cai2022finite}. However, it is unclear whether these algorithms with $\ell_2$ regularization are practical in large-scale LLM settings, as no known efficient implementations exist. In contrast, prox operators with entropy regularization can be computed using practical preference optimization algorithms such as DPO, IPO, INPO as we discussed in~\S\ref{app:prox}. Regarding the possibility of establishing a last-iterate convergence rate of our algorithm, we note that a uniform convergence rate for algorithms of this type is unlikely, as suggested by a recent work~\citep{cai2024fast}. Here, “uniform” refers to an upper bound on the duality gap that holds for all instances. While it may be possible to obtain a weaker instance-dependent rate, similar to the ones in~\citep{wei2021linear}, under a unique Nash equilibrium assumption, the rate depends on a problem-dependent parameter that could be arbitrarily large and difficult to characterize—particularly in the LLM setting. As such, such rates offer limited practical guidance for implementation or for understanding convergence speed in realistic scenarios. Nevertheless, getting a convergence rate is an interesting question and we leave it for future work.

\end{document}

%% file: macro.tex
\newtheorem{definition}{Definition}
\newtheorem{theorem}{Theorem}
\newtheorem*{theorem*}{Theorem}
\newtheorem{lemma}{Lemma}

\newtheorem{corollary}{Corollary}

\newtheorem{assumption}{Assumption}

\newtheorem{proposition}{Proposition}
\newtheorem{remark}{Remark}

\usepackage{pifont}
\newcommand{\cmark}{\ding{51}}%
\newcommand{\xmark}{\ding{55}}%

\DeclareMathOperator*{\argmin}{argmin}
\DeclareMathOperator*{\argmax}{argmax}

\DeclareMathOperator{\KL}{KL}
\DeclareMathOperator{\Prox}{Prox}

\def\+#1{\mathcal{#1}}
\def\-#1{\mathbb{#1}}

\newcommand{\notshow}[1]{{}}
\newcommand{\AutoAdjust}[3]{{ \mathchoice{ \left #1 #2  \right #3}{#1 #2 #3}{#1 #2 #3}{#1 #2 #3} }}
\newcommand{\Xcomment}[1]{{}}

\newcommand{\InParentheses}[1]{\AutoAdjust{(}{#1}{)}}
\newcommand{\InBrackets}[1]{\AutoAdjust{[}{#1}{]}}
\newcommand{\InAngles}[1]{\AutoAdjust{\langle}{#1}{\rangle}}
\newcommand{\InNorms}[1]{\AutoAdjust{\|}{#1}{\|}}
\renewcommand{\part}[2]{\frac{\partial #1}{\partial #2}}

\newcommand{\half}{\frac{1}{2}}

\newcommand{\supp}{\mathrm{supp}}

%% file: main.bbl
\begin{thebibliography}{76}
\providecommand{\natexlab}[1]{#1}
\providecommand{\url}[1]{\texttt{#1}}
\expandafter\ifx\csname urlstyle\endcsname\relax
  \providecommand{\doi}[1]{doi: #1}\else
  \providecommand{\doi}{doi: \begingroup \urlstyle{rm}\Url}\fi

\bibitem[Abe et~al.(2024)Abe, Ariu, Sakamoto, and Iwasaki]{pmlr-v235-abe24a}
Kenshi Abe, Kaito Ariu, Mitsuki Sakamoto, and Atsushi Iwasaki.
\newblock Adaptively perturbed mirror descent for learning in games.
\newblock In \emph{Proceedings of the 41st International Conference on Machine Learning}, 2024.

\bibitem[Arora et~al.(2012)Arora, Hazan, and Kale]{arora2012multiplicative}
Sanjeev Arora, Elad Hazan, and Satyen Kale.
\newblock The multiplicative weights update method: a meta-algorithm and applications.
\newblock \emph{Theory of computing}, 8\penalty0 (1):\penalty0 121--164, 2012.

\bibitem[Azar et~al.(2024)Azar, Guo, Piot, Munos, Rowland, Valko, and Calandriello]{azar2024general}
Mohammad~Gheshlaghi Azar, Zhaohan~Daniel Guo, Bilal Piot, Remi Munos, Mark Rowland, Michal Valko, and Daniele Calandriello.
\newblock A general theoretical paradigm to understand learning from human preferences.
\newblock In \emph{International Conference on Artificial Intelligence and Statistics}, pages 4447--4455. PMLR, 2024.

\bibitem[Bauschke et~al.(2021)Bauschke, Moursi, and Wang]{BauschkeMW21}
Heinz~H. Bauschke, Walaa~M. Moursi, and Xianfu Wang.
\newblock Generalized monotone operators and their averaged resolvents.
\newblock \emph{Math. Program.}, 189\penalty0 (1):\penalty0 55--74, 2021.
\newblock \doi{10.1007/S10107-020-01500-6}.
\newblock URL \url{https://doi.org/10.1007/s10107-020-01500-6}.

\bibitem[Bradley and Terry(1952)]{Bradley1952RankAO}
Ralph~Allan Bradley and Milton~E. Terry.
\newblock Rank analysis of incomplete block designs: I. the method of paired comparisons.
\newblock \emph{Biometrika}, 39:\penalty0 324, 1952.
\newblock URL \url{https://api.semanticscholar.org/CorpusID:125209808}.

\bibitem[Cai and Zheng(2023{\natexlab{a}})]{cai2023accelerated}
Yang Cai and Weiqiang Zheng.
\newblock Accelerated single-call methods for constrained min-max optimization.
\newblock \emph{International Conference on Learning Representations (ICLR)}, 2023{\natexlab{a}}.

\bibitem[Cai and Zheng(2023{\natexlab{b}})]{cai2023doubly}
Yang Cai and Weiqiang Zheng.
\newblock Doubly optimal no-regret learning in monotone games.
\newblock In \emph{International Conference on Machine Learning}, pages 3507--3524. PMLR, 2023{\natexlab{b}}.

\bibitem[Cai et~al.(2022)Cai, Oikonomou, and Zheng]{cai2022finite}
Yang Cai, Argyris Oikonomou, and Weiqiang Zheng.
\newblock Finite-time last-iterate convergence for learning in multi-player games.
\newblock In \emph{Advances in Neural Information Processing Systems (NeurIPS)}, 2022.

\bibitem[Cai et~al.(2023)Cai, Luo, Wei, and Zheng]{cai2023uncoupled}
Yang Cai, Haipeng Luo, Chen-Yu Wei, and Weiqiang Zheng.
\newblock Uncoupled and convergent learning in two-player zero-sum markov games with bandit feedback.
\newblock \emph{Advances in Neural Information Processing Systems}, 36, 2023.

\bibitem[Cai et~al.(2024{\natexlab{a}})Cai, Farina, Grand-Cl{\'e}ment, Kroer, Lee, Luo, and Zheng]{cai2024fast}
Yang Cai, Gabriele Farina, Julien Grand-Cl{\'e}ment, Christian Kroer, Chung-Wei Lee, Haipeng Luo, and Weiqiang Zheng.
\newblock Fast last-iterate convergence of learning in games requires forgetful algorithms.
\newblock \emph{arXiv preprint arXiv:2406.10631}, 2024{\natexlab{a}}.

\bibitem[Cai et~al.(2024{\natexlab{b}})Cai, Oikonomou, and Zheng]{cai2024accelerated}
Yang Cai, Argyris Oikonomou, and Weiqiang Zheng.
\newblock Accelerated algorithms for constrained nonconvex-nonconcave min-max optimization and comonotone inclusion.
\newblock In \emph{Forty-first International Conference on Machine Learning}, 2024{\natexlab{b}}.

\bibitem[Calandriello et~al.(2024)Calandriello, Guo, Munos, Rowland, Tang, Pires, Richemond, Lan, Valko, Liu, Joshi, Zheng, and Piot]{calandriello2024human}
Daniele Calandriello, Zhaohan~Daniel Guo, Remi Munos, Mark Rowland, Yunhao Tang, Bernardo~Avila Pires, Pierre~Harvey Richemond, Charline~Le Lan, Michal Valko, Tianqi Liu, Rishabh Joshi, Zeyu Zheng, and Bilal Piot.
\newblock Human alignment of large language models through online preference optimisation.
\newblock In \emph{Forty-first International Conference on Machine Learning}, 2024.
\newblock URL \url{https://openreview.net/forum?id=2RQqg2Y7Y6}.

\bibitem[Chen et~al.(2021)Chen, Tworek, Jun, Yuan, Pond{\'e}, Kaplan, Edwards, Burda, Joseph, Brockman, Ray, Puri, Krueger, Petrov, Khlaaf, Sastry, Mishkin, Chan, Gray, Ryder, Pavlov, Power, Kaiser, Bavarian, Winter, Tillet, Such, Cummings, Plappert, Chantzis, Barnes, Herbert-Voss, Guss, Nichol, Babuschkin, Balaji, Jain, Carr, Leike, Achiam, Misra, Morikawa, Radford, Knight, Brundage, Murati, Mayer, Welinder, McGrew, Amodei, McCandlish, Sutskever, and Zaremba]{Chen2021EvaluatingLL}
Mark Chen, Jerry Tworek, Heewoo Jun, Qiming Yuan, Henrique Pond{\'e}, Jared Kaplan, Harrison Edwards, Yura Burda, Nicholas Joseph, Greg Brockman, Alex Ray, Raul Puri, Gretchen Krueger, Michael Petrov, Heidy Khlaaf, Girish Sastry, Pamela Mishkin, Brooke Chan, Scott Gray, Nick Ryder, Mikhail Pavlov, Alethea Power, Lukasz Kaiser, Mohammad Bavarian, Clemens Winter, Philippe Tillet, Felipe~Petroski Such, David~W. Cummings, Matthias Plappert, Fotios Chantzis, Elizabeth Barnes, Ariel Herbert-Voss, William~H. Guss, Alex Nichol, Igor Babuschkin, Suchir Balaji, Shantanu Jain, Andrew Carr, Jan Leike, Joshua Achiam, Vedant Misra, Evan Morikawa, Alec Radford, Matthew~M. Knight, Miles Brundage, Mira Murati, Katie Mayer, Peter Welinder, Bob McGrew, Dario Amodei, Sam McCandlish, Ilya Sutskever, and Wojciech Zaremba.
\newblock Evaluating large language models trained on code.
\newblock \emph{ArXiv}, abs/2107.03374, 2021.

\bibitem[Christiano et~al.(2017)Christiano, Leike, Brown, Martic, Legg, and Amodei]{ChristianoLBMLA17}
Paul~F. Christiano, Jan Leike, Tom~B. Brown, Miljan Martic, Shane Legg, and Dario Amodei.
\newblock Deep reinforcement learning from human preferences.
\newblock In Isabelle Guyon, Ulrike von Luxburg, Samy Bengio, Hanna~M. Wallach, Rob Fergus, S.~V.~N. Vishwanathan, and Roman Garnett, editors, \emph{Advances in Neural Information Processing Systems 30: Annual Conference on Neural Information Processing Systems 2017, December 4-9, 2017, Long Beach, CA, {USA}}, pages 4299--4307, 2017.
\newblock URL \url{https://proceedings.neurips.cc/paper/2017/hash/d5e2c0adad503c91f91df240d0cd4e49-Abstract.html}.

\bibitem[Cobbe et~al.(2021)Cobbe, Kosaraju, Bavarian, Chen, Jun, Kaiser, Plappert, Tworek, Hilton, Nakano, Hesse, and Schulman]{Cobbe2021TrainingVT}
Karl Cobbe, Vineet Kosaraju, Mohammad Bavarian, Mark Chen, Heewoo Jun, Lukasz Kaiser, Matthias Plappert, Jerry Tworek, Jacob Hilton, Reiichiro Nakano, Christopher Hesse, and John Schulman.
\newblock Training verifiers to solve math word problems.
\newblock \emph{ArXiv}, abs/2110.14168, 2021.

\bibitem[Cui et~al.(2023)Cui, Yuan, Ding, Yao, Zhu, Ni, Xie, Liu, and Sun]{cui2023ultrafeedback}
Ganqu Cui, Lifan Yuan, Ning Ding, Guanming Yao, Wei Zhu, Yuan Ni, Guotong Xie, Zhiyuan Liu, and Maosong Sun.
\newblock Ultrafeedback: Boosting language models with high-quality feedback.
\newblock \emph{arXiv preprint arXiv:2310.01377}, 2023.

\bibitem[Daskalakis and Panageas(2018)]{daskalakis2018limit}
Constantinos Daskalakis and Ioannis Panageas.
\newblock The limit points of (optimistic) gradient descent in min-max optimization.
\newblock In \emph{the 32nd Annual Conference on Neural Information Processing Systems (NeurIPS)}, 2018.

\bibitem[Dong et~al.(2024)Dong, Xiong, Pang, Wang, Zhao, Zhou, Jiang, Sahoo, Xiong, and Zhang]{dong2024rlhf}
Hanze Dong, Wei Xiong, Bo~Pang, Haoxiang Wang, Han Zhao, Yingbo Zhou, Nan Jiang, Doyen Sahoo, Caiming Xiong, and Tong Zhang.
\newblock Rlhf workflow: From reward modeling to online rlhf.
\newblock \emph{Transactions on Machine Learning Research}, 2024.

\bibitem[Dubey et~al.(2024)Dubey, Jauhri, Pandey, Kadian, Al-Dahle, Letman, Mathur, Schelten, Yang, Fan, Goyal, Hartshorn, Yang, Mitra, Sravankumar, Korenev, Hinsvark, Rao, Zhang, Rodriguez, Gregerson, Spataru, Roziere, Biron, Tang, Chern, Caucheteux, Nayak, Bi, Marra, McConnell, Keller, Touret, Wu, Wong, Ferrer, Nikolaidis, Allonsius, Song, Pintz, Livshits, Esiobu, Choudhary, Mahajan, Garcia-Olano, Perino, Hupkes, Lakomkin, AlBadawy, Lobanova, Dinan, Smith, Radenovic, Zhang, Synnaeve, Lee, Anderson, Nail, Mialon, Pang, Cucurell, Nguyen, Korevaar, Xu, Touvron, Zarov, Ibarra, Kloumann, Misra, Evtimov, Copet, Lee, Geffert, Vranes, Park, Mahadeokar, Shah, van~der Linde, Billock, Hong, Lee, Fu, Chi, Huang, Liu, Wang, Yu, Bitton, Spisak, Park, Rocca, Johnstun, Saxe, Jia, Alwala, Upasani, Plawiak, Li, Heafield, Stone, El-Arini, Iyer, Malik, Chiu, Bhalla, Rantala-Yeary, van~der Maaten, Chen, Tan, Jenkins, Martin, Madaan, Malo, Blecher, Landzaat, de~Oliveira, Muzzi, Pasupuleti, Singh, Paluri, Kardas, Oldham, Rita,
  Pavlova, Kambadur, Lewis, Si, Singh, Hassan, Goyal, Torabi, Bashlykov, Bogoychev, Chatterji, Duchenne, Çelebi, Alrassy, Zhang, Li, Vasic, Weng, Bhargava, Dubal, Krishnan, Koura, Xu, He, Dong, Srinivasan, Ganapathy, Calderer, Cabral, Stojnic, Raileanu, Girdhar, Patel, Sauvestre, Polidoro, Sumbaly, Taylor, Silva, Hou, Wang, Hosseini, Chennabasappa, Singh, Bell, Kim, Edunov, Nie, Narang, Raparthy, Shen, Wan, Bhosale, Zhang, Vandenhende, Batra, Whitman, Sootla, Collot, Gururangan, Borodinsky, Herman, Fowler, Sheasha, Georgiou, Scialom, Speckbacher, Mihaylov, Xiao, Karn, Goswami, Gupta, Ramanathan, Kerkez, Gonguet, Do, Vogeti, Petrovic, Chu, Xiong, Fu, Meers, Martinet, Wang, Tan, Xie, Jia, Wang, Goldschlag, Gaur, Babaei, Wen, Song, Zhang, Li, Mao, Coudert, Yan, Chen, Papakipos, Singh, Grattafiori, Jain, Kelsey, Shajnfeld, Gangidi, Victoria, Goldstand, Menon, Sharma, Boesenberg, Vaughan, Baevski, Feinstein, Kallet, Sangani, Yunus, Lupu, Alvarado, Caples, Gu, Ho, Poulton, Ryan, Ramchandani, Franco, Saraf,
  Chowdhury, Gabriel, Bharambe, Eisenman, Yazdan, James, Maurer, Leonhardi, Huang, Loyd, Paola, Paranjape, Liu, Wu, Ni, Hancock, Wasti, Spence, Stojkovic, Gamido, Montalvo, Parker, Burton, Mejia, Wang, Kim, Zhou, Hu, Chu, Cai, Tindal, Feichtenhofer, Civin, Beaty, Kreymer, Li, Wyatt, Adkins, Xu, Testuggine, David, Parikh, Liskovich, Foss, Wang, Le, Holland, Dowling, Jamil, Montgomery, Presani, Hahn, Wood, Brinkman, Arcaute, Dunbar, Smothers, Sun, Kreuk, Tian, Ozgenel, Caggioni, Guzmán, Kanayet, Seide, Florez, Schwarz, Badeer, Swee, Halpern, Thattai, Herman, Sizov, Guangyi, Zhang, Lakshminarayanan, Shojanazeri, Zou, Wang, Zha, Habeeb, Rudolph, Suk, Aspegren, Goldman, Damlaj, Molybog, Tufanov, Veliche, Gat, Weissman, Geboski, Kohli, Asher, Gaya, Marcus, Tang, Chan, Zhen, Reizenstein, Teboul, Zhong, Jin, Yang, Cummings, Carvill, Shepard, McPhie, Torres, Ginsburg, Wang, Wu, U, Saxena, Prasad, Khandelwal, Zand, Matosich, Veeraraghavan, Michelena, Li, Huang, Chawla, Lakhotia, Huang, Chen, Garg, A, Silva, Bell,
  Zhang, Guo, Yu, Moshkovich, Wehrstedt, Khabsa, Avalani, Bhatt, Tsimpoukelli, Mankus, Hasson, Lennie, Reso, Groshev, Naumov, Lathi, Keneally, Seltzer, Valko, Restrepo, Patel, Vyatskov, Samvelyan, Clark, Macey, Wang, Hermoso, Metanat, Rastegari, Bansal, Santhanam, Parks, White, Bawa, Singhal, Egebo, Usunier, Laptev, Dong, Zhang, Cheng, Chernoguz, Hart, Salpekar, Kalinli, Kent, Parekh, Saab, Balaji, Rittner, Bontrager, Roux, Dollar, Zvyagina, Ratanchandani, Yuvraj, Liang, Alao, Rodriguez, Ayub, Murthy, Nayani, Mitra, Li, Hogan, Battey, Wang, Maheswari, Howes, Rinott, Bondu, Datta, Chugh, Hunt, Dhillon, Sidorov, Pan, Verma, Yamamoto, Ramaswamy, Lindsay, Lindsay, Feng, Lin, Zha, Shankar, Zhang, Zhang, Wang, Agarwal, Sajuyigbe, Chintala, Max, Chen, Kehoe, Satterfield, Govindaprasad, Gupta, Cho, Virk, Subramanian, Choudhury, Goldman, Remez, Glaser, Best, Kohler, Robinson, Li, Zhang, Matthews, Chou, Shaked, Vontimitta, Ajayi, Montanez, Mohan, Kumar, Mangla, Albiero, Ionescu, Poenaru, Mihailescu, Ivanov, Li, Wang,
  Jiang, Bouaziz, Constable, Tang, Wang, Wu, Wang, Xia, Wu, Gao, Chen, Hu, Jia, Qi, Li, Zhang, Zhang, Adi, Nam, Yu, Wang, Hao, Qian, He, Rait, DeVito, Rosnbrick, Wen, Yang, and Zhao]{dubey2024llama3herdmodels}
Abhimanyu Dubey, Abhinav Jauhri, Abhinav Pandey, Abhishek Kadian, Ahmad Al-Dahle, Aiesha Letman, Akhil Mathur, Alan Schelten, Amy Yang, Angela Fan, Anirudh Goyal, Anthony Hartshorn, Aobo Yang, Archi Mitra, Archie Sravankumar, Artem Korenev, Arthur Hinsvark, Arun Rao, Aston Zhang, Aurelien Rodriguez, Austen Gregerson, Ava Spataru, Baptiste Roziere, Bethany Biron, Binh Tang, Bobbie Chern, Charlotte Caucheteux, Chaya Nayak, Chloe Bi, Chris Marra, Chris McConnell, Christian Keller, Christophe Touret, Chunyang Wu, Corinne Wong, Cristian~Canton Ferrer, Cyrus Nikolaidis, Damien Allonsius, Daniel Song, Danielle Pintz, Danny Livshits, David Esiobu, Dhruv Choudhary, Dhruv Mahajan, Diego Garcia-Olano, Diego Perino, Dieuwke Hupkes, Egor Lakomkin, Ehab AlBadawy, Elina Lobanova, Emily Dinan, Eric~Michael Smith, Filip Radenovic, Frank Zhang, Gabriel Synnaeve, Gabrielle Lee, Georgia~Lewis Anderson, Graeme Nail, Gregoire Mialon, Guan Pang, Guillem Cucurell, Hailey Nguyen, Hannah Korevaar, Hu~Xu, Hugo Touvron, Iliyan Zarov,
  Imanol~Arrieta Ibarra, Isabel Kloumann, Ishan Misra, Ivan Evtimov, Jade Copet, Jaewon Lee, Jan Geffert, Jana Vranes, Jason Park, Jay Mahadeokar, Jeet Shah, Jelmer van~der Linde, Jennifer Billock, Jenny Hong, Jenya Lee, Jeremy Fu, Jianfeng Chi, Jianyu Huang, Jiawen Liu, Jie Wang, Jiecao Yu, Joanna Bitton, Joe Spisak, Jongsoo Park, Joseph Rocca, Joshua Johnstun, Joshua Saxe, Junteng Jia, Kalyan~Vasuden Alwala, Kartikeya Upasani, Kate Plawiak, Ke~Li, Kenneth Heafield, Kevin Stone, Khalid El-Arini, Krithika Iyer, Kshitiz Malik, Kuenley Chiu, Kunal Bhalla, Lauren Rantala-Yeary, Laurens van~der Maaten, Lawrence Chen, Liang Tan, Liz Jenkins, Louis Martin, Lovish Madaan, Lubo Malo, Lukas Blecher, Lukas Landzaat, Luke de~Oliveira, Madeline Muzzi, Mahesh Pasupuleti, Mannat Singh, Manohar Paluri, Marcin Kardas, Mathew Oldham, Mathieu Rita, Maya Pavlova, Melanie Kambadur, Mike Lewis, Min Si, Mitesh~Kumar Singh, Mona Hassan, Naman Goyal, Narjes Torabi, Nikolay Bashlykov, Nikolay Bogoychev, Niladri Chatterji, Olivier
  Duchenne, Onur Çelebi, Patrick Alrassy, Pengchuan Zhang, Pengwei Li, Petar Vasic, Peter Weng, Prajjwal Bhargava, Pratik Dubal, Praveen Krishnan, Punit~Singh Koura, Puxin Xu, Qing He, Qingxiao Dong, Ragavan Srinivasan, Raj Ganapathy, Ramon Calderer, Ricardo~Silveira Cabral, Robert Stojnic, Roberta Raileanu, Rohit Girdhar, Rohit Patel, Romain Sauvestre, Ronnie Polidoro, Roshan Sumbaly, Ross Taylor, Ruan Silva, Rui Hou, Rui Wang, Saghar Hosseini, Sahana Chennabasappa, Sanjay Singh, Sean Bell, Seohyun~Sonia Kim, Sergey Edunov, Shaoliang Nie, Sharan Narang, Sharath Raparthy, Sheng Shen, Shengye Wan, Shruti Bhosale, Shun Zhang, Simon Vandenhende, Soumya Batra, Spencer Whitman, Sten Sootla, Stephane Collot, Suchin Gururangan, Sydney Borodinsky, Tamar Herman, Tara Fowler, Tarek Sheasha, Thomas Georgiou, Thomas Scialom, Tobias Speckbacher, Todor Mihaylov, Tong Xiao, Ujjwal Karn, Vedanuj Goswami, Vibhor Gupta, Vignesh Ramanathan, Viktor Kerkez, Vincent Gonguet, Virginie Do, Vish Vogeti, Vladan Petrovic, Weiwei Chu,
  Wenhan Xiong, Wenyin Fu, Whitney Meers, Xavier Martinet, Xiaodong Wang, Xiaoqing~Ellen Tan, Xinfeng Xie, Xuchao Jia, Xuewei Wang, Yaelle Goldschlag, Yashesh Gaur, Yasmine Babaei, Yi~Wen, Yiwen Song, Yuchen Zhang, Yue Li, Yuning Mao, Zacharie~Delpierre Coudert, Zheng Yan, Zhengxing Chen, Zoe Papakipos, Aaditya Singh, Aaron Grattafiori, Abha Jain, Adam Kelsey, Adam Shajnfeld, Adithya Gangidi, Adolfo Victoria, Ahuva Goldstand, Ajay Menon, Ajay Sharma, Alex Boesenberg, Alex Vaughan, Alexei Baevski, Allie Feinstein, Amanda Kallet, Amit Sangani, Anam Yunus, Andrei Lupu, Andres Alvarado, Andrew Caples, Andrew Gu, Andrew Ho, Andrew Poulton, Andrew Ryan, Ankit Ramchandani, Annie Franco, Aparajita Saraf, Arkabandhu Chowdhury, Ashley Gabriel, Ashwin Bharambe, Assaf Eisenman, Azadeh Yazdan, Beau James, Ben Maurer, Benjamin Leonhardi, Bernie Huang, Beth Loyd, Beto~De Paola, Bhargavi Paranjape, Bing Liu, Bo~Wu, Boyu Ni, Braden Hancock, Bram Wasti, Brandon Spence, Brani Stojkovic, Brian Gamido, Britt Montalvo, Carl
  Parker, Carly Burton, Catalina Mejia, Changhan Wang, Changkyu Kim, Chao Zhou, Chester Hu, Ching-Hsiang Chu, Chris Cai, Chris Tindal, Christoph Feichtenhofer, Damon Civin, Dana Beaty, Daniel Kreymer, Daniel Li, Danny Wyatt, David Adkins, David Xu, Davide Testuggine, Delia David, Devi Parikh, Diana Liskovich, Didem Foss, Dingkang Wang, Duc Le, Dustin Holland, Edward Dowling, Eissa Jamil, Elaine Montgomery, Eleonora Presani, Emily Hahn, Emily Wood, Erik Brinkman, Esteban Arcaute, Evan Dunbar, Evan Smothers, Fei Sun, Felix Kreuk, Feng Tian, Firat Ozgenel, Francesco Caggioni, Francisco Guzmán, Frank Kanayet, Frank Seide, Gabriela~Medina Florez, Gabriella Schwarz, Gada Badeer, Georgia Swee, Gil Halpern, Govind Thattai, Grant Herman, Grigory Sizov, Guangyi, Zhang, Guna Lakshminarayanan, Hamid Shojanazeri, Han Zou, Hannah Wang, Hanwen Zha, Haroun Habeeb, Harrison Rudolph, Helen Suk, Henry Aspegren, Hunter Goldman, Ibrahim Damlaj, Igor Molybog, Igor Tufanov, Irina-Elena Veliche, Itai Gat, Jake Weissman, James
  Geboski, James Kohli, Japhet Asher, Jean-Baptiste Gaya, Jeff Marcus, Jeff Tang, Jennifer Chan, Jenny Zhen, Jeremy Reizenstein, Jeremy Teboul, Jessica Zhong, Jian Jin, Jingyi Yang, Joe Cummings, Jon Carvill, Jon Shepard, Jonathan McPhie, Jonathan Torres, Josh Ginsburg, Junjie Wang, Kai Wu, Kam~Hou U, Karan Saxena, Karthik Prasad, Kartikay Khandelwal, Katayoun Zand, Kathy Matosich, Kaushik Veeraraghavan, Kelly Michelena, Keqian Li, Kun Huang, Kunal Chawla, Kushal Lakhotia, Kyle Huang, Lailin Chen, Lakshya Garg, Lavender A, Leandro Silva, Lee Bell, Lei Zhang, Liangpeng Guo, Licheng Yu, Liron Moshkovich, Luca Wehrstedt, Madian Khabsa, Manav Avalani, Manish Bhatt, Maria Tsimpoukelli, Martynas Mankus, Matan Hasson, Matthew Lennie, Matthias Reso, Maxim Groshev, Maxim Naumov, Maya Lathi, Meghan Keneally, Michael~L. Seltzer, Michal Valko, Michelle Restrepo, Mihir Patel, Mik Vyatskov, Mikayel Samvelyan, Mike Clark, Mike Macey, Mike Wang, Miquel~Jubert Hermoso, Mo~Metanat, Mohammad Rastegari, Munish Bansal, Nandhini
  Santhanam, Natascha Parks, Natasha White, Navyata Bawa, Nayan Singhal, Nick Egebo, Nicolas Usunier, Nikolay~Pavlovich Laptev, Ning Dong, Ning Zhang, Norman Cheng, Oleg Chernoguz, Olivia Hart, Omkar Salpekar, Ozlem Kalinli, Parkin Kent, Parth Parekh, Paul Saab, Pavan Balaji, Pedro Rittner, Philip Bontrager, Pierre Roux, Piotr Dollar, Polina Zvyagina, Prashant Ratanchandani, Pritish Yuvraj, Qian Liang, Rachad Alao, Rachel Rodriguez, Rafi Ayub, Raghotham Murthy, Raghu Nayani, Rahul Mitra, Raymond Li, Rebekkah Hogan, Robin Battey, Rocky Wang, Rohan Maheswari, Russ Howes, Ruty Rinott, Sai~Jayesh Bondu, Samyak Datta, Sara Chugh, Sara Hunt, Sargun Dhillon, Sasha Sidorov, Satadru Pan, Saurabh Verma, Seiji Yamamoto, Sharadh Ramaswamy, Shaun Lindsay, Shaun Lindsay, Sheng Feng, Shenghao Lin, Shengxin~Cindy Zha, Shiva Shankar, Shuqiang Zhang, Shuqiang Zhang, Sinong Wang, Sneha Agarwal, Soji Sajuyigbe, Soumith Chintala, Stephanie Max, Stephen Chen, Steve Kehoe, Steve Satterfield, Sudarshan Govindaprasad, Sumit Gupta,
  Sungmin Cho, Sunny Virk, Suraj Subramanian, Sy~Choudhury, Sydney Goldman, Tal Remez, Tamar Glaser, Tamara Best, Thilo Kohler, Thomas Robinson, Tianhe Li, Tianjun Zhang, Tim Matthews, Timothy Chou, Tzook Shaked, Varun Vontimitta, Victoria Ajayi, Victoria Montanez, Vijai Mohan, Vinay~Satish Kumar, Vishal Mangla, Vítor Albiero, Vlad Ionescu, Vlad Poenaru, Vlad~Tiberiu Mihailescu, Vladimir Ivanov, Wei Li, Wenchen Wang, Wenwen Jiang, Wes Bouaziz, Will Constable, Xiaocheng Tang, Xiaofang Wang, Xiaojian Wu, Xiaolan Wang, Xide Xia, Xilun Wu, Xinbo Gao, Yanjun Chen, Ye~Hu, Ye~Jia, Ye~Qi, Yenda Li, Yilin Zhang, Ying Zhang, Yossi Adi, Youngjin Nam, Yu, Wang, Yuchen Hao, Yundi Qian, Yuzi He, Zach Rait, Zachary DeVito, Zef Rosnbrick, Zhaoduo Wen, Zhenyu Yang, and Zhiwei Zhao.
\newblock The llama 3 herd of models, 2024.
\newblock URL \url{https://arxiv.org/abs/2407.21783}.

\bibitem[Ethayarajh et~al.(2024)Ethayarajh, Xu, Muennighoff, Jurafsky, and Kiela]{EthayarajhXMJK}
Kawin Ethayarajh, Winnie Xu, Niklas Muennighoff, Dan Jurafsky, and Douwe Kiela.
\newblock {KTO:} model alignment as prospect theoretic optimization.
\newblock \emph{CoRR}, abs/2402.01306, 2024.
\newblock \doi{10.48550/ARXIV.2402.01306}.
\newblock URL \url{https://doi.org/10.48550/arXiv.2402.01306}.

\bibitem[Facchinei and Pang(2003)]{facchinei2003finite}
Francisco Facchinei and Jong-Shi Pang.
\newblock \emph{Finite-dimensional variational inequalities and complementarity problems}.
\newblock Springer, 2003.

\bibitem[Farina et~al.(2022)Farina, Kroer, Lee, and Luo]{farina2022clairvoyant}
Gabriele Farina, Christian Kroer, Chung-Wei Lee, and Haipeng Luo.
\newblock Clairvoyant regret minimization: Equivalence with nemirovski's conceptual prox method and extension to general convex games.
\newblock \emph{arXiv preprint arXiv:2208.14891}, 2022.

\bibitem[Gao et~al.(2024)Gao, Chang, Zhan, Oertell, Swamy, Brantley, Joachims, Bagnell, Lee, and Sun]{GaoCZOSBJBLS}
Zhaolin Gao, Jonathan~D. Chang, Wenhao Zhan, Owen Oertell, Gokul Swamy, Kiant{\'{e}} Brantley, Thorsten Joachims, J.~Andrew Bagnell, Jason~D. Lee, and Wen Sun.
\newblock {REBEL:} reinforcement learning via regressing relative rewards.
\newblock \emph{CoRR}, abs/2404.16767, 2024.
\newblock \doi{10.48550/ARXIV.2404.16767}.
\newblock URL \url{https://doi.org/10.48550/arXiv.2404.16767}.

\bibitem[Golowich et~al.(2020{\natexlab{a}})Golowich, Pattathil, and Daskalakis]{golowich2020tight}
Noah Golowich, Sarath Pattathil, and Constantinos Daskalakis.
\newblock Tight last-iterate convergence rates for no-regret learning in multi-player games.
\newblock \emph{Advances in neural information processing systems (NeurIPS)}, 2020{\natexlab{a}}.

\bibitem[Golowich et~al.(2020{\natexlab{b}})Golowich, Pattathil, Daskalakis, and Ozdaglar]{golowich2020last}
Noah Golowich, Sarath Pattathil, Constantinos Daskalakis, and Asuman Ozdaglar.
\newblock Last iterate is slower than averaged iterate in smooth convex-concave saddle point problems.
\newblock In \emph{Conference on Learning Theory (COLT)}, 2020{\natexlab{b}}.

\bibitem[Gorbunov et~al.(2022)Gorbunov, Taylor, and Gidel]{GorbunovTG22}
Eduard Gorbunov, Adrien~B. Taylor, and Gauthier Gidel.
\newblock Last-iterate convergence of optimistic gradient method for monotone variational inequalities.
\newblock In Sanmi Koyejo, S.~Mohamed, A.~Agarwal, Danielle Belgrave, K.~Cho, and A.~Oh, editors, \emph{Advances in Neural Information Processing Systems 35: Annual Conference on Neural Information Processing Systems 2022, NeurIPS 2022, New Orleans, LA, USA, November 28 - December 9, 2022}, 2022.
\newblock URL \url{http://papers.nips.cc/paper\_files/paper/2022/hash/893cd874ba98afa54ae9e385a24a83ac-Abstract-Conference.html}.

\bibitem[Guo et~al.(2025)Guo, Yang, Zhang, Song, Zhang, Xu, Zhu, Ma, Wang, Bi, et~al.]{guo2025deepseek}
Daya Guo, Dejian Yang, Haowei Zhang, Junxiao Song, Ruoyu Zhang, Runxin Xu, Qihao Zhu, Shirong Ma, Peiyi Wang, Xiao Bi, et~al.
\newblock Deepseek-r1: Incentivizing reasoning capability in llms via reinforcement learning.
\newblock \emph{Nature}, 645:\penalty0 633--638, 2025.
\newblock URL \url{https://doi.org/10.1038/s41586-025-09422-z}.

\bibitem[Hendrycks et~al.(2021)Hendrycks, Burns, Basart, Zou, Mazeika, Song, and Steinhardt]{hendrycks2021measuring}
Dan Hendrycks, Collin Burns, Steven Basart, Andy Zou, Mantas Mazeika, Dawn Song, and Jacob Steinhardt.
\newblock Measuring massive multitask language understanding.
\newblock In \emph{International Conference on Learning Representations}, 2021.
\newblock URL \url{https://openreview.net/forum?id=d7KBjmI3GmQ}.

\bibitem[Hsieh et~al.(2021)Hsieh, Antonakopoulos, and Mertikopoulos]{hsieh2021adaptive}
Yu-Guan Hsieh, Kimon Antonakopoulos, and Panayotis Mertikopoulos.
\newblock Adaptive learning in continuous games: Optimal regret bounds and convergence to nash equilibrium.
\newblock In \emph{Conference on Learning Theory}, pages 2388--2422. PMLR, 2021.

\bibitem[Iusem et~al.(2003)Iusem, Pennanen., and Svaiter]{IusemPS}
A.~N. Iusem, T.~Pennanen., and B.~F. Svaiter.
\newblock Inexact variants of the proximal point algorithm without monotonicity.
\newblock \emph{SIAM Journal on Optimization}, 13\penalty0 (4):\penalty0 1080--1097, 2003.
\newblock \doi{10.1137/S1052623401399587}.
\newblock URL \url{https://doi.org/10.1137/S1052623401399587}.

\bibitem[Korpelevich(1976)]{korpelevich_extragradient_1976}
G.~M. Korpelevich.
\newblock The extragradient method for finding saddle points and other problems.
\newblock \emph{Matecon}, 12:\penalty0 747--756, 1976.
\newblock URL \url{https://ci.nii.ac.jp/naid/10017556617/}.

\bibitem[Lambert et~al.(2024{\natexlab{a}})Lambert, Morrison, Pyatkin, Huang, Ivison, Brahman, Miranda, Liu, Dziri, Lyu, et~al.]{lambert2024t}
Nathan Lambert, Jacob Morrison, Valentina Pyatkin, Shengyi Huang, Hamish Ivison, Faeze Brahman, Lester James~V Miranda, Alisa Liu, Nouha Dziri, Shane Lyu, et~al.
\newblock T$\backslash$" ulu 3: Pushing frontiers in open language model post-training.
\newblock \emph{arXiv preprint arXiv:2411.15124}, 2024{\natexlab{a}}.

\bibitem[Lambert et~al.(2024{\natexlab{b}})Lambert, Pyatkin, Morrison, Miranda, Lin, Chandu, Dziri, Kumar, Zick, Choi, et~al.]{lambert2024rewardbench}
Nathan Lambert, Valentina Pyatkin, Jacob Morrison, LJ~Miranda, Bill~Yuchen Lin, Khyathi Chandu, Nouha Dziri, Sachin Kumar, Tom Zick, Yejin Choi, et~al.
\newblock Rewardbench: Evaluating reward models for language modeling.
\newblock \emph{arXiv preprint arXiv:2403.13787}, 2024{\natexlab{b}}.

\bibitem[Li et~al.(2023)Li, Zhang, Dubois, Taori, Gulrajani, Guestrin, Liang, and Hashimoto]{alpaca_eval}
Xuechen Li, Tianyi Zhang, Yann Dubois, Rohan Taori, Ishaan Gulrajani, Carlos Guestrin, Percy Liang, and Tatsunori~B. Hashimoto.
\newblock Alpacaeval: An automatic evaluator of instruction-following models.
\newblock \url{https://github.com/tatsu-lab/alpaca_eval}, 5 2023.

\bibitem[Liu et~al.(2024)Liu, Zeng, Liu, Yan, He, Wang, Yan, Liu, and Zhou]{liu2024skywork}
Chris~Yuhao Liu, Liang Zeng, Jiacai Liu, Rui Yan, Jujie He, Chaojie Wang, Shuicheng Yan, Yang Liu, and Yahui Zhou.
\newblock Skywork-reward: Bag of tricks for reward modeling in llms.
\newblock \emph{arXiv preprint arXiv:2410.18451}, 2024.

\bibitem[Liu et~al.(2025)Liu, Chen, Li, Qi, Pang, Du, Lee, and Lin]{liu2025understanding}
Zichen Liu, Changyu Chen, Wenjun Li, Penghui Qi, Tianyu Pang, Chao Du, Wee~Sun Lee, and Min Lin.
\newblock Understanding r1-zero-like training: A critical perspective.
\newblock In \emph{Conference on Language Modeling (COLM)}, 2025.

\bibitem[May(1954)]{may1954intransitivity}
Kenneth~O May.
\newblock Intransitivity, utility, and the aggregation of preference patterns.
\newblock \emph{Econometrica: Journal of the Econometric Society}, pages 1--13, 1954.

\bibitem[Meng et~al.(2024)Meng, Xia, and Chen]{meng2024simpo}
Yu~Meng, Mengzhou Xia, and Danqi Chen.
\newblock Simpo: Simple preference optimization with a reference-free reward.
\newblock \emph{arXiv preprint arXiv:2405.14734}, 2024.

\bibitem[Mertikopoulos et~al.(2018)Mertikopoulos, Papadimitriou, and Piliouras]{MertikopoulosPP18}
Panayotis Mertikopoulos, Christos~H. Papadimitriou, and Georgios Piliouras.
\newblock Cycles in adversarial regularized learning.
\newblock In Artur Czumaj, editor, \emph{Proceedings of the Twenty-Ninth Annual {ACM-SIAM} Symposium on Discrete Algorithms, {SODA} 2018, New Orleans, LA, USA, January 7-10, 2018}, pages 2703--2717. {SIAM}, 2018.
\newblock \doi{10.1137/1.9781611975031.172}.
\newblock URL \url{https://doi.org/10.1137/1.9781611975031.172}.

\bibitem[Mokhtari et~al.(2020{\natexlab{a}})Mokhtari, Ozdaglar, and Pattathil]{mokhtari2020unified}
Aryan Mokhtari, Asuman Ozdaglar, and Sarath Pattathil.
\newblock A unified analysis of extra-gradient and optimistic gradient methods for saddle point problems: Proximal point approach.
\newblock In \emph{International Conference on Artificial Intelligence and Statistics (AISTATS)}, 2020{\natexlab{a}}.

\bibitem[Mokhtari et~al.(2020{\natexlab{b}})Mokhtari, Ozdaglar, and Pattathil]{mokhtari2020convergence}
Aryan Mokhtari, Asuman~E Ozdaglar, and Sarath Pattathil.
\newblock Convergence rate of $\mathcal{O}(1/k)$ for optimistic gradient and extragradient methods in smooth convex-concave saddle point problems.
\newblock \emph{SIAM Journal on Optimization}, 30\penalty0 (4):\penalty0 3230--3251, 2020{\natexlab{b}}.

\bibitem[Munos et~al.(2024)Munos, Valko, Calandriello, Azar, Rowland, Guo, Tang, Geist, Mesnard, Fiegel, et~al.]{munos2024nash}
Remi Munos, Michal Valko, Daniele Calandriello, Mohammad~Gheshlaghi Azar, Mark Rowland, Zhaohan~Daniel Guo, Yunhao Tang, Matthieu Geist, Thomas Mesnard, C{\^o}me Fiegel, et~al.
\newblock Nash learning from human feedback.
\newblock In \emph{Forty-first International Conference on Machine Learning}, 2024.

\bibitem[Nemirovski(2004)]{nemirovski2004prox}
Arkadi Nemirovski.
\newblock Prox-method with rate of convergence o (1/t) for variational inequalities with lipschitz continuous monotone operators and smooth convex-concave saddle point problems.
\newblock \emph{SIAM Journal on Optimization}, 15\penalty0 (1):\penalty0 229--251, 2004.

\bibitem[Nemirovskij and Yudin(1983)]{nemirovskij1983problem}
Arkadij~Semenovi{\v{c}} Nemirovskij and David~Borisovich Yudin.
\newblock \emph{Problem complexity and method efficiency in optimization}.
\newblock Wiley-Interscience, 1983.

\bibitem[Ouyang et~al.(2022)Ouyang, Wu, Jiang, Almeida, Wainwright, Mishkin, Zhang, Agarwal, Slama, Ray, Schulman, Hilton, Kelton, Miller, Simens, Askell, Welinder, Christiano, Leike, and Lowe]{Ouyang0JAWMZASR22}
Long Ouyang, Jeffrey Wu, Xu~Jiang, Diogo Almeida, Carroll~L. Wainwright, Pamela Mishkin, Chong Zhang, Sandhini Agarwal, Katarina Slama, Alex Ray, John Schulman, Jacob Hilton, Fraser Kelton, Luke Miller, Maddie Simens, Amanda Askell, Peter Welinder, Paul~F. Christiano, Jan Leike, and Ryan Lowe.
\newblock Training language models to follow instructions with human feedback.
\newblock In Sanmi Koyejo, S.~Mohamed, A.~Agarwal, Danielle Belgrave, K.~Cho, and A.~Oh, editors, \emph{Advances in Neural Information Processing Systems 35: Annual Conference on Neural Information Processing Systems 2022, NeurIPS 2022, New Orleans, LA, USA, November 28 - December 9, 2022}, 2022.
\newblock URL \url{http://papers.nips.cc/paper\_files/paper/2022/hash/b1efde53be364a73914f58805a001731-Abstract-Conference.html}.

\bibitem[Parikh et~al.(2014)Parikh, Boyd, et~al.]{parikh2014proximal}
Neal Parikh, Stephen Boyd, et~al.
\newblock Proximal algorithms.
\newblock \emph{Foundations and trends{\textregistered} in Optimization}, 1\penalty0 (3):\penalty0 127--239, 2014.

\bibitem[Park et~al.(2024)Park, Jwa, Ren, Kim, and Choi]{park2024offsetbias}
Junsoo Park, Seungyeon Jwa, Meiying Ren, Daeyoung Kim, and Sanghyuk Choi.
\newblock Offsetbias: Leveraging debiased data for tuning evaluators.
\newblock \emph{arXiv preprint arXiv:2407.06551}, 2024.

\bibitem[Perolat et~al.(2021)Perolat, Munos, Lespiau, Omidshafiei, Rowland, Ortega, Burch, Anthony, Balduzzi, De~Vylder, et~al.]{perolat2021poincare}
Julien Perolat, Remi Munos, Jean-Baptiste Lespiau, Shayegan Omidshafiei, Mark Rowland, Pedro Ortega, Neil Burch, Thomas Anthony, David Balduzzi, Bart De~Vylder, et~al.
\newblock From poincar{\'e} recurrence to convergence in imperfect information games: Finding equilibrium via regularization.
\newblock In \emph{International Conference on Machine Learning}, pages 8525--8535. PMLR, 2021.

\bibitem[Perolat et~al.(2022)Perolat, De~Vylder, Hennes, Tarassov, Strub, de~Boer, Muller, Connor, Burch, Anthony, et~al.]{perolat2022mastering}
Julien Perolat, Bart De~Vylder, Daniel Hennes, Eugene Tarassov, Florian Strub, Vincent de~Boer, Paul Muller, Jerome~T Connor, Neil Burch, Thomas Anthony, et~al.
\newblock Mastering the game of stratego with model-free multiagent reinforcement learning.
\newblock \emph{Science}, 378\penalty0 (6623):\penalty0 990--996, 2022.

\bibitem[Popov(1980)]{popov_modification_1980}
Leonid~Denisovich Popov.
\newblock A modification of the {Arrow}-{Hurwicz} method for search of saddle points.
\newblock \emph{Mathematical notes of the Academy of Sciences of the USSR}, 28\penalty0 (5):\penalty0 845--848, 1980.
\newblock Publisher: Springer.

\bibitem[Rafailov et~al.(2024)Rafailov, Sharma, Mitchell, Manning, Ermon, and Finn]{rafailov2024direct}
Rafael Rafailov, Archit Sharma, Eric Mitchell, Christopher~D Manning, Stefano Ermon, and Chelsea Finn.
\newblock Direct preference optimization: Your language model is secretly a reward model.
\newblock \emph{Advances in Neural Information Processing Systems}, 36, 2024.

\bibitem[Rakhlin and Sridharan(2013)]{rakhlin2013optimization}
Sasha Rakhlin and Karthik Sridharan.
\newblock Optimization, learning, and games with predictable sequences.
\newblock \emph{Advances in Neural Information Processing Systems}, 2013.

\bibitem[Richemond et~al.(2024)Richemond, Tang, Guo, Calandriello, Azar, Rafailov, Pires, Tarassov, Spangher, Ellsworth, et~al.]{richemond2024offline}
Pierre~Harvey Richemond, Yunhao Tang, Daniel Guo, Daniele Calandriello, Mohammad~Gheshlaghi Azar, Rafael Rafailov, Bernardo~Avila Pires, Eugene Tarassov, Lucas Spangher, Will Ellsworth, et~al.
\newblock Offline regularised reinforcement learning for large language models alignment.
\newblock \emph{arXiv preprint arXiv:2405.19107}, 2024.

\bibitem[Rockafellar(1976)]{Rockafellar_pp}
R.~Tyrrell Rockafellar.
\newblock Monotone operators and the proximal point algorithm.
\newblock \emph{SIAM Journal on Control and Optimization}, 14\penalty0 (5):\penalty0 877--898, 1976.
\newblock \doi{10.1137/0314056}.
\newblock URL \url{https://doi.org/10.1137/0314056}.

\bibitem[Rosset et~al.(2024)Rosset, Cheng, Mitra, Santacroce, Awadallah, and Xie]{RossetCMSAX}
Corby Rosset, Ching{-}An Cheng, Arindam Mitra, Michael Santacroce, Ahmed Awadallah, and Tengyang Xie.
\newblock Direct nash optimization: Teaching language models to self-improve with general preferences.
\newblock \emph{CoRR}, abs/2404.03715, 2024.
\newblock \doi{10.48550/ARXIV.2404.03715}.
\newblock URL \url{https://doi.org/10.48550/arXiv.2404.03715}.

\bibitem[Schulman et~al.(2017)Schulman, Wolski, Dhariwal, Radford, and Klimov]{schulman2017proximal}
John Schulman, Filip Wolski, Prafulla Dhariwal, Alec Radford, and Oleg Klimov.
\newblock Proximal policy optimization algorithms.
\newblock \emph{arXiv preprint arXiv:1707.06347}, 2017.

\bibitem[Shao et~al.(2024)Shao, Wang, Zhu, Xu, Song, Bi, Zhang, Zhang, Li, Wu, et~al.]{shao2024deepseekmath}
Zhihong Shao, Peiyi Wang, Qihao Zhu, Runxin Xu, Junxiao Song, Xiao Bi, Haowei Zhang, Mingchuan Zhang, YK~Li, Yang Wu, et~al.
\newblock Deepseekmath: Pushing the limits of mathematical reasoning in open language models.
\newblock \emph{arXiv preprint arXiv:2402.03300}, 2024.

\bibitem[Sokota et~al.(2023)Sokota, D'Orazio, Kolter, Loizou, Lanctot, Mitliagkas, Brown, and Kroer]{sokota2023a}
Samuel Sokota, Ryan D'Orazio, J~Zico Kolter, Nicolas Loizou, Marc Lanctot, Ioannis Mitliagkas, Noam Brown, and Christian Kroer.
\newblock A unified approach to reinforcement learning, quantal response equilibria, and two-player zero-sum games.
\newblock In \emph{The Eleventh International Conference on Learning Representations}, 2023.
\newblock URL \url{https://openreview.net/forum?id=DpE5UYUQzZH}.

\bibitem[Song et~al.(2020)Song, Zhou, Zhou, Jiang, and Ma]{song2020optimistic}
Chaobing Song, Zhengyuan Zhou, Yichao Zhou, Yong Jiang, and Yi~Ma.
\newblock Optimistic dual extrapolation for coherent non-monotone variational inequalities.
\newblock \emph{Advances in Neural Information Processing Systems (NeurIPS)}, 33:\penalty0 14303--14314, 2020.

\bibitem[Suzgun et~al.(2023)Suzgun, Scales, Sch{\"a}rli, Gehrmann, Tay, Chung, Chowdhery, Le, Chi, Zhou, and Wei]{suzgun-etal-2023-challenging}
Mirac Suzgun, Nathan Scales, Nathanael Sch{\"a}rli, Sebastian Gehrmann, Yi~Tay, Hyung~Won Chung, Aakanksha Chowdhery, Quoc Le, Ed~Chi, Denny Zhou, and Jason Wei.
\newblock Challenging {BIG}-bench tasks and whether chain-of-thought can solve them.
\newblock In Anna Rogers, Jordan Boyd-Graber, and Naoaki Okazaki, editors, \emph{Findings of the Association for Computational Linguistics: ACL 2023}, pages 13003--13051, Toronto, Canada, July 2023. Association for Computational Linguistics.
\newblock \doi{10.18653/v1/2023.findings-acl.824}.
\newblock URL \url{https://aclanthology.org/2023.findings-acl.824}.

\bibitem[Swamy et~al.(2024)Swamy, Dann, Kidambi, Wu, and Agarwal]{swamy2024minimaximalist}
Gokul Swamy, Christoph Dann, Rahul Kidambi, Zhiwei~Steven Wu, and Alekh Agarwal.
\newblock A minimaximalist approach to reinforcement learning from human feedback, 2024.

\bibitem[Syrgkanis et~al.(2015)Syrgkanis, Agarwal, Luo, and Schapire]{syrgkanis2015fast}
Vasilis Syrgkanis, Alekh Agarwal, Haipeng Luo, and Robert~E Schapire.
\newblock Fast convergence of regularized learning in games.
\newblock \emph{Advances in Neural Information Processing Systems (NeurIPS)}, 2015.

\bibitem[Tiapkin et~al.(2025)Tiapkin, Calandriello, Belomestny, Moulines, Naumov, Rasul, Valko, and Menard]{tiapkin2025accelerating}
Daniil Tiapkin, Daniele Calandriello, Denis Belomestny, Eric Moulines, Alexey Naumov, Kashif Rasul, Michal Valko, and Pierre Menard.
\newblock Accelerating nash learning from human feedback via mirror prox.
\newblock \emph{arXiv preprint arXiv:2505.19731}, 2025.

\bibitem[Tunstall et~al.(2023)Tunstall, Beeching, Lambert, Rajani, Rasul, Belkada, Huang, von Werra, Fourrier, Habib, et~al.]{tunstall2023zephyr}
Lewis Tunstall, Edward Beeching, Nathan Lambert, Nazneen Rajani, Kashif Rasul, Younes Belkada, Shengyi Huang, Leandro von Werra, Cl{\'e}mentine Fourrier, Nathan Habib, et~al.
\newblock Zephyr: Direct distillation of lm alignment.
\newblock \emph{arXiv preprint arXiv:2310.16944}, 2023.

\bibitem[Tversky(1969)]{tversky1969intransitivity}
Amos Tversky.
\newblock Intransitivity of preferences.
\newblock \emph{Psychological review}, 76\penalty0 (1):\penalty0 31, 1969.

\bibitem[Wang et~al.(2024)Wang, Xiong, Xie, Zhao, and Zhang]{wang2024interpretable}
Haoxiang Wang, Wei Xiong, Tengyang Xie, Han Zhao, and Tong Zhang.
\newblock Interpretable preferences via multi-objective reward modeling and mixture-of-experts.
\newblock \emph{arXiv preprint arXiv:2406.12845}, 2024.

\bibitem[Wang et~al.(2025)Wang, Ma, Chen, Meng, Han, Xiao, Zhang, Huo, Su, and Yang]{wang2025magnetic}
Mingzhi Wang, Chengdong Ma, Qizhi Chen, Linjian Meng, Yang Han, Jiancong Xiao, Zhaowei Zhang, Jing Huo, Weijie~J Su, and Yaodong Yang.
\newblock Magnetic preference optimization: Achieving last-iterate convergence for language model alignment.
\newblock In \emph{The Thirteenth International Conference on Learning Representations}, 2025.
\newblock URL \url{https://openreview.net/forum?id=PDnEDS244P}.

\bibitem[Wei et~al.(2021)Wei, Lee, Zhang, and Luo]{wei2021linear}
Chen-Yu Wei, Chung-Wei Lee, Mengxiao Zhang, and Haipeng Luo.
\newblock Linear last-iterate convergence in constrained saddle-point optimization.
\newblock In \emph{International Conference on Learning Representations (ICLR)}, 2021.

\bibitem[Wu et~al.(2025)Wu, Viano, Chen, Zhu, Antonakopoulos, Gu, and Cevher]{wu2025multi}
Yongtao Wu, Luca Viano, Yihang Chen, Zhenyu Zhu, Kimon Antonakopoulos, Quanquan Gu, and Volkan Cevher.
\newblock Multi-step alignment as markov games: An optimistic online gradient descent approach with convergence guarantees.
\newblock \emph{arXiv preprint arXiv:2502.12678}, 2025.

\bibitem[Wu et~al.(2024)Wu, Sun, Yuan, Ji, Yang, and Gu]{wu2024self}
Yue Wu, Zhiqing Sun, Huizhuo Yuan, Kaixuan Ji, Yiming Yang, and Quanquan Gu.
\newblock Self-play preference optimization for language model alignment.
\newblock \emph{arXiv preprint arXiv:2405.00675}, 2024.

\bibitem[Yang et~al.(2024{\natexlab{a}})Yang, Yang, Hui, Zheng, Yu, Zhou, Li, Li, Liu, Huang, et~al.]{yang2024qwen2}
An~Yang, Baosong Yang, Binyuan Hui, Bo~Zheng, Bowen Yu, Chang Zhou, Chengpeng Li, Chengyuan Li, Dayiheng Liu, Fei Huang, et~al.
\newblock Qwen2 technical report.
\newblock \emph{arXiv preprint arXiv:2407.10671}, 2024{\natexlab{a}}.

\bibitem[Yang et~al.(2024{\natexlab{b}})Yang, Yang, Zhang, Hui, Zheng, Yu, Li, Liu, Huang, Wei, et~al.]{yang2024qwen25}
An~Yang, Baosong Yang, Beichen Zhang, Binyuan Hui, Bo~Zheng, Bowen Yu, Chengyuan Li, Dayiheng Liu, Fei Huang, Haoran Wei, et~al.
\newblock Qwen2. 5 technical report.
\newblock \emph{arXiv preprint arXiv:2412.15115}, 2024{\natexlab{b}}.

\bibitem[Ye et~al.(2024)Ye, Xiong, Zhang, Jiang, and Zhang]{ye2024theoretical}
Chenlu Ye, Wei Xiong, Yuheng Zhang, Nan Jiang, and Tong Zhang.
\newblock A theoretical analysis of nash learning from human feedback under general kl-regularized preference.
\newblock \emph{arXiv preprint arXiv:2402.07314}, 2024.

\bibitem[Zhang et~al.(2025{\natexlab{a}})Zhang, Yu, Ge, Song, Zeng, Mi, Jiang, and Yu]{zhang2025improving}
Yuheng Zhang, Dian Yu, Tao Ge, Linfeng Song, Zhichen Zeng, Haitao Mi, Nan Jiang, and Dong Yu.
\newblock Improving llm general preference alignment via optimistic online mirror descent.
\newblock \emph{arXiv preprint arXiv:2502.16852}, 2025{\natexlab{a}}.

\bibitem[Zhang et~al.(2025{\natexlab{b}})Zhang, Yu, Peng, Song, Tian, Huo, Jiang, Mi, and Yu]{zhang2024iterative}
Yuheng Zhang, Dian Yu, Baolin Peng, Linfeng Song, Ye~Tian, Mingyue Huo, Nan Jiang, Haitao Mi, and Dong Yu.
\newblock Iterative nash policy optimization: Aligning {LLM}s with general preferences via no-regret learning.
\newblock In \emph{The Thirteenth International Conference on Learning Representations}, 2025{\natexlab{b}}.
\newblock URL \url{https://openreview.net/forum?id=Pujt3ADZgI}.

\bibitem[Zhou et~al.(2025)Zhou, Fazel, and Du]{zhou2025extragradient}
Runlong Zhou, Maryam Fazel, and Simon~S Du.
\newblock Extragradient preference optimization (egpo): Beyond last-iterate convergence for nash learning from human feedback.
\newblock \emph{arXiv preprint arXiv:2503.08942}, 2025.

\end{thebibliography}
